\documentclass{article}
\usepackage{iclr2025_conference,times}
\pdfoutput=1

\usepackage{amsmath}
\usepackage{amssymb}
\usepackage{amsfonts}
\usepackage{amsthm}
\usepackage{mathtools}
\usepackage{dsfont}
\usepackage{enumitem}
\usepackage[
    colorlinks=true,
    citecolor={blue!50!black},
    linkcolor={blue!50!black},
    urlcolor={blue!50!black}
]{hyperref}
\usepackage[acronym]{glossaries}
\usepackage[noend]{algpseudocode}
\usepackage{algorithm, algorithmicx}
\usepackage{subcaption}
\usepackage{pifont}
\usepackage{adjustbox}
\usepackage{array}
\usepackage{multicol}
\usepackage{listings}
\usepackage{calc}
\usepackage{thmtools}

\makeglossaries

\newacronym{gp}{GP}{Gaussian process}
\newacronym{vo}{VO}{variational optimization}
\newacronym{vi}{VI}{variational inference}
\newacronym{gvi}{GVI}{generalized variational inference}
\newacronym{bo}{BO}{Bayesian optimization}
\newacronym{ei}{EI}{expected improvement}
\newacronym{pi}{PI}{probability of improvement}
\newacronym{ucb}{UCB}{upper confidence bound}
\newacronym{bore}{BORE}{Bayesian optimization by density-ratio estimation}
\newacronym{bopr}{BOPR}{Bayesian optimization with probabilistic reparametrization}
\newacronym{elbo}{ELBO}{evidence lower bound}
\newacronym{cpe}{CPE}{class probability estimation}
\newacronym{eda}{EDA}{estimation of distribution algorithms}
\newacronym{nes}{NES}{natural evolution strategies}
\newacronym{es}{ES}{evolution strategies}
\newacronym{sgd}{SGD}{stochastic gradient descent}
\newacronym{vsd}{VSD}{variational search distributions}
\newacronym{dbas}{DbAS}{design by adaptive sampling}
\newacronym{cbas}{CbAS}{conditioning by adaptive sampling}
\newacronym{ell}{ELL}{expected log-likelihood}
\newacronym{fdr}{FDR}{false discovery rate}
\newacronym{kl}{KL}{Kullback-Leibler}
\newacronym{pex}{PEX}{proximal exploration}
\newacronym{bbo}{BBO}{black-box optimization}
\newacronym{lstm}{LSTM}{long short-term memory}
\newacronym{rnn}{RNN}{recurrent neural network}
\newacronym{ml}{ML}{maximum likelihood}
\newacronym{lse}{LSE}{level set estimation}
\newacronym{lso}{LSO}{latent space optimization}
\newacronym{ntk}{NTK}{neural tangent kernel}
\newacronym{nos}{NOS}{diffusioN Optimized Sampling}
\newacronym{ga}{GA}{genetic algorithm}
\newacronym{nn}{NN}{neural network}

\newcommand{\thresh}{\ensuremath{\tau}}
\newcommand{\acid}{\ensuremath{x}}
\newcommand{\obs}{\ensuremath{\mathbf{x}}}

\newcommand{\tar}{\ensuremath{y}}

\newcommand{\bbf}{\ensuremath{{f\!\centerdot}}}
\newcommand{\labl}{\ensuremath{z}}

\newcommand{\qparam}{\ensuremath{\phi}}
\newcommand{\qparamspace}{\ensuremath{\Phi}}
\newcommand{\mparam}{\ensuremath{\theta}}

\newcommand{\err}{\ensuremath{\epsilon}}

\newcommand{\ident}{\ensuremath{\mathbf{I}}}
\newcommand{\indep}{\perp \!\!\! \perp}
\newcommand{\real}{\ensuremath{\mathbb{R}}}
\newcommand{\obsspace}{\ensuremath{\mathcal{X}}}
\newcommand{\acidspace}{\ensuremath{\mathcal{V}}}
\newcommand{\solnspace}{\ensuremath{\mathcal{S}}}

\newcommand{\seenspace}{\ensuremath{\obsspace^q}}

\newcommand{\data}{\ensuremath{\mathcal{D}}}

\newcommand{\regretsim}{\ensuremath{\mathrm{r}}}
\newcommand{\recall}{\ensuremath{\textrm{Recall}}}
\newcommand{\precision}{\ensuremath{\textrm{Precision}}}
\newcommand{\performance}{\ensuremath{\textrm{Performance}}}
\newcommand{\threshparam}{\ensuremath{\gamma}}

\newcommand{\acqfn}[2]{\ensuremath{\alpha_{#2}\!\left({#1}\right)}}

\newcommand{\prob}[1]{\ensuremath{p({#1})}}
\newcommand{\probc}[2]{\ensuremath{p({#1}|{#2}})}

\newcommand{\qrob}[1]{\ensuremath{q({#1})}}
\newcommand{\qrobc}[2]{\ensuremath{q({#1}|{#2}})}
\newcommand{\normal}[1]{\ensuremath{\mathcal{N}\!\left({#1}\right)}}

\newcommand{\categc}[2]{\ensuremath{\mathrm{Categ}({#1}|{#2})}}

\newcommand{\expec}[2]{\ensuremath{\mathbb{E}_{#1}\!\left[{#2}\right]}}
\newcommand{\dkl}[2]{\ensuremath{\mathbb{D}_\mathrm{KL}\!\left[{#1}\|{#2}\right]}}
\newcommand{\diver}[2]{\ensuremath{\mathbb{D}\!\left[{#1}\|{#2}\right]}}

\newcommand{\softmax}[1]{\ensuremath{\mathrm{softmax}\!\left({#1}\right)}}
\newcommand{\cpe}[2]{\ensuremath{\pi_{#1}({#2})}}
\newcommand{\indic}[1]{\ensuremath{\mathds{1}[{#1}]}}
\newcommand{\elbo}[1]{\ensuremath{\mathcal{L_{\textrm{ELBO}}}\!\left({#1}\right)}}
\newcommand{\lcpe}[1]{\ensuremath{\mathcal{L_{\textrm{CPE}}}\!\left({#1}\right)}}

\newcommand{\threshfn}[1]{\ensuremath{f_\thresh}\!\left({#1}\right)}

\DeclareMathOperator*{\argmax}{arg\,\!max}
\DeclareMathOperator*{\argmin}{arg\,\!min}

\newcommand*\Let[2]{\State #1 $\gets$ #2}

\newcommand{\card}[1]{\ensuremath{|#1|}}

\newcommand*{\set}[1]{{\mathcal{\MakeUppercase{#1}}}}		
\newcommand*{\collection}[1]{{\mathfrak{\MakeUppercase{#1}}}} 
\renewcommand{\vec}[1]{{\boldsymbol{\mathbf{#1}}}}          
\newcommand*{\mat}[1]{\vec{\MakeUppercase{#1}}}             
\newcommand*{\transpose}{\top}
\newcommand*{\norm}[1]{\lVert #1 \rVert}                    
\newcommand*{\diff}{{\mathop{}\operatorname{d}}}
\newcommand*{\eye}{\ident}							        
\newcommand*{\obsnoise}{\err}
\newcommand*{\lablnoise}{\zeta}
\newcommand*{\gp}{\mathcal{GP}}                             
\newcommand*{\gpmean}{\mu}                                  
\newcommand*{\gpkernel}{k}                                  
\newcommand*{\cdf}{\Psi}                                  
\newcommand*{\N}{\mathbb{N}}                                

\newcommand*{\pdist}{p}
\newcommand*{\pmin}{b}
\newcommand*{\psum}{B}
\newcommand*{\plimit}{L}
\newcommand*{\sumrv}{N}
\newcommand*{\nobs}{N}                                      
\newcommand*{\obsIdx}{n}
\newcommand*{\iterIdx}{t}
\newcommand*{\niter}{T}
\newcommand*{\nhits}{H}
\newcommand*{\Probm}[1]{\mathbb{P}\!\left[ #1 \right]}                            
\newcommand*{\anyscalar}{a}
\newcommand*{\indrv}{\chi}													
\newcommand*{\anyrv}{\rho}
\newcommand*{\SampleSpace}{\Omega}
\newcommand*{\sample}{\omega}
\newcommand*{\EventsAlgebra}{\collection{A}}
\newcommand*{\filtration}{\collection{F}}
\newcommand*{\pMeasure}{\mathbb{P}}
\newcommand*{\anyevent}{\set{A}}
\newcommand*{\borel}{\collection{B}}
\newcommand*{\expectation}{\mathbb{E}}
\newcommand*{\bigo}{\set{O}}							
\newcommand*{\constant}{C}								
\newcommand*{\lik}{\ell}										
\newcommand*{\ldiff}{{\Delta\lik}}								
\newcommand*{\anyfunction}{g}
\newcommand*{\anotherfunction}{h}
\newcommand*{\classifier}{\pi}
\newcommand*{\lmse}[1]{\mathcal{L}_{\mathrm{MSE}}(#1)}
\newcommand*{\fspace}{\set{F}}
\newcommand*{\inner}[1]{\langle #1 \rangle}
\newcommand*{\feature}{\varphi}
\newcommand*{\features}{\mat{\Phi}}

\newcommand*{\lsv}{\mat{U}}	 	
\newcommand*{\lsvec}{u}	
\newcommand*{\rsvec}{\vec{v}}
\newcommand*{\rsv}{\mat{V}}		
\newcommand*{\svs}{\mat{S}}  

\newcommand*{\evs}{\mat{\Lambda}}	
\newcommand*{\eigval}{\lambda}

\newcommand*{\mwidth}{m}
\newcommand*{\regfactor}{\rho}
\newcommand*{\regop}{\mat{R}}
\newcommand*{\regmat}{\mat{\Sigma}}
\newcommand*{\regdiag}{\mat{D}}
\newcommand*{\tstop}{s}

\newcommand*{\lrate}{\nu}
\newcommand*{\mig}{\xi}

\declaretheorem[numberwithin=section]{proposition}
\declaretheorem[numberwithin=section]{theorem}
\declaretheorem[numberwithin=section]{lemma}

\declaretheorem[numberwithin=section, name=Remark]{rremark}
\newtheorem{assumption}{Assumption}[section]

\makeatletter
\def\namedlabel#1#2{\begingroup
    #2%
    \def\@currentlabel{#2}%
    \phantomsection\label{#1}\endgroup
}
\makeatother

\newcolumntype{R}[2]{%
    >{\adjustbox{angle=#1,lap=\width-(#2)}\bgroup}%
    l%
    <{\egroup}%
}
\newcommand*\rot{\multicolumn{1}{R{27}{1em}}}

\newtheorem*{require*}{Requirements \& Desiderata}

\newcommand{\cmark}{\ding{51}}
\newcommand{\xmark}{\ding{55}}

\title{Variational Search Distributions}
\author{Daniel M. Steinberg, Rafael Oliveira, Cheng Soon Ong \& Edwin V. Bonilla \\
Data61, CSIRO, Australia\\
\scriptsize\texttt{\{dan.steinberg, rafael.dossantosdeoliveira, cheng-soon.ong, edwin.bonilla\}@data61.csiro.au} \\
}

\iclrfinalcopy 
\begin{document}

\maketitle

\begin{abstract}
    We develop \gls{vsd}, a method for conditioning a generative model of discrete, combinatorial designs on a rare desired class by efficiently evaluating a black-box (e.g.~experiment, simulation) in a batch sequential manner.
    We call this task active generation; we formalize active generation's requirements and desiderata, and formulate a solution via variational inference. \gls{vsd} uses off-the-shelf gradient based optimization routines, can learn powerful generative models for desirable designs, and can take advantage of scalable predictive models. We derive asymptotic convergence rates for learning the true conditional generative distribution of designs with certain configurations of our method. After illustrating the generative model on images, we empirically demonstrate that \gls{vsd} can outperform existing baseline methods on a set of real sequence-design problems in various protein and DNA/RNA engineering tasks.
\end{abstract}

\section{Introduction}

We consider a variant of the active search problem \citep{garnett2012bayesian, jiang2017efficient, vanchinathan2015discovering}, where we wish to find members (designs) of a rare desired class in a batch sequential manner with a fixed black-box evaluation (e.g.~experiment) budget. We call sequential active learning of a \emph{generative} model of these designs \textbf{active generation}. Examples of rare designs are compounds that could be useful pharmaceutical drugs, or highly active enzymes for catalyzing chemical reactions. We assume the design space is discrete or partially discrete, high-dimensional and practically \emph{innumerable}.
For example, the number of possible configurations of a single protein is $20^{\bigo(100)}$ \citep[see, e.g.,][]{sarkisyan2016local}.
Learning a generative model of these designs allows us to circumvent the need for traversing the whole search space.

We are interested in this active generation objective for a variety of reasons. We may wish to study the properties of the ``fitness landscape'' \citep{papkou2023rugged} to gain a better scientific understanding of a phenomenon such as natural evolution. Or, we may not be able to completely specify the constraints and objectives of a task, but we would like to characterize the space of, and generate new feasible designs. For example, we want enzymes that can degrade plastics in an industrial setting, but we may not yet know the exact conditions (e.g.~temperature, pH), some of which may be anti-correlated with enzyme catalytic activity. Alternatively, if we know these multiple objectives and constraints, we may only want to generate designs from a Pareto set.

Assuming we can take advantage of a prior distribution over designs, we formulate the search problem as inferring the posterior distribution over rare, desirable designs. Importantly, this posterior can be used for \textit{generating new designs}. Specifically, we use (black-box) \gls{vi}~\citep{ranganath2014black}, and so refer to our method as \acrfull{vsd}. Our major contributions are: (1) we formulate the batch active generation objective over a (practically) innumerable discrete design space, (2)  we present a variational inference algorithm, \gls{vsd}, which solves this objective, (3) we show that \gls{vsd} performs well theoretically and empirically, and (4) we connect active generation to other recent advances in \gls{bbo} of discrete sequences that use generative models. \gls{vsd} uses off-the-shelf gradient based optimization routines, is able to learn powerful generative models, and can take advantage of scalable predictive models. In our experiments we show that \gls{vsd} can outperform existing baseline methods on a set of real applications.
Finally, we evaluate our approach on the related sequential \gls{bbo} problem, where we want to find the globally optimal design for a specific objective and show competitive performance when compared with state-of-the-art methods, e.g., based on \gls{lso} \citep{gruver2023protein}.

\section{Method}
\label{sec:method}
In this section we formalize our problem and describe its requirements and desiderata. We also develop our proposed  solution, based on variational inference, which we will refer to as \acrfull{vsd}.

\subsection{The Problem of Active Generation}
\label{sub:formulation}

We are given a design space $\obsspace$, which can be discrete or mixed discrete-continuous and high dimensional, and where for each instance that we choose $\obs \in \obsspace$, we measure some corresponding property of interest (so-called fitness) $\tar \in \real$. For example, in our motivating application of DNA/RNA or protein sequences, $\obsspace = \acidspace^M$, where $\acidspace$ is the sequence vocabulary (e.g.,~amino acid labels, $\card{\acidspace}=20$) and $M$ is the length of the sequence. However, we do not limit the application of our method to sequences. Using this framing, a real world experiment (e.g.,~measuring the activity of an enzyme) can be modeled as an unknown relationship,
\begin{align}
    \tar = \bbf(\obs) + \err,
    \label{eq:targen}
\end{align}
for some black-box function (the experiment), $\bbf$, and measurement error $\err \in \real$, distributed according to $\prob{\err}$ with $\expec{\prob{\err}}{\err} = 0$.
Instead of modeling the whole space, $\obsspace$, we are only interested in a set of events which we choose based on fitness, $\solnspace \subset \obsspace$.
In particular for active generation we wish to learn a generative model, $\qrob{\obs}$, that only returns samples $\obs^{(s)} \in \solnspace$ by efficiently querying the black-box function in \autoref{eq:targen}.
%
%
We assume that $\solnspace$ are rare events in a high dimensional space, and that we have access to a prior belief, $\prob{\obs}$, which helps narrow in on this subset of $\obsspace$.
We are given an initial dataset, $\data_N := \{(\tar_n, \obs_n)\}_{n=1}^N$, which may contain only a few instances of $\obs_n \in \solnspace$. Given $\prob{\obs}$ and $\data_N$ we aim to generate batches of unique candidates, $\{\obs_{bt}\}^B_{b=1}$, for black-box (experimental) evaluation in a series of rounds, $\iterIdx \in \{1, \ldots, \niter\}$, where $B = \bigo(1000)$ and we desire $\obs_{bt} \in \solnspace$. After each round $\data_N$ is augmented with the black-box results of the batch, i.e.~$\data_N \leftarrow \data_N \cup \{(\obs_{bt}, \tar_{bt})\}^B_{b=1}$.
As we shall see later, our solution allows us to satisfy the following requirements and additional desiderata for active generation.
\\
\begin{require*} Active generation requirements (R) and other desiderata (D).%
\vspace{-.1em}
\small%
\begin{multicols}{2}%
\begin{description}
    \item[\namedlabel{req:rare}{(R1)} Rare] feasible designs, $\solnspace$, are rare events in $\obsspace$ that need to be identified
    \item[\namedlabel{req:seq}{(R2)} Sequential] non-myopic candidate generation, $\obs \in \solnspace$, for sequential black-box evaluation
    \item[\namedlabel{req:disc}{(R3)} Discrete] search over (combinatorially) large design spaces, e.g.~$\obs \in \obsspace = \acidspace^M$
    \item[\namedlabel{req:batch}{(R4)} Batch] generation of up to $\mathcal{O}(1000)$ \emph{diverse} candidate designs per round
    \item[\namedlabel{req:gen}{(R5)} Generative] models, $\obs^{(s)} \sim \qrob{\obs}$, that are task-specific for rare, feasible designs
    \item[\namedlabel{desi:guar}{(D1)} Guaranteed]%
    convergence for certain choices of priors, variational distributions and predictive models
    \item[\namedlabel{desi:gd}{(D2)} Gradient] based optimization strategies for candidate searching
    \item[\namedlabel{desi:scale}{(D3)} Scalable] predictive models that enable high-throughput evaluation/experiments.
\end{description}
\end{multicols}
\label{requirements}
\vspace{-1em}
\end{require*}
Like active search \citep{garnett2012bayesian} in our case we are interested in the solution space of the super level-set, $\solnspace_\textrm{SLS} := \{\obs : \bbf(\obs) > \thresh \}$ for a threshold $\thresh \in \real$ (e.g.,~wild-type fitness). As we only have access to noisy measurements, $\tar$, \emph{our task is to estimate the super level-set distribution}, $\probc{\obs}{\tar > \thresh}$, \emph{using active generation}. Estimating this distribution over $\solnspace_\textrm{SLS}$ is computationally and statistically challenging and, therefore, we cast this as a \emph{variational inference} problem.
We also consider the case of \gls{bbo} for which $\solnspace_\textrm{BBO} := \argmax_\obs \bbf(\obs)$, and we show that we can accommodate this in our variational framework by iteratively raising $\thresh_t$ per round, $t$.
We visualize the properties and models involved in active generation as applied to a continuous ``fitness landscape'' in \autoref{fig:islands}.
%
%
\begin{figure}[tb]%
    \centering
    \subcaptionbox{$\argmax_\obs \bbf(\obs)$\label{sfig:argmaxx}}
        {\includegraphics[width=.22\textwidth]{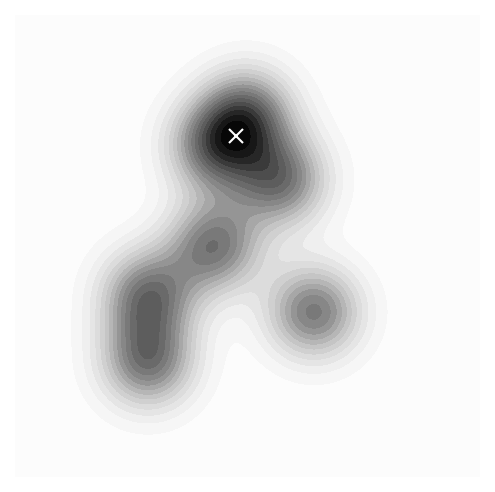}}
    \hfill
    \subcaptionbox{$\{\obs : \bbf(\obs) > \thresh \}$\label{sfig:fitset}}
        {\includegraphics[width=.22\textwidth]{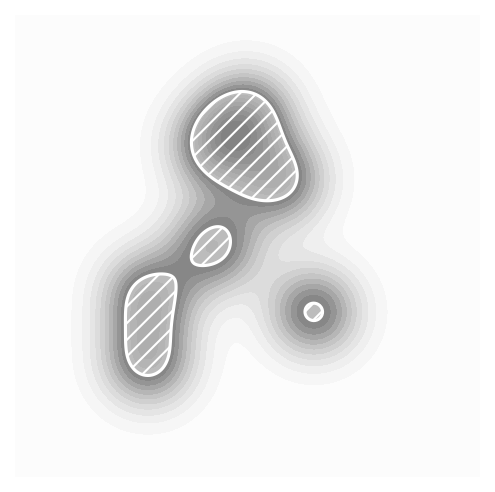}}
    \hfill
    \subcaptionbox{$\prob{\obs}$\label{sfig:px}}
        {\includegraphics[width=.22\textwidth]{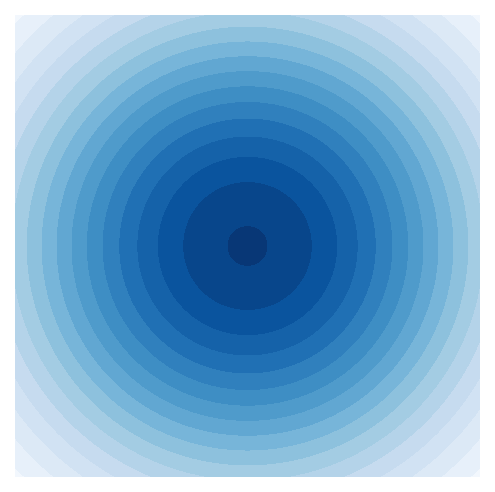}}
    \hfill
    \subcaptionbox{$\probc{\obs}{\tar > \thresh}$\label{sfig:pxgy}}
        {\includegraphics[width=.22\textwidth]{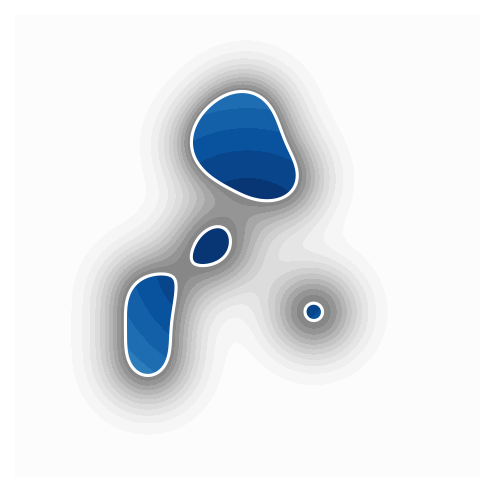}}
    \caption{Fitness landscape properties and models. (\subref{sfig:argmaxx}) A noise-less fitness landscape, $\bbf(\obs)$, and the maximum fitness design, $\solnspace_\textrm{BBO} = \{\obs^*\}$, as the white `$\times$'. (\subref{sfig:fitset}) The super level-set, $\solnspace_\textrm{SLS}$, of all fit designs as the white hatched area. (\subref{sfig:px}) Prior belief $\prob{\obs}$. (\subref{sfig:pxgy}) The density/mass function of the super level-set, $\probc{\obs}{\tar > \thresh}$, as blue contours.
    Our goal is to sequentially estimate a generative model for the distribution of the super level-set (\subref{sfig:pxgy}). We assume a noisy relationship between $\bbf$ and $\tar$, so the super level-set will not have a hard boundary, and $\probc{\obs}{\tar > \thresh}$ will be defined over all $\obsspace$.}
    \label{fig:islands}
\end{figure}%

\subsection{Variational Search Distributions}

We cast the estimation of $\probc{\obs}{\tar > \thresh}$ as a sequential optimization problem. A suitable objective for a round, $t$, is to minimize a divergence,
\begin{align}
    \qparam_t^* = \argmin_\qparam \diver{\probc{\obs}{\tar > \thresh}}{\qrobc{\obs}{\qparam}}
    \label{eq:diver}
\end{align}
where $\qrobc{\obs}{\qparam}$ is a parameterized distribution from which we sample candidate designs $\obs_{bt}$, \ref{req:gen}, and which we aim to match to $\probc{\obs}{\tar > \thresh}$.
The difficulty is that we cannot directly evaluate or empirically sample from $\probc{\obs}{\tar > \thresh}$. However, if we consider the reverse \gls{kl} divergence,
\begin{align}
    \argmin_\qparam\, \dkl{\qrobc{\obs}{\qparam}}{\probc{\obs}{\tar > \thresh}}
    = \argmin_\qparam \expec{\qrobc{\obs}{\qparam}}
        {\log \frac{\qrobc{\obs}{\qparam}}{\prob{\obs}} - \log \probc{\tar > \thresh}{\obs}},
    \label{eq:kld}
\end{align}
where we have expanded $\probc{\obs}{\tar > \thresh}$ using Bayes rule and dropped the constant term $\prob{\tar > \thresh}$, we note that we no longer require evaluation of $\probc{\obs}{\tar > \thresh}$ directly. We recognize the right-hand side of \autoref{eq:kld} as the well known (negative) variational \gls{elbo},
\begin{align}
    \elbo{\qparam} :=
    \expec{\qrobc{\obs}{\qparam}}{\log \probc{\tar > \thresh}{\obs}} - \dkl{\qrobc{\obs}{\qparam}}{\prob{\obs}}.
    \label{eq:gen_elbo}
\end{align}
For this we assume access to a prior distribution over the space of designs, $\prob{\obs}$, that may be informed from the data at hand (or pre-trained). Henceforth, as we will develop a sequential algorithm, we will denote this prior as $\probc{\obs}{\data_0}$.
We note the relationship between $\log \probc{\tar > \thresh}{\obs}$ and the \gls{pi} acquisition function from \gls{bo}~\citep{kushner1964new},
\begin{align}
    \log \probc{\tar > \thresh}{\obs} \approx \log \probc{\tar > \thresh}{\obs, \data_N}
    = \log \expec{\probc{\tar}{\obs, \data_N}}{\indic{\tar > \thresh}} = \log \acqfn{\obs, \data_N, \thresh}{PI}.
    \label{eq:pr2piacq}
\end{align}
Here $\mathds{1}: \{\textrm{false}, \textrm{true}\} \to \{0, 1\}$ is the indicator function and $\probc{\tar}{\obs, \data_\nobs}$ is typically estimated using the posterior predictive distribution of a \gls{gp} given data, $\data_N$. So $\probc{\tar > \thresh}{\obs, \data_\nobs} = \cdf((\gpmean_\nobs(\obs) - \thresh)/\sigma_\nobs(\obs))$, where $\cdf(\cdot)$ is a cumulative standard normal distribution function, and $\gpmean_\nobs(\obs)$, $\sigma_\nobs^2(\obs)$ are the posterior predictive mean and variance, respectively, of the \gls{gp}. We refer to this estimation strategy as \gls{gp}-\gls{pi}, and rewrite the \gls{elbo} accordingly,
\begin{align}
    \elbo{\qparam, \thresh, \data_N} = \expec{\qrobc{\obs}{\qparam}}{\log \acqfn{\obs, \data_N, \thresh}{PI}} - \dkl{\qrobc{\obs}{\qparam}}{\probc{\obs}{\data_0}}.
    \label{eq:vsd_elbo}
\end{align}
We refer to the method that maximizes the objective in \autoref{eq:vsd_elbo} as \acrfull{vsd}, since we are using the variational posterior distribution as a means of searching the space of fit designs, satisfying \ref{req:rare}, \ref{req:seq} and \ref{req:batch}. It is well known that when the true posterior is a member of the variational family indexed by $\qparam$, the above variational inference procedure has the potential to recover the exact posterior distribution. To recommend candidates for black-box evaluation we sample a set of designs from our search distribution each round,
\begin{align}
    \{\obs_{bt}\}^B_{b=1} \sim \prod^B_{b=1} \qrobc{\obs}{\qparam_t^*},
    \quad \textrm{where} \quad
    \qparam_t^* = \argmax_\qparam \elbo{\qparam, \thresh, \data_N}.
\end{align}
We discuss the relationship between \gls{vsd} and \gls{bo} in \autoref{app:vsdbound}. In general, because of the discrete combinatorial nature of our problem, we cannot use the re-parameterization trick \citep{kingma2014auto} to estimate the gradients of the \gls{elbo} straightforwardly. Instead, we use the score function gradient, also known as REINFORCE~\citep{williams1992simple, mohamed2020monte} with standard gradient descent methods \ref{desi:gd} such as Adam~\citep{kingma2014adam},
\begin{align}
   \nabla_\qparam \elbo{\qparam, \thresh, \data_N} &= \expec{\qrobc{\obs}{\qparam}}{\left(\log \acqfn{\obs, \data_N, \thresh}{PI}
        - \log \qrobc{\obs}{\qparam} + \log\probc{\obs}{\data_0}
        \right)
        \nabla_\qparam \log \qrobc{\obs}{\qparam}}.
    \label{eq:vsdderiv}
\end{align}
Here we use Monte-Carlo sampling to approximate the expectation with a suitable variance reduction scheme, such as control variates \citep{mohamed2020monte}. We find that the exponentially smoothed average of the \gls{elbo} works well in practice, and is the same strategy employed in~\citet{daulton2022bayesian}. \Gls{vsd} implements black-box variational inference~\citep{ranganath2014black} for parameter estimation, and despite the high-dimensional nature of $\obsspace$, we find we only need $\bigo(1000)$ samples to estimate the required expectations for \gls{elbo} optimization on problems with $M = \mathcal{O}(100)$, satisfying~\ref{req:disc}. Note that \autoref{eq:vsd_elbo} -- \ref{eq:vsdderiv} do not involve any data ($\data_N$) directly, only indirectly through the acquisition function. Hence the scalability of \gls{vsd} is dependent on the complexity of training the underlying estimator of $\probc{\tar}{\obs, \data_N}$.

\subsection{Class Probability Estimation}

So far our method indirectly computes \gls{pi} by transforming the predictions of a \gls{gp} surrogate model, $\probc{\tar}{\obs, \data_N}$, as in \autoref{eq:pr2piacq}. Instead we may choose to follow the reasoning used by \gls{bore} in~\citet{tiao2021bore, oliveira2022batch, song2022general}, and directly estimate the quantity we care about, $\probc{\tar > \thresh}{\obs, \data_N}$. This can be accomplished using \gls{cpe} on the labels $ \labl := \indic{\tar > \thresh} \in \{0, 1\}$ so
$\probc{\tar > \thresh}{\obs, \data_N} = \probc{\labl=1}{\obs, \data_N} \approx \cpe{\mparam}{\obs}$, where $\pi_\mparam : \obsspace \to [0, 1]$. We can estimate the class probabilities
using a proper scoring rule~\citep{gneiting2007strictly} such as
log-loss,
\begin{align}
    \lcpe{\mparam, \data_N^\labl} := - \frac{1}{N} \sum\nolimits_{n=1}^N
    \labl_n \log \cpe{\mparam}{\obs_n}
    + (1-\labl_n) \log (1 - \cpe{\mparam}{\obs_n}),
\end{align}
where $\data_N^\labl = \{(\labl_n, \obs_n)\}^N_{n=1}$. The \gls{vsd} objective and gradient estimator using \gls{cpe} then become,
\begin{align}
    \elbo{\qparam, \mparam} &=
        \expec{\qrobc{\obs}{\qparam}}{\log \cpe{\mparam}{\obs}}
        - \dkl{\qrobc{\obs}{\qparam}}{\probc{\obs}{\data_0}}, \label{eq:vsdcpe} \\
    \nabla_\qparam \elbo{\qparam, \mparam} &= \expec{\qrobc{\obs}{\qparam}}{\left(\log \cpe{\mparam}{\obs}
       - \log \qrobc{\obs}{\qparam} + \log \probc{\obs}{\data_0}
       \right)
        \nabla_\qparam \log \qrobc{\obs}{\qparam}}.
\end{align}
into which we plug $\mparam_t^* = \argmin_\mparam \lcpe{\mparam, \data_N^\labl}$. 
We refer to this strategy as \gls{cpe}-\gls{pi}. Using a \gls{cpe} enables the use of more scalable estimators than \gls{gp}-\gls{pi}, satisfying our desideratum \ref{desi:scale}. This is crucial if we choose to run more than a few rounds of experiments with $B=\bigo(1000)$. Since \gls{vsd} is a black-box method we may choose to use \glspl{cpe} that are non-differentiable, such as decision tree ensembles.
The complete \gls{vsd} algorithm is given in~\autoref{alg:optloop} and depicted in \autoref{fig:diagrams}. We have allowed for a threshold function, $\thresh_t = \threshfn{\{\tar : \tar \in \data_N\}, \threshparam_t}$, that can be used to modify the threshold each round. For example, an empirical quantile function $\thresh_t = \hat{Q}_\tar(\threshparam_t)$ where $\threshparam_t \in (0, 1)$ as in \citet{tiao2021bore}. Or a constant $\thresh$ for estimating a constant distribution of the super level-set.
\begin{algorithm}
    \caption{\gls{vsd} optimization loop with \gls{cpe}-\gls{pi}.}
    \begin{algorithmic}[1]
        \Require{
            Threshold function $f_\thresh$ and $\threshparam_1$,
            dataset $\data_N$,
            black-box $\bbf$,
            prior $\probc{\obs}{\data_0}$,
            \gls{cpe} $\cpe{\mparam}{\obs}$,
            variational family $\qrobc{\obs}{\qparam}$,
            budget $T$ and $B$.
        }
        \Function{FitModels}{$\data_N, \thresh$}
            \Let{$\data^\labl_N$}{$\{(\labl_n, \obs_n)\}^N_{n=1}$,~
                where $\labl_n = \indic{\tar_n > \thresh}$}
            \Let{$\mparam^*$}{$\argmin_\mparam \mathcal{L}_\text{CPE}(\mparam, \data_N^\labl)$}
            \Let{$\qparam^*$}{$\argmax_\qparam
                \mathcal{L}_\text{ELBO}(\qparam, \mparam^*)$}
            \State \Return{$\qparam^*, \mparam^*$}
        \EndFunction
        \For{round $t \in \{1, \ldots, T\}$}
            \Let{$\thresh_t$}{$\threshfn{\{\tar : \tar \in \data_N\}, \threshparam_t}$}
            \Let{$\qparam_t^*, \mparam_t^*$}{\Call{FitModels}{$\data_N, \thresh_t$}}
            \Let{$\{\obs_{bt}\}^B_{b=1}$}{$\qrobc{\obs}{\qparam_t^*}$}
            \Let{$\{\tar_{bt}\}^B_{b=1}$}{$\{\bbf(\obs_{bt}) + \err_{bt}\}^B_{b=1}$}
            \Let{$\data_{N}$}{$\data_N \cup \{(\obs_{bt}, \tar_{bt})\}^B_{b=1}$}
        \EndFor
        \Let{$\thresh_*$}{$\threshfn{\{\tar : \tar \in \data_N\}, \threshparam_*}$}
        \Let{$\qparam^*, \mparam^*$}{\Call{FitModels}{$\data_N, \thresh_*$}}
        \State \Return{$\qparam^*, \mparam^*$}
    \end{algorithmic}
    \label{alg:optloop}
\end{algorithm}

\subsection{Theoretical Analysis}
We show that \gls{vsd} sampling distributions converge to a target distribution that characterizes the level set given by $\thresh$, satisfying \ref{desi:guar} in two general settings.
We first derive results assuming $\bbf$ is drawn from a Gaussian process, i.e., $\bbf\sim\gp(0, \gpkernel)$, with a positive-semidefinite covariance (or kernel) function $\gpkernel: \obsspace \times \obsspace \to \real$ (\autoref{app:theory}), using \gls{gp}-\gls{pi} as the \gls{cpe} for \gls{vsd}. These results are then extended to probabilistic classifiers based on wide \glspl{nn} (\autoref{app:ntk-theory}) by means of the \gls{ntk} for the given architecture \citep{jacot2018}.
For the analysis, we set $B=1$ and $N=t$, though having $B>1$ should improve the rates by a multiplicative factor.
\begin{restatable}{theorem}{mainthm}
	\label{thr:kl-bound}
	Under mild assumptions (\ref{a:gp} to \ref{a:prior}), the variational distribution of \gls{vsd} equipped with \gls{gp}-\gls{pi} converges to the level-set distribution in probability at the following rate:
	\begin{equation}
		\begin{split}
				\dkl{\probc{\obs}{\tar > \thresh_\iterIdx, \data_\iterIdx}}{\probc{\obs}{\tar > \thresh_\iterIdx, \bbf}}
				\in \bigo_\pMeasure(\iterIdx^{-1/2})
		\end{split}\,.
	\end{equation}%
\end{restatable}%
This result is based on showing that the \gls{gp} posterior variance vanishes at an optimal rate of $\bigo(\iterIdx^{-1})$ in our setting (\autoref{thr:gp-variance-convergence}).
We also analyze the rate at which \gls{vsd} finds feasible designs, or ``hits'', compared to an oracle with full knowledge of $\bbf$.
After $\niter$ rounds, the number of hits found by \gls{vsd} is
$\nhits_\niter = \sum_{t=1}^\niter \indic{\tar_\iterIdx > \thresh_{\iterIdx-1}}$, where $\tar_\iterIdx$ follows \autoref{eq:targen} and $\obs_\iterIdx \sim \probc{\obs}{\tar > \thresh_{\iterIdx-1}, \data_{\iterIdx-1}}$. The number of hits, $\nhits_\niter^*$, from an agent that fully knows $\bbf$ is the same but for generating conditioned on $\bbf$ with $\obs_\iterIdx \sim \probc{\obs}{\tar > \thresh_{\iterIdx-1}, \bbf}$.
Using this definition and \autoref{thr:kl-bound}, we prove the following.
\begin{restatable}{corollary}{hitscor}
    \label{thm:hits}
    Under the settings in \autoref{thr:kl-bound}, we also have that:
    \begin{equation}
        \expectation[ |\nhits_\niter - \nhits_\niter^*| ] \in \bigo(\sqrt{\niter})\,.
    \end{equation}
\end{restatable}
$\expectation[\nhits_\niter]$ is related to the empirical recall measure \eqref{eq:rec} up to the normalization constant, but it does not account for repeated hits, which are treated as false discoveries (false positives) under recall. Lastly, for \gls{nn}-based \glspl{cpe}, we obtain convergence rates dependent on the spectrum of the \gls{ntk} (\autoref{thr:ntk-kl}), which we instantiate for infinitely wide ReLU networks below. For the full results and proofs, please see \autoref{app:theory} for the GP-based analysis and \autoref{app:ntk-theory} for the \gls{ntk} results.
\begin{restatable}{corollary}{relukl}
\label{thr:relu-kl}
Let $\classifier_\mparam$ be modeled via a fully connected ReLU network. Then, under assumptions on identifiability and sampling (\ref{a:rkhs} to \ref{a:learning-rate}), in the infinite-width limit, \gls{vsd} with \gls{cpe}-\gls{pi} achieves:
\begin{equation}
    \dkl{\probc{\obs}{\tar > \thresh_\iterIdx, \data_\iterIdx}}{\probc{\obs}{\tar > \thresh_\iterIdx, \bbf}} \in \widetilde{\bigo}_\pMeasure\left(\iterIdx^{-\frac{1}{2(M + 1)}}\right).
\end{equation}
\end{restatable}
This result finally indicates that, when equipped with flexible \gls{nn}-based \glspl{cpe}, \gls{vsd} is also capable of recovering the target distribution for arbitrary sequence lengths in combinatorial problems.

\section{Related Work}
\label{sec:related}

We will consider related work firstly in terms of methods that have similar components to \gls{vsd},
then secondly in terms of related problems to our specification of active generation.
\gls{vsd} can be viewed as one of many methods that makes use of the variational bound \citep{staines2013optimization},
\begin{align}
    \max_\obs \bbf(\obs)
        \geq \max_\qparam \expec{\qrobc{\obs}{\qparam}}{\bbf(\obs)}.%
    \label{eq:maxbound}%
\end{align}%
The maximum is always greater than or equal to the expected value of a random variable. This bound is useful for \acrfull{bbo} of $\bbf$, and becomes tight if $\qrobc{\obs}{\qparam} \to \delta(\obs^*)$. See \autoref{app:vsdbound} for more detail and \gls{vsd}'s relation to \gls{bo}. Other well known methods that make use of this bound are \gls{nes}~\citep{wierstra2014natural}, \gls{vo}~\citep{staines2013optimization, bird2018stochastic}, \gls{eda}~\citep{larranaga2001estimation, brookes2020view}, and \gls{bopr}~\citep{daulton2022bayesian}. For learning the parameters of the variational distribution, $\qparam$, they variously make use of maximum likelihood estimation or the score function gradient estimator (REINFORCE) \citep{williams1992simple}.
Algorithms that explicitly modify \autoref{eq:maxbound} to stop the collapse of $\qrobc{\obs}{\qparam}$ to a point mass for batch design include \gls{dbas}~\citep{brookes2018design} and \gls{cbas}~\citep{brookes2019conditioning}. They use fixed samples $\obs^{(s)}$ from $\qrobc{\obs}{\qparam^*_{t-1}}$ for approximating the expectation, and then optimize $\qparam$ using a weighted maximum-likelihood or variational style procedure.
Though \gls{dbas} and \gls{cbas} were formulated for offline (non-sequential) tasks, they have often been used in a sequential setting~\citep{ren2022proximal}.
We can take a unifying view of algorithms that use a surrogate model for $\bbf$ by recognizing the general gradient estimator,
\begin{align}
    \expec{\qrobc{\obs}{\qparam'}}{w(\obs) \nabla_\qparam \log \qrobc{\obs}{\qparam}}
    \label{eq:eda_grads}.
\end{align}
where we give each component in \autoref{tab:eda}.
For our experiments \gls{bore} has been adapted to discrete $\obsspace$ by using the score function gradient estimator, which we denote by \Gls{bore}$^*$, while \gls{cbas} and \gls{dbas} have been adapted to use a \gls{cpe} -- their original derivations use a \gls{pi} acquisition function.

\begin{table}[tb]
    \small
    \centering
    \begin{tabular}{r|c|c|c|c}
    \textbf{Method} & $w(\obs)$ & $\qparam'$ & Fixed $\obs^{(s)} \sim \qrobc{\obs}{\qparam'}$ per round? \\
    \hline
    \gls{vsd} & $\log\cpe{\mparam^*}{\obs} + \log \probc{\obs}{\data_0} - \log \qrobc{\obs}
        {\qparam}$ & $\qparam$ & No (REINFORCE)\\
    \gls{cbas} & $\cpe{\mparam^*}{\obs} \probc{\obs}{\data_0} / \qrobc{\obs}
        {\qparam^*_{t-1}}$ & $\qparam^*_{t-1}$ & Yes (importance Monte Carlo) \\
    \gls{dbas} & $\cpe{\mparam^*}{\obs}$ & $\qparam^*_{t-1}$ & Yes (Monte Carlo) \\
    \gls{bore}$^*$ & $\cpe{\mparam^*}{\obs}$ & $\qparam$ & No (REINFORCE)\\
    \gls{bopr} & $\acqfn{\obs, \data_N}{}$ & $\qparam$ & No (REINFORCE)
    \end{tabular}
    \caption{How related methods can be adapted from \autoref{eq:eda_grads}. \gls{vsd}, \gls{cbas} and \gls{dbas} may also use a cumulative distribution representation of $\acqfn{\obs, \data_N, \thresh}{\textrm{PI}}$ in place of $\cpe{\mparam^*}{\obs}$.}
    \label{tab:eda}
    \vspace{-1em}
\end{table}

A number of finite horizon methods have been applied to biological sequence \gls{bbo} tasks, such as Amortized \gls{bo} \citep{swersky2020amortized}, GFlowNets~\citep{jain2022biological}, and the reinforcement learning based DynaPPO~\citep{angermueller2019model}. \Gls{lso}-like methods \citep{gomez2018automatic, tripp2020sample, stanton2022accelerating, gruver2023protein} tackle optimization of sequences by encoding them into a continuous latent space within which candidate optimization or generation takes place.
Selected candidates are decoded back into sequences before black box evaluation; see \citet{gonzalez2024survey} for a comprehensive survey. \Gls{vsd} does not require a latent space nor an encoder, and as such can be seen as an amortized variant of probabilistic reparameterisation methods \citep{daulton2022bayesian} or continuous relaxations \citep{michael2024continuous}.  Heuristic stochastic search methods such as AdaLead~\citep{sinai2020adalead} and \gls{pex}~\citep{ren2022proximal} have also demonstrated strong empirical performance on these tasks. We compare the properties of the most relevant methods to our problem in \autoref{tab:compare}.

In contrast to just finding the maximum using BBO, active generation considers another problem -- generating samples from a rare set of feasible solutions. Generation methods that estimate the super level-set distribution, $\probc{\obs}{\tar > \thresh}$, include \gls{cbas}, which optimizes the forward \gls{kl} divergence, $\dkl{\probc{\obs}{\tar > \thresh}}{\qrobc{\obs}{\phi}}$ using importance weighted cross entropy estimation~\citep{rubinstein1999cross}. Batch-\gls{bore}~\citep{oliveira2022batch} also optimizes the reverse \gls{kl} divergence and uses \gls{cpe}, but with Stein variational inference~\citep{liu2016stein} for diverse batch candidates (with a continuous relaxation for discrete variables).
There is a rich literature on the related task of active learning and \gls{bo} for \gls{lse} \citep{bryan2005active, gotovos2013active, bogunovic2016truncated, zhang2023learning}. However, we focus on learning a generative model of a discrete space.

For active generation \gls{vsd}, \gls{cbas} and \gls{dbas} all use an acquisition function defined in the \emph{original} domain, $\obsspace$, to weight gradients (see \autoref{eq:eda_grads}) for learning a conditional generative model, from which $\obs_{bt}$ are sampled.
An alternative is to use \emph{guided generation}, that is to train an unconditional generative model, and then have a discriminative model guide (condition) the samples from the unconditional model at test time. This plug-and-play of a discriminative model has shown promise for controlled image and text generation of pre-trained models~\citep{nguyen2017plug, dathathri2020plug, li2022diffusion, zhang2023survey}.
LaMBO~\citep{stanton2022accelerating} and LaMBO-2~\citep{gruver2023protein} take a guided generation approach to solve the active generation problem. LaMBO uses an (unconditional) masked language model auto-encoder, and then optimizes sampling from its latent space using an acquisition function as a guide. LaMBO-2 takes a similar approach, but uses a diffusion process as the unconditional model, and modifies a Langevin sampling de-noising process with an acquisition function guide.

\begin{table}[tb]
    \scriptsize
    \centering
    \begin{tabular}{r|c|c|c|c|c|c|c|c|c|c|}
    \textbf{Method} & \rot{Rare $\obs \in \solnspace$ \ref{req:rare}} & \rot{Sequential \ref{req:seq}} & \rot{Discrete $\obsspace$ \ref{req:disc}} &\rot{Batch $\{\obs_{bt}\}^B_{b=1}$ \ref{req:batch}}& \rot{Generative $\qrobc{\obs}{\qparam}$ \ref{req:gen}} & \rot{Guaranteed \ref{desi:guar}} & \rot{Gradient descent \ref{desi:gd}} & \rot{Scalable \ref{desi:scale}} & \rot{General acq./reward fn.} & \rot{Amortization} \\
    \hline
    \gls{bopr}~\citep{daulton2022bayesian}& \xmark & \cmark & \cmark & \xmark & -- & \cmark & \cmark & \xmark & \cmark & -- \\
    \gls{bore}~\citep{tiao2021bore}& \xmark & \cmark & -- & \xmark & -- & \cmark & \cmark & \cmark & \xmark & -- \\
    Batch \gls{bore}~\citep{oliveira2022batch}& \cmark & \cmark & \xmark & \cmark & \cmark & \cmark & \cmark & \cmark & \xmark & \cmark \\
    \gls{dbas}~\citep{brookes2018design}& \cmark & -- & \cmark & \cmark & \cmark & \xmark & \cmark & \cmark & \xmark & \cmark \\
    \gls{cbas}~\citep{brookes2019conditioning}& \cmark & -- & \cmark & \cmark & \cmark & \xmark & \cmark & \cmark & \xmark & \cmark \\
    Amortized BO~\citep{swersky2020amortized}& \xmark & \cmark & \cmark & \cmark & \cmark & \xmark & \cmark & \cmark & \cmark & \cmark \\
    GFlowNets~\citep{jain2022biological}& \xmark & \cmark & \cmark & \cmark & \cmark & \xmark & \cmark & \cmark & \cmark & \cmark \\
    DynaPPO~\citep{angermueller2019model}& \xmark & \cmark & \cmark & \cmark & \cmark & \xmark & \cmark & -- & \cmark & \cmark \\
    AdaLead~\citep{sinai2020adalead}& \xmark & \cmark & \cmark & \cmark & \xmark & \xmark & \xmark & -- & \xmark & \xmark \\
    \gls{pex}~\citep{ren2022proximal}& \xmark & \cmark  & \cmark & \cmark & \xmark & \xmark & \xmark & -- & \xmark & \xmark \\
    GGS~\citep{kirjner2024improving}& \xmark & \xmark & \cmark & \cmark & \cmark & \xmark & \xmark & \xmark & \xmark & \xmark \\
    \gls{lso} e.g.~\citep{tripp2020sample}& \xmark & \cmark & \cmark & \xmark & \cmark & \xmark & \cmark & -- & \cmark & -- \\
    LaMBO~\citep{stanton2022accelerating}& \cmark & \cmark & \cmark & \cmark & \cmark & \xmark & \cmark & -- & \cmark & \cmark \\
    LaMBO-2~\citep{gruver2023protein}& \cmark & \cmark & \cmark & \cmark & \cmark & \xmark & \cmark & \cmark & \cmark & \cmark \\
    \gls{vsd} (ours)  & \cmark & \cmark& \cmark & \cmark & \cmark & \cmark & \cmark & \cmark & \xmark & \cmark \\
    \end{tabular}
    \hspace{0.9cm}
    \caption{Feature table of competing methods: \cmark~has feature, \xmark~does not have feature, -- partially has feature, or requires only simple modification. We follow~\citet{swersky2020amortized} in their definition of amortization referring to the ability to use $\qrobc{\obs}{\qparam^*_{t-1}}$ for warm-starting the optimization of $\qparam_t$.}
    \label{tab:compare}
    \vspace{-2em}
\end{table}

\section{Experiments}
\label{sec:experiments}

Firstly we test our method, \gls{vsd}, on its ability to generate complex, structured designs, $\obs$, in a single round by training it to generate a subset of handwritten digits from flattened MNIST images \citep{lecun1998gradient} in \autoref{sub:cde}.
We then compare \gls{vsd} on two sequence design tasks against existing baseline methods.
The first of these tasks (\autoref{sub:fitnesslandscapes}) is to generate as many unique, fit sequences as possible using the datasets DHFR \citep{papkou2023rugged}, TrpB \citep{johnston2024combinatorially} and TFBIND8 \citep{barrera2016survey}. These datasets contain near complete evaluations of $\obsspace$, and to our knowledge DHFR and TrpB are novel in the machine learning literature. The second (\autoref{sub:bbo}) is a more traditional black-box optimization task of finding the maximum of an unknown function; using datasets AAV \citep{bryant2021deep}, GFP \citep{sarkisyan2016local} and the  biologically inspired Ehrlich functions~\citep{stanton2024closed}.
The corresponding datasets involve $\card{\acidspace} \in \{4,20\}$, $4 \leq M \leq 237$ and $65,000 < \card{\obsspace} < 20^{237}$. We discuss the settings and properties of these datasets in greater detail in \autoref{app:expdet}.
For the biological sequence experiments we run a predetermined number of experimental rounds, $T=10$ or $32$ for the Ehrlich functions. We set the batch size to $B=128$, and use five different seeds for random initialization. We compare against \gls{dbas} \citep{brookes2018design}, \gls{cbas} \citep{brookes2019conditioning}, AdaLead \citep{sinai2020adalead}, and  \gls{pex} \citep{ren2022proximal} -- all of which we have adapted to use a \gls{cpe}, \gls{bore} \citep{tiao2021bore} -- which we have adapted to use the score function gradient estimator, and a na\"ive baseline that uses random samples from the prior, $\probc{\obs}{\data_0}$. To reduce confounding, all methods share the same surrogate model, acquisition functions, priors and variational distributions where possible. We compare against LaMBO-2 \citep{gruver2023protein} on the Ehrlich functions, it uses its own surrogate and generative models.

\subsection{Conditional Generation of Handwritten Digits}
\label{sub:cde}

Our motivating application for \gls{vsd} is to model the space of fit DNA and protein sequences, which are string-representations of complex 3-dimensional structures. In this experiment we aim to demonstrate, by analogy, that \gls{vsd} can generate sequences that represent 2-dimensional structures.
For this task, we have chosen to `unroll' (reverse the order of every odd row, and flatten) down-scaled ($14\times14$ pixel, 8-bit) MNIST \citep{lecun1998gradient} images into sequences, $\obs$, where $M$ = 196 and $\card{\acidspace} = 8$. We then train \gls{lstm} \gls{rnn} and decoder-only causal transformer generative models on the entire MNIST training set by \gls{ml}. These generative distributions are used as the prior models, $\probc{\obs}{\data_0}$, for \gls{vsd} and we detail their form in Appendix \ref{app:var_dist}.
\begin{figure}[bth]
\centering
\subcaptionbox{LSTM Prior\label{sfig:cde_lstm_p}}
    {\includegraphics[width=.24\textwidth]{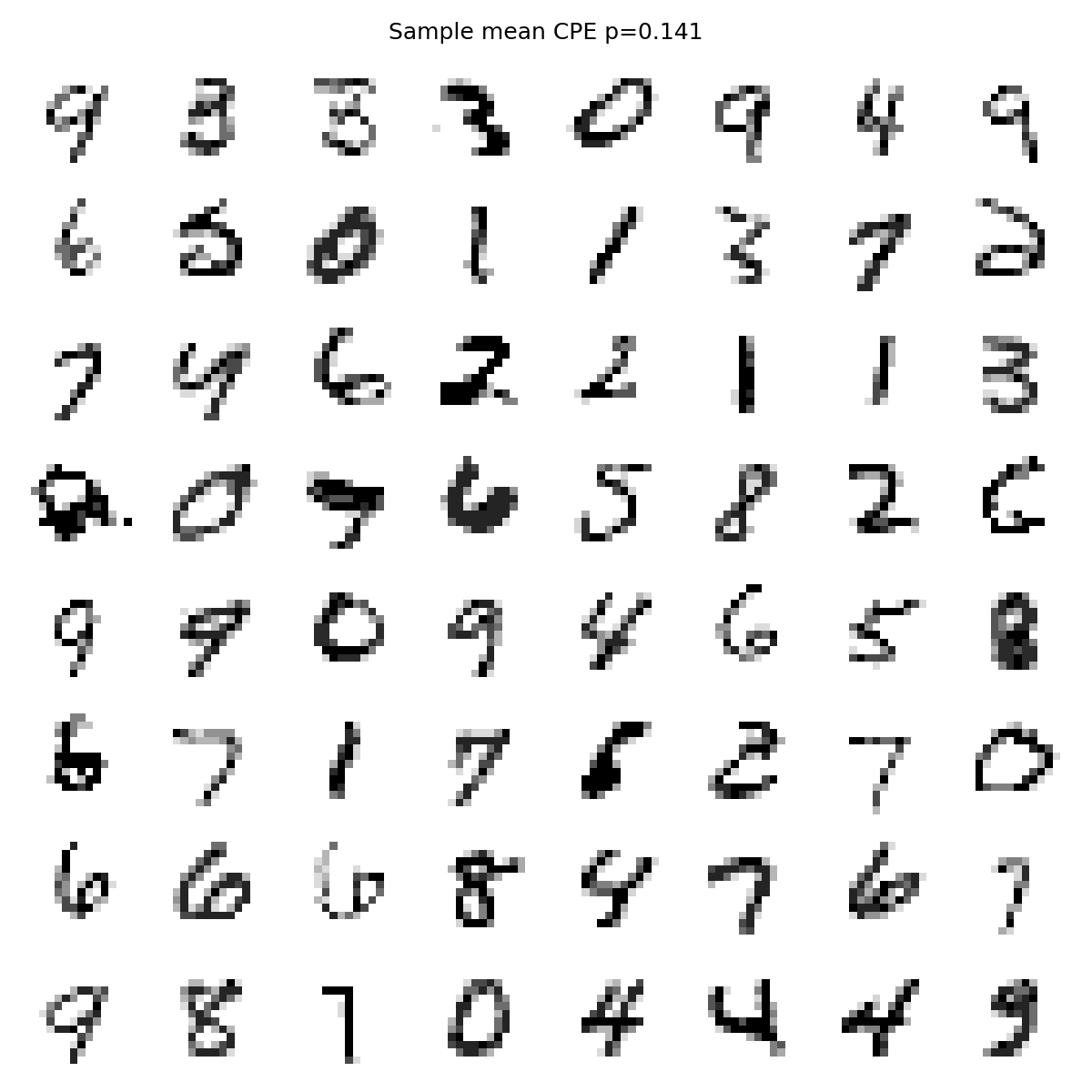}}
\subcaptionbox{Transformer Prior\label{sfig:cde_dtfm_p}}
    {\includegraphics[width=.24\textwidth]{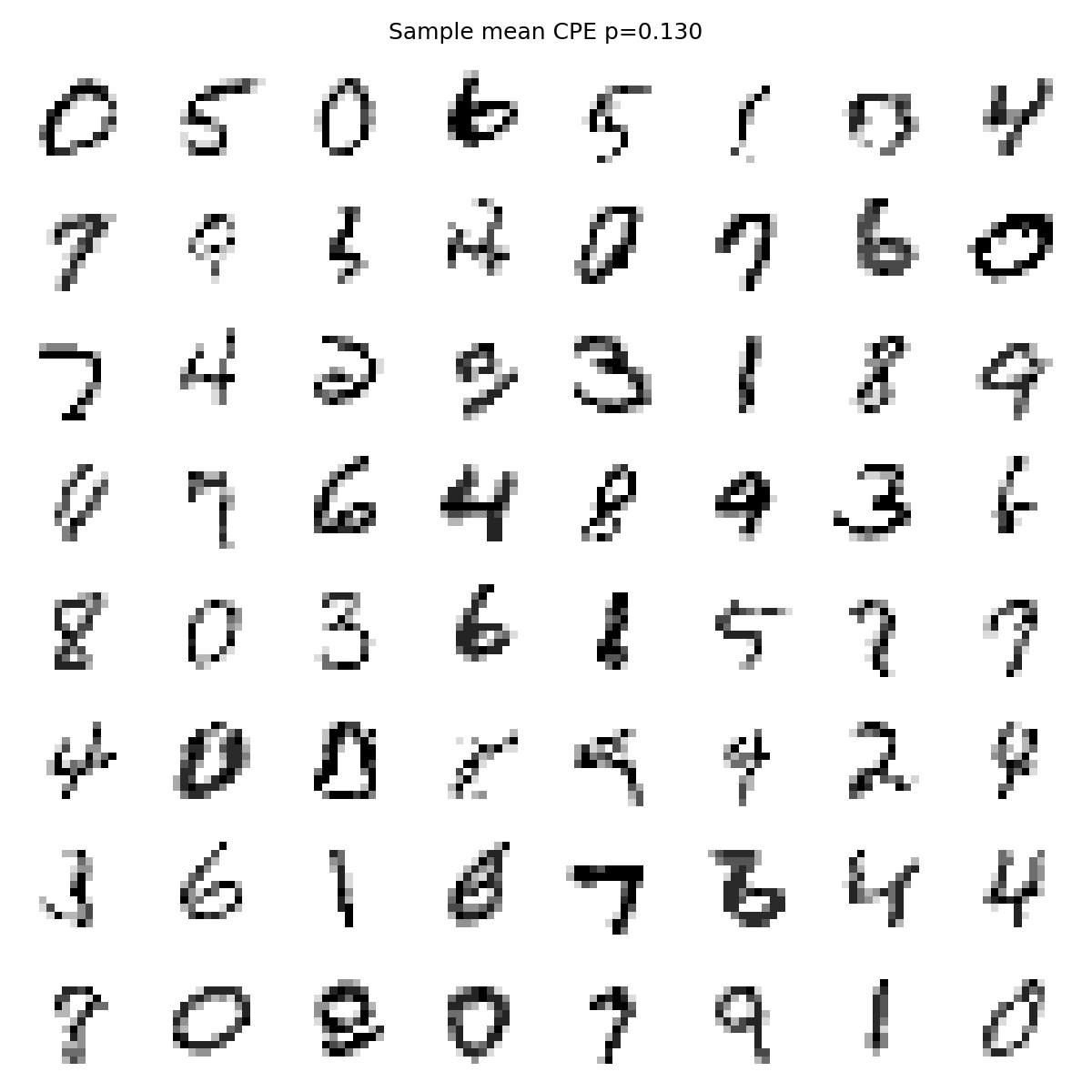}}
\subcaptionbox{LSTM Posterior\label{sfig:cde_lstm}}
    {\includegraphics[width=.24\textwidth]{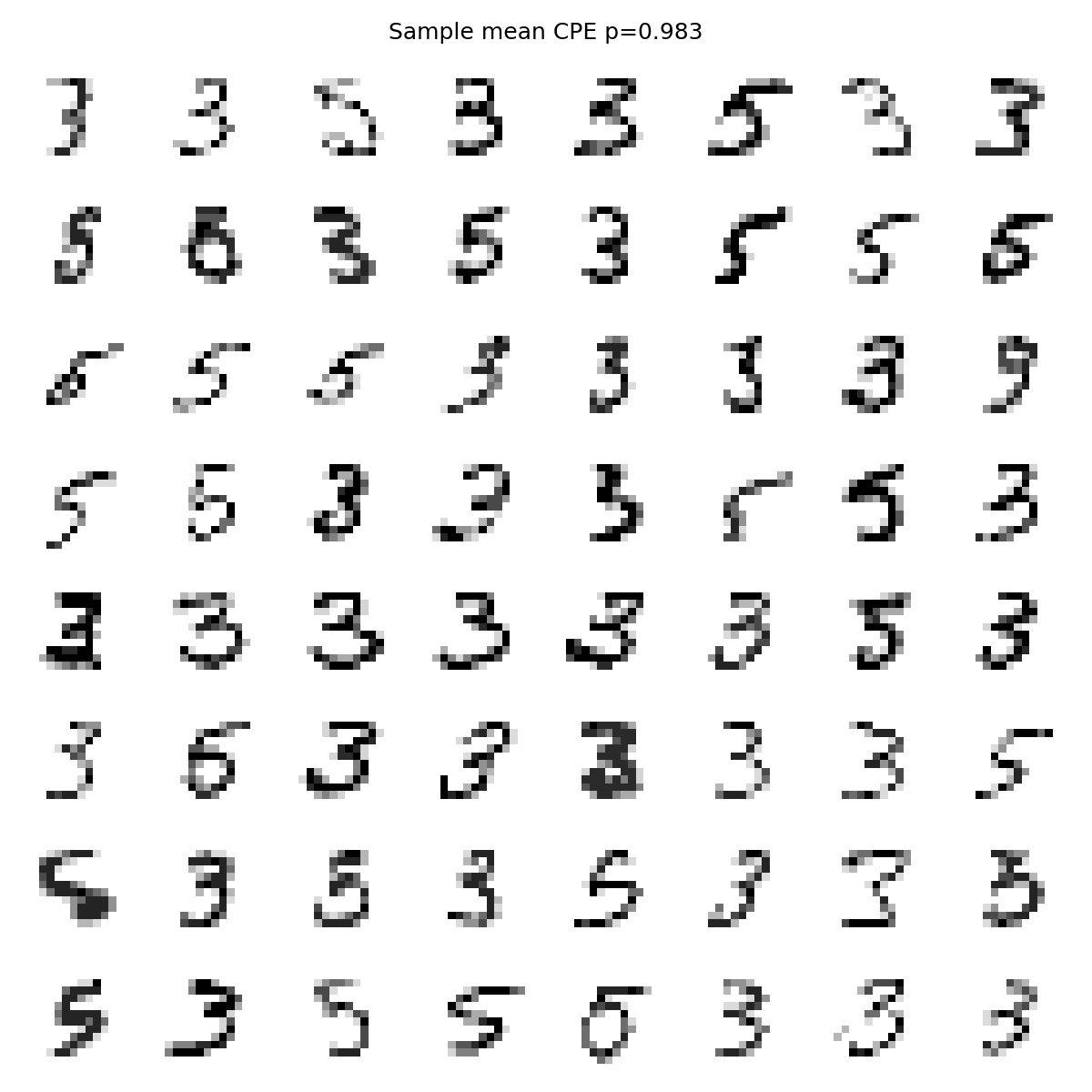}}
\subcaptionbox{Transformer Posterior\label{sfig:cde_dtfm}}
    {\includegraphics[width=.24\textwidth]{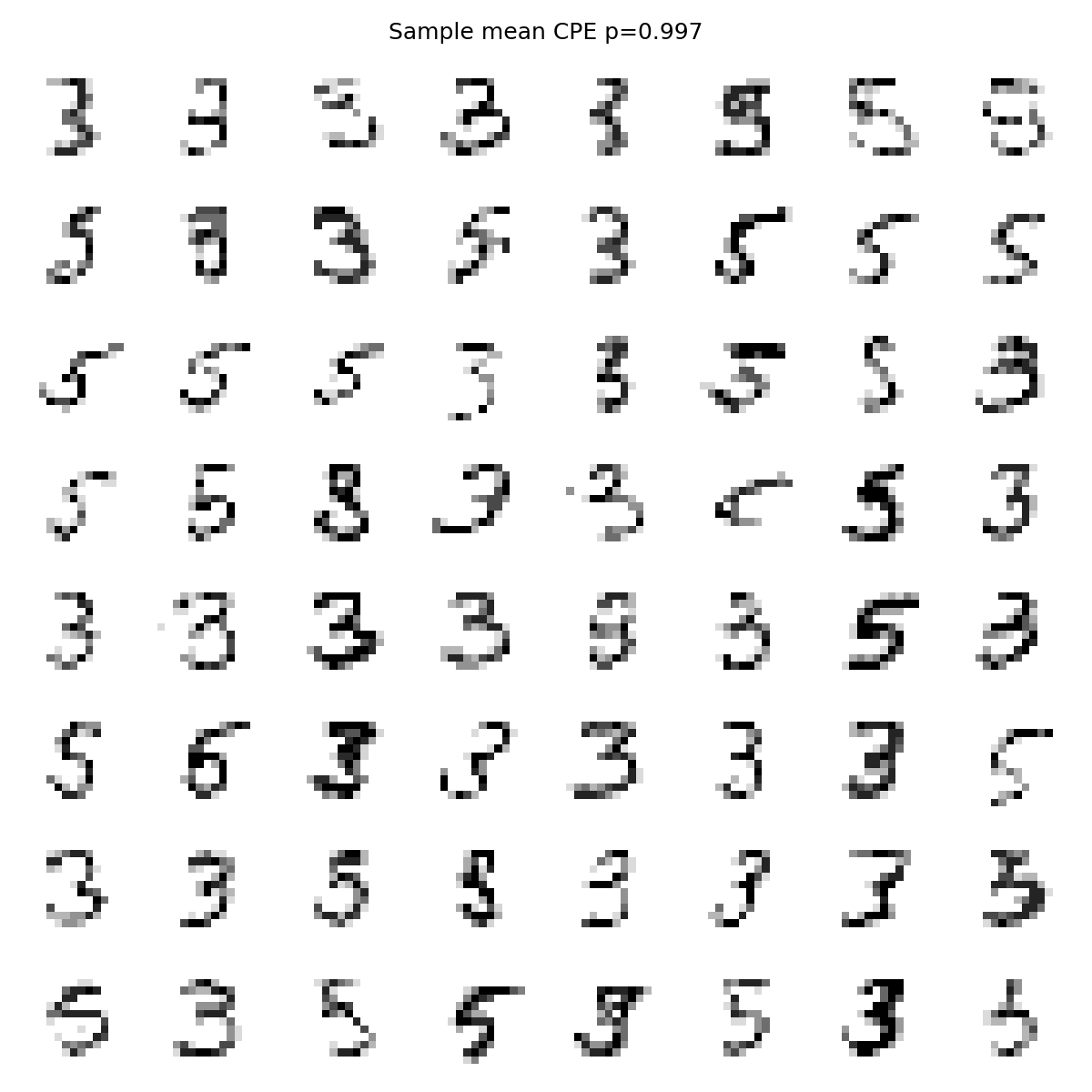}}
\caption{(\subref{sfig:cde_lstm_p}) and (\subref{sfig:cde_dtfm_p}) are samples from the LSTM and transformer priors, respectively. (\subref{sfig:cde_lstm}) and (\subref{sfig:cde_dtfm}) show samples from the \gls{lstm} and transformer \gls{vsd} variational distributions respectively. We also report the samples mean scores according to the \gls{cpe} probabilities.}
\label{fig:cde}
\end{figure}%
The task is then to use \gls{vsd} in one round to estimate the posterior $\probc{\obs}{\tar \in \{3, 5\}}$ using a \gls{cpe} trained on labels $\labl_n = \indic{\tar_n \in \{3, 5\}}$. We use a convolutional architecture for the \gls{cpe} given in Appendix \ref{sub:cpe}, and it achieves a test balanced accuracy score of $\sim99\%$. We parameterize the variational distributions, $\qrobc{\obs}{\qparam}$, in the same way as the priors, and initialize these distribution parameters from the prior distribution parameters. During training with \gls{elbo} the prior distribution parameters are locked, and we run training for 5000 iterations. This is exactly lines 8 and 9 in \autoref{alg:optloop}. Samples are visualized from the resulting variational distributions with the corresponding priors in \autoref{fig:cde}.
We see that the prior \gls{lstm} and transformer are able to generate convincing digits once the sampled sequences are `re-rolled', and that \gls{vsd} is able to effectively refine these distributions, even though it does not have access to any data directly -- only scores from the \gls{cpe}. Both the \gls{lstm} and transformer yield qualitatively similar results, and have similar  mean scores from the \gls{cpe}.

\subsection{Fitness Landscapes}
\label{sub:fitnesslandscapes}

In this setting we wish to find fit sequences $\obs \in \solnspace_\textrm{SLS}$, so we fix $\thresh$ over all rounds. We only consider the combinatorially (near) complete datasets to avoid any pathological behavior from relying on machine learning oracles \citep{surana2024overconfident}. Results are presented in \autoref{fig:fl_res}. The primary measures by which we compare methods are precision, recall and performance,
\begin{align}
    \precision_t &= \frac{1}{\min\{tB, \card{\solnspace}\}}\sum^t_{r=1} \sum^B_{b=1} \indic{\tar_{br} > \thresh}
        \cdot \indic{\obs_{br} \notin \seenspace_{b-1,r}},
    \label{eq:prec} \\
    \recall_t &= \frac{1}{\min\{TB, \card{\solnspace}\}}\sum^t_{r=1} \sum^B_{b=1} \indic{\tar_{br} > \thresh}
        \cdot \indic{\obs_{br} \notin \seenspace_{b-1,r}},
    \label{eq:rec} \\
    \performance_t &= \sum^t_{r=1} \sum^B_{b=1} \tar_{br}
        \cdot \indic{\obs_{br} \notin \seenspace_{b-1,r}}.
    \label{eq:perf}
\end{align}%
Here $\seenspace_{br} \subset \obsspace$ is the set of experimentally queried sequences by the $b$th batch member of the $r$th round, including the initial training set.
These measures are comparable among probabilistic and non probabilistic methods. Precision and recall measure the ability of a method to efficiently explore $\solnspace$, where $\min\{tB, \card{\solnspace}\}$ is the size of the selected set at round $t$ (bounded by the number of good solutions), and $\min\{TB, \card{\solnspace}\}$ is the number of positive elements possible in the experimental budget.
Performance measures the cumulative fitness of the unique batch members, but unlike \citet{jain2022biological} we do not normalize this measure.

\begin{figure}[htb]
\centering
\includegraphics[width=.32\textwidth]{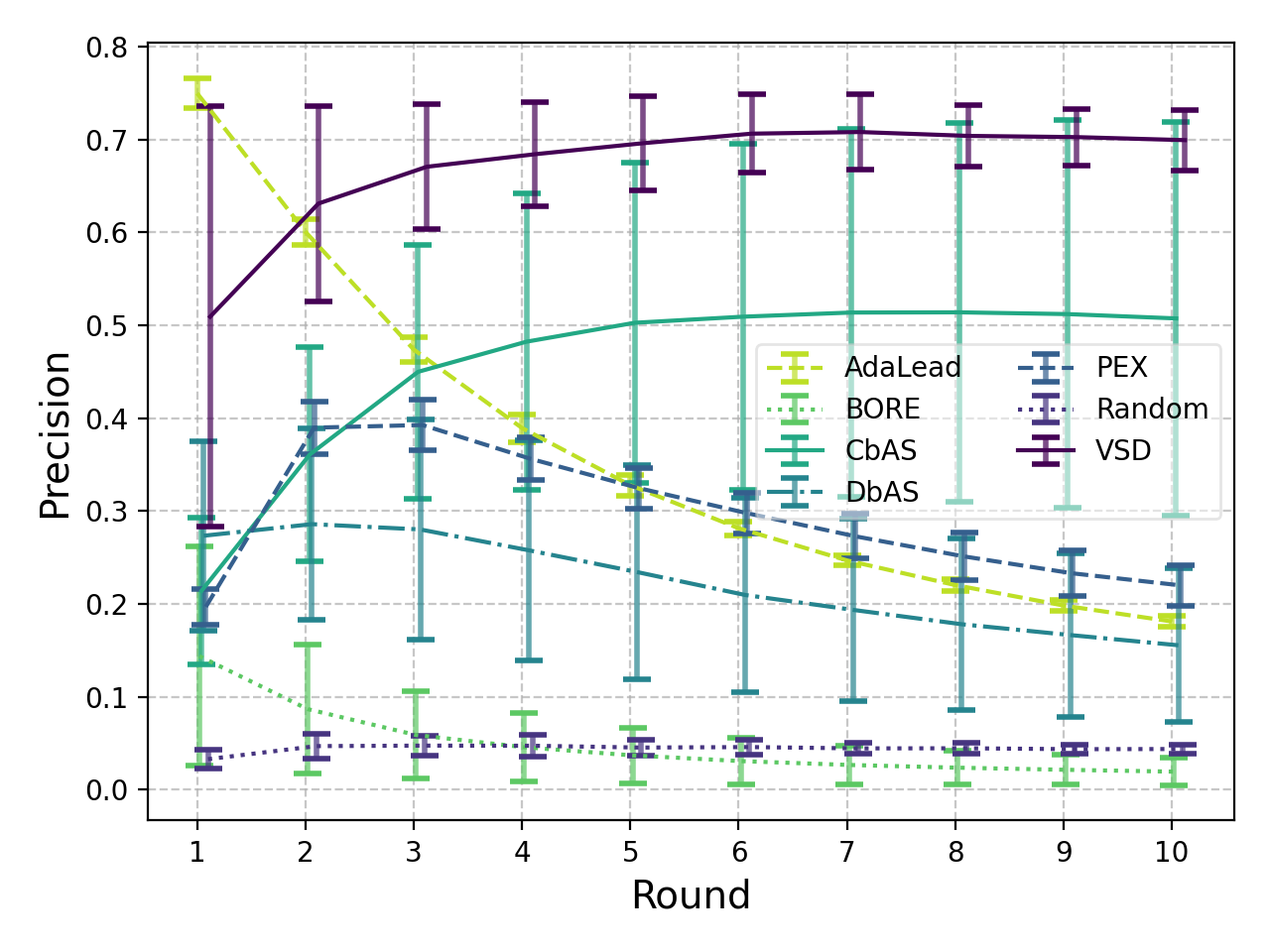}
\includegraphics[width=.32\textwidth]{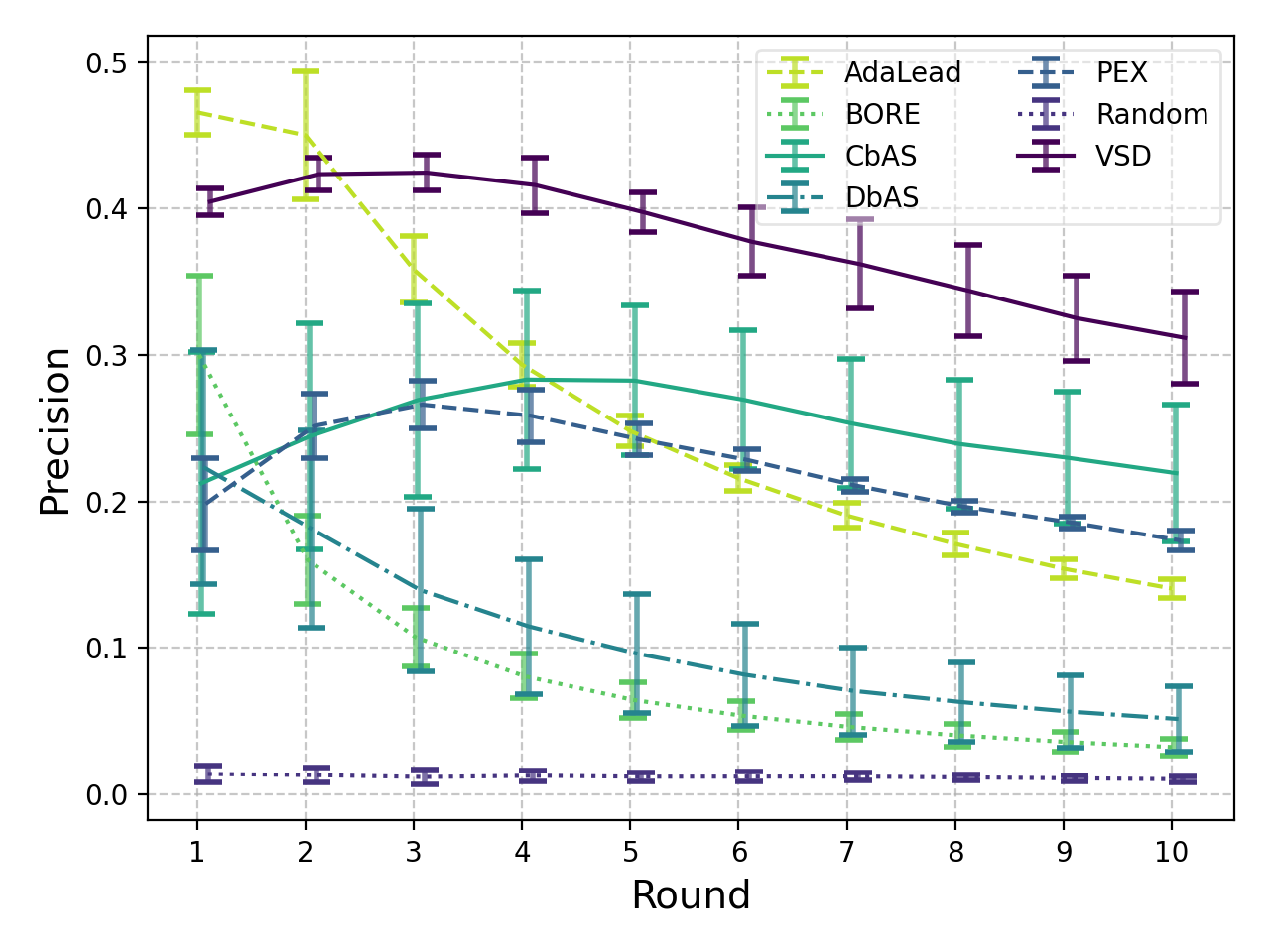}
\includegraphics[width=.32\textwidth]{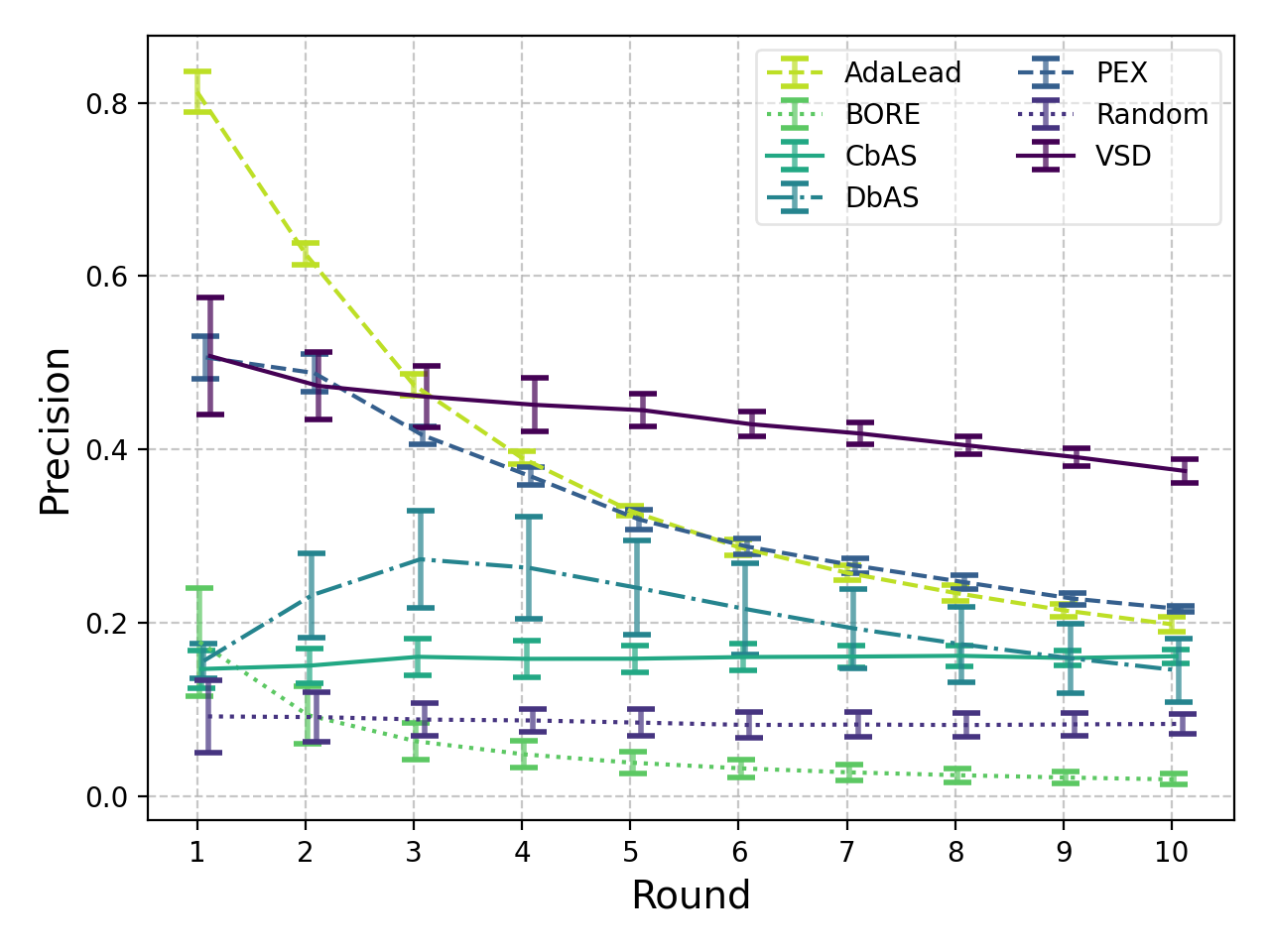} \\
\includegraphics[width=.32\textwidth]{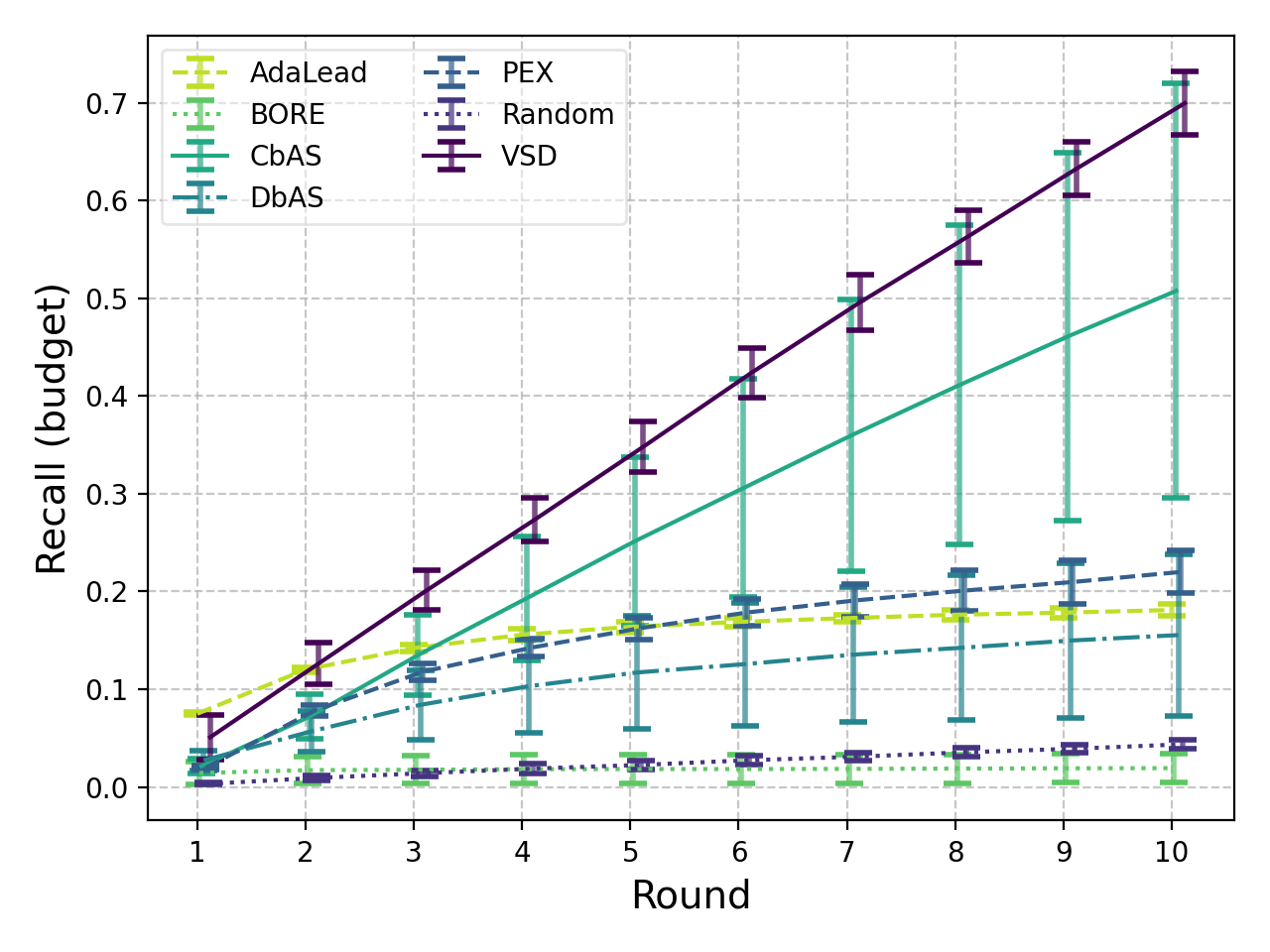}
\includegraphics[width=.32\textwidth]{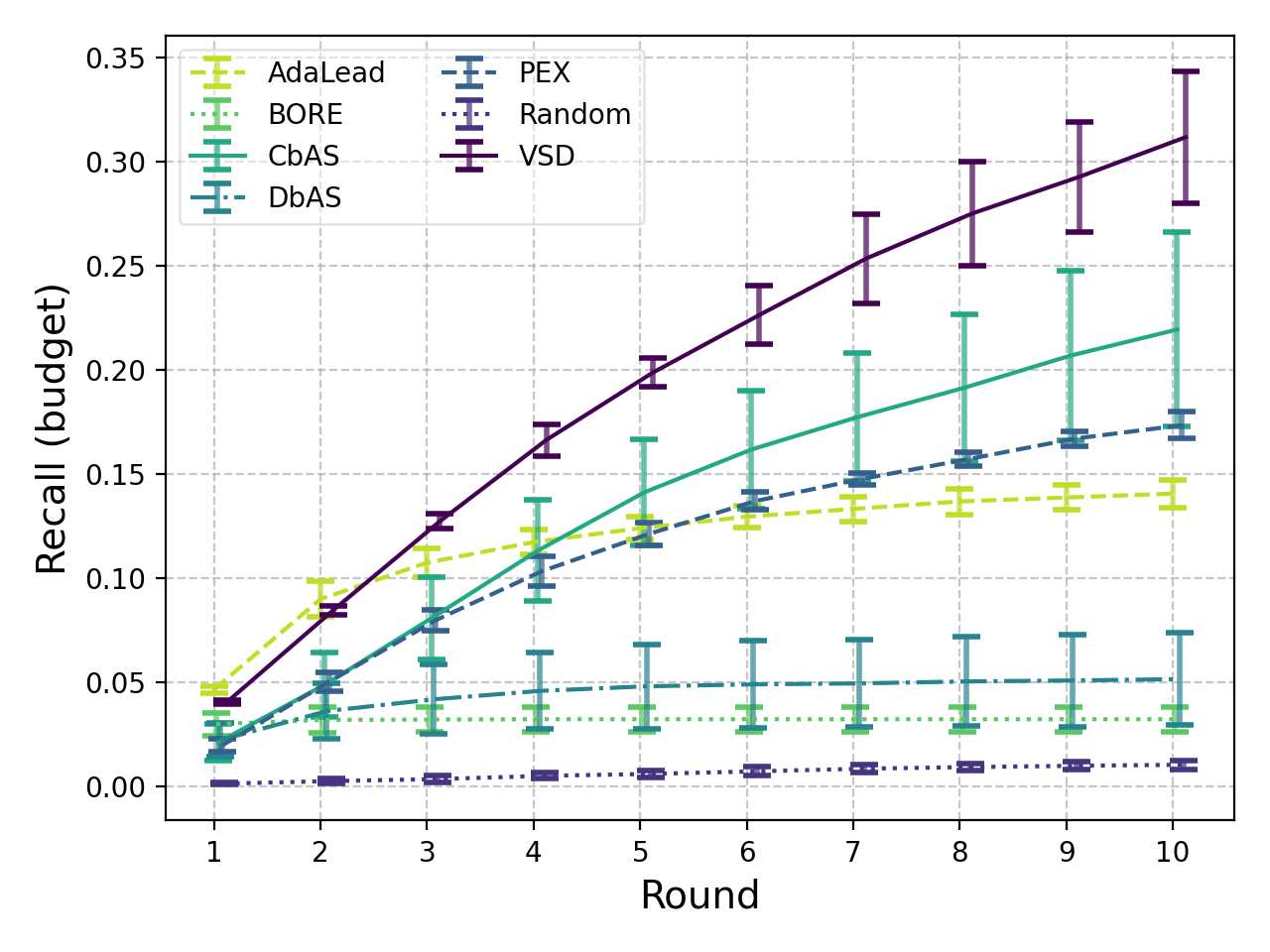}
\includegraphics[width=.32\textwidth]{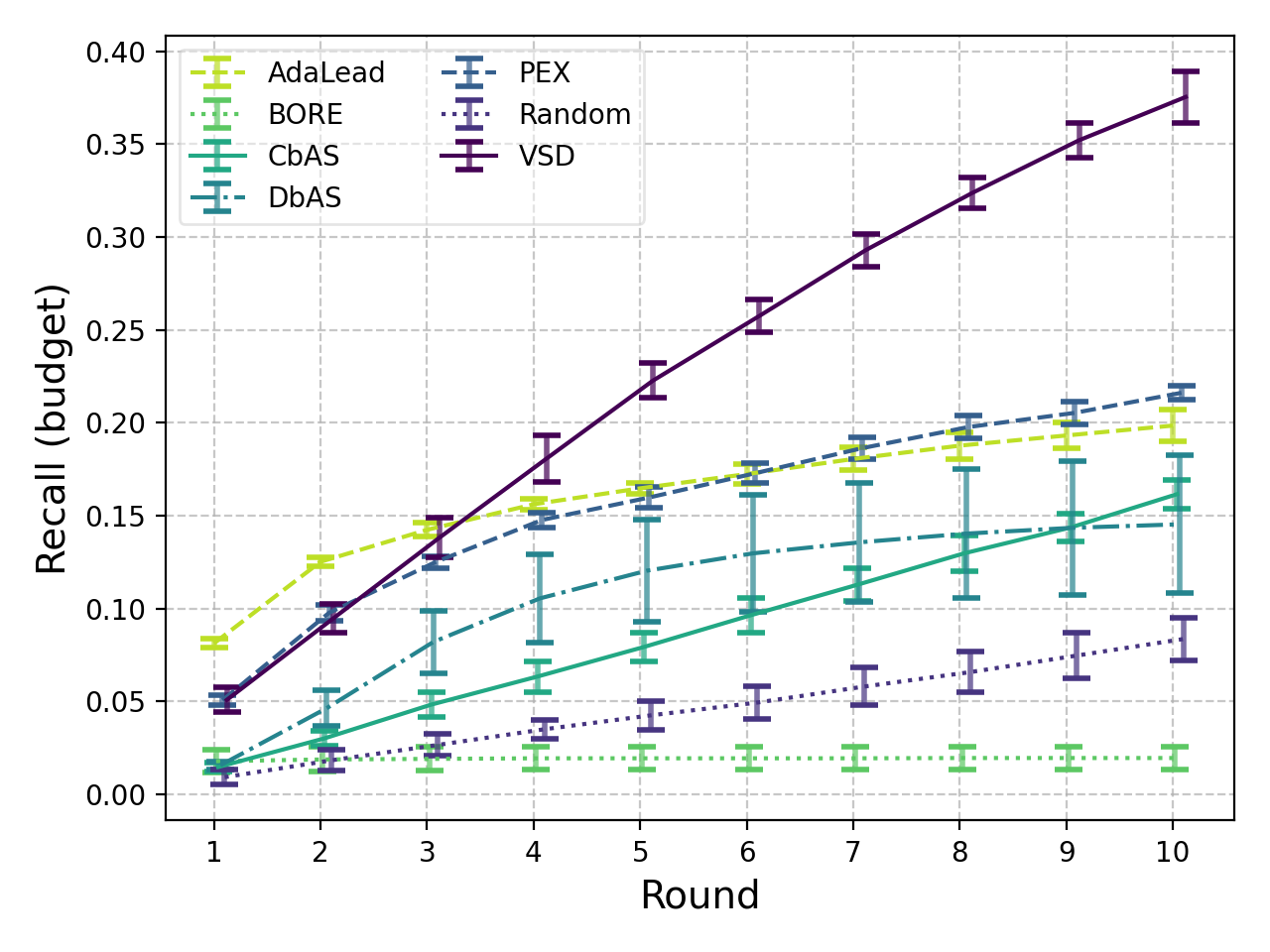} \\
\subcaptionbox{DHFR\label{sfig:dhfr_fl}}
    {\includegraphics[width=.32\textwidth]{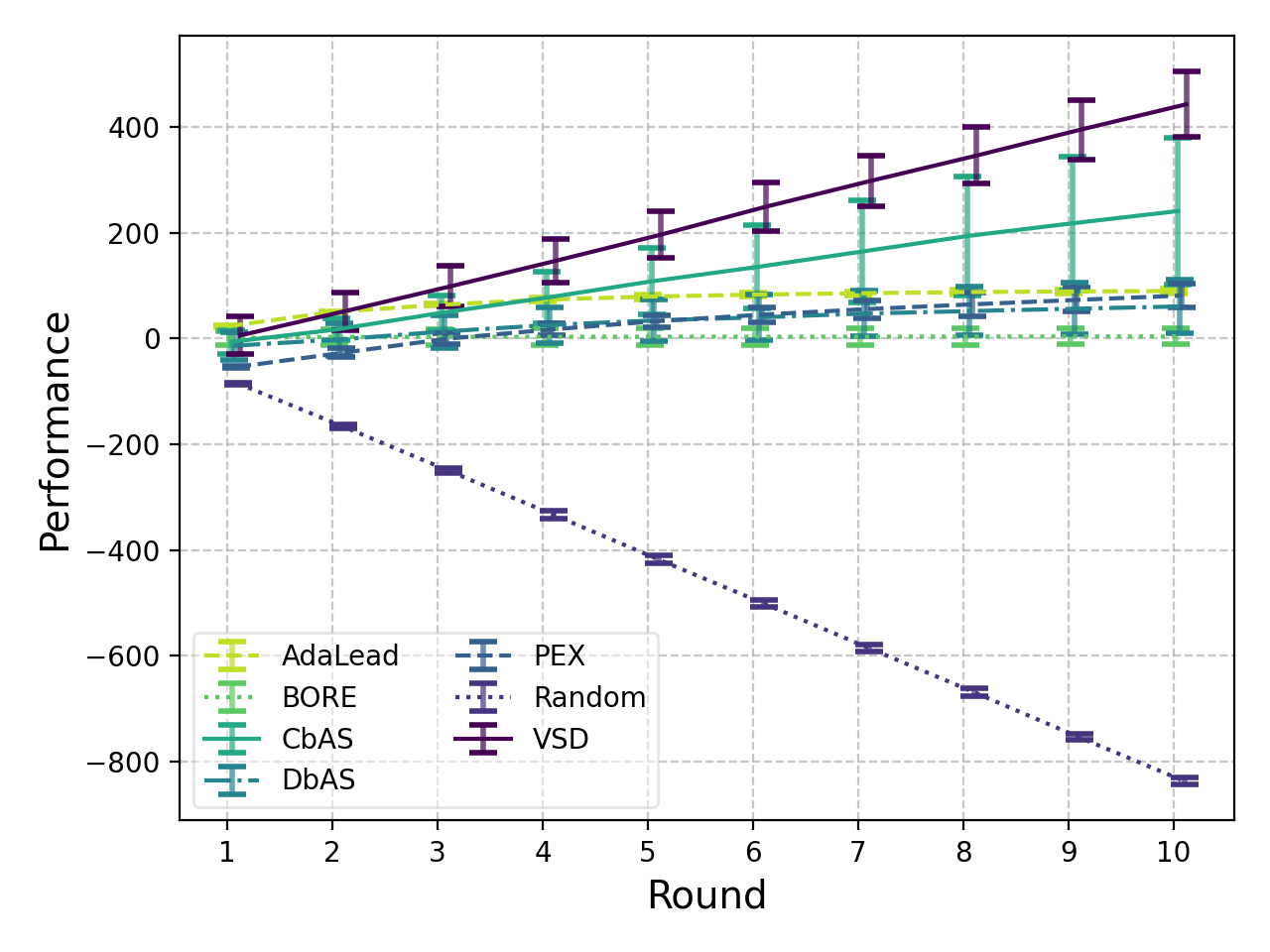}}
\subcaptionbox{TrpB\label{sfig:trpb_fl}}
    {\includegraphics[width=.32\textwidth]{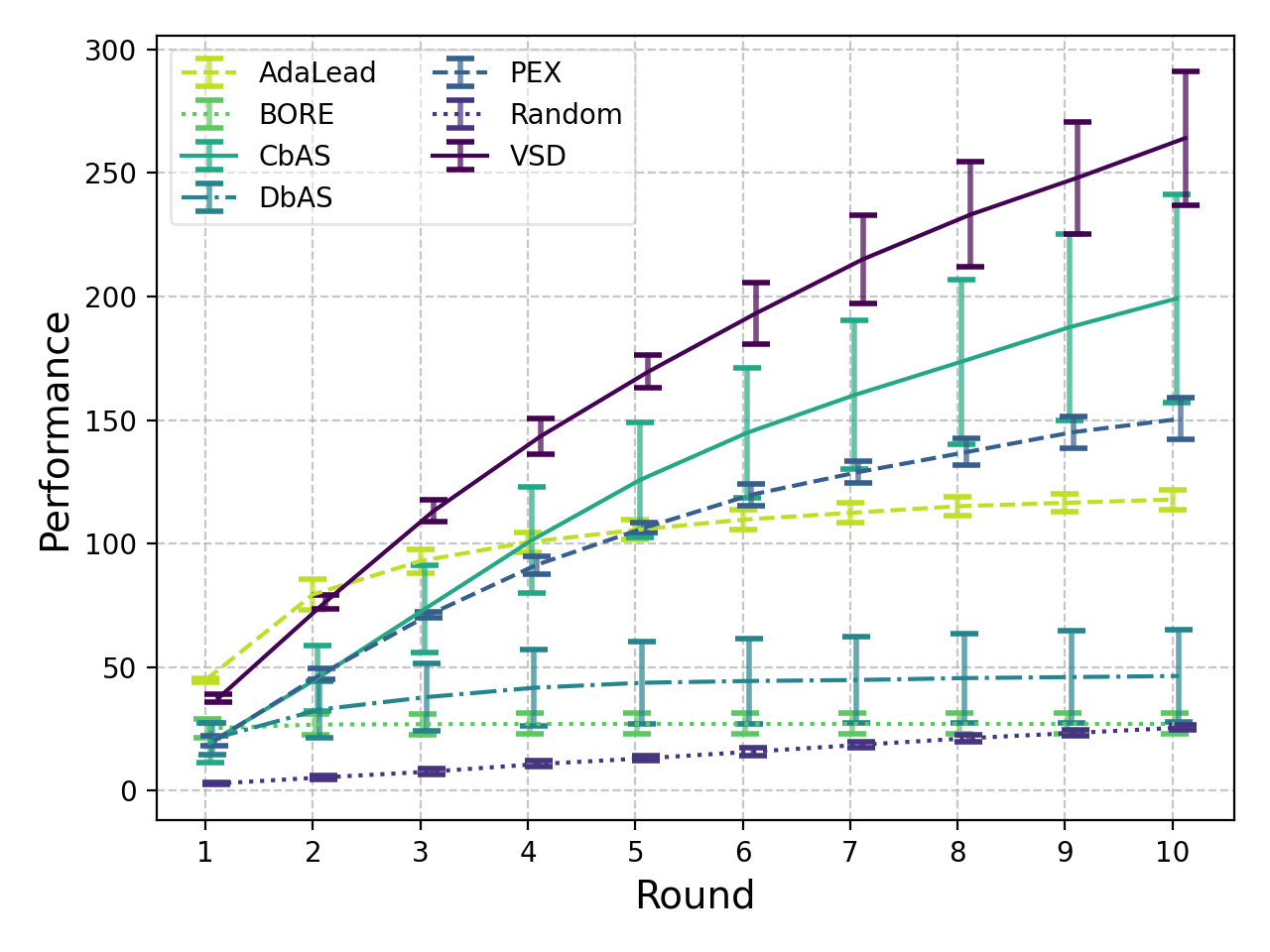}}
\subcaptionbox{TFBIND8\label{sfig:tfbind8_fl}}
    {\includegraphics[width=.32\textwidth]{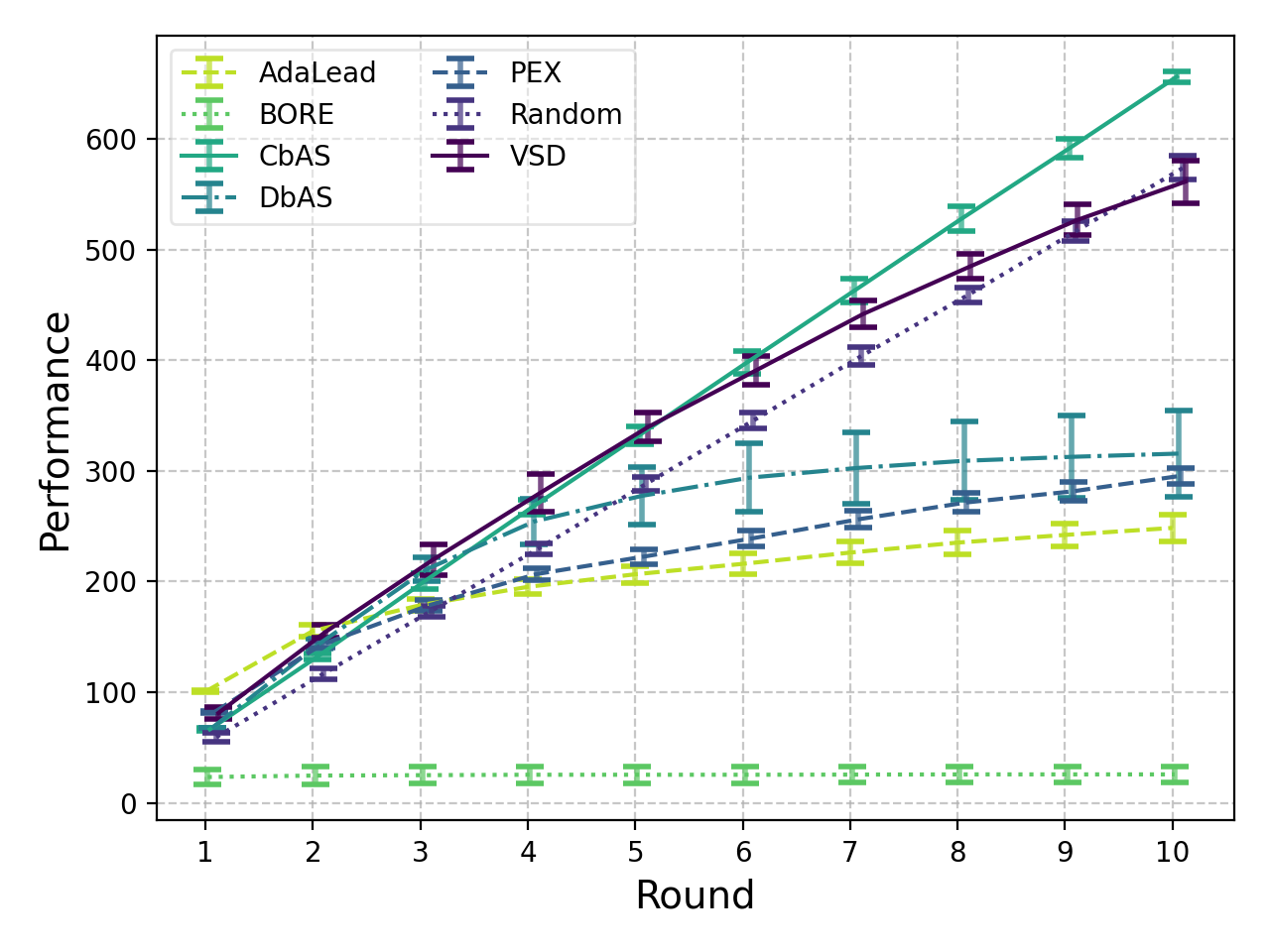}}
\caption{Fitness landscape results. Precision (\autoref{eq:prec}), recall (\autoref{eq:rec}) and performance (\autoref{eq:perf}) -- higher is better -- for the combinatorially (near) complete datasets, DHFR, TrpB and TFBIND8. The random method is implemented by drawing $B$ samples uniformly.
}
\label{fig:fl_res}
\vspace{-1em}
\end{figure}

For exact experimental settings we refer the reader to Appendix \ref{app:fl_settings}. We set $\thresh$ to be that of the wild-type sequences in the DHFR and TrpB datasets, and use $\thresh = 0.75$ for TFBIND8. We find that a uniform prior over sequences, and a mean field variational distribution (\autoref{eq:proposal}) are adequate for these experiments, as is a simple MLP for the \gls{cpe}.
Results are presented in \autoref{fig:fl_res}. \Gls{vsd} is the best performing method by most of the measures. We have found the AdaLead and \gls{pex} evolutionary-search based methods to be effective on lower-dimensional problems (TFBIND8 being the lowest here), however we consistently observe their performance degrading as the dimension of the problem increases. We suspect this is a direct consequence of their random mutation strategies being suited to exploration in low dimensions, but less efficient in higher dimensions compared to the learned generative models employed by \gls{vsd}, \gls{cbas}, and \gls{dbas}. Our modified version of \gls{bore} (which is just the expected log-likelihood component of \autoref{eq:vsdcpe}) performs badly in all cases, and this is a direct consequence of its variational distribution collapsing to a point mass. In a non-batch setting this behavior is not problematic, but shows the importance of the \gls{kl} divergence of \gls{vsd} in this batch setting.
We replicate these experiments in Appendix \ref{sub:gp_fl} using \gls{gp}-\gls{pi}, also backed by our guarantees. In all cases \gls{vsd}'s results remain similar or improve slightly, whereas the other methods results remain similar or degrade. We report on batch diversity scores in Appendix~\ref{sub:diversity}.

\subsection{Black-Box Optimization}
\label{sub:bbo}

In this experiment we use \gls{vsd} on the related task of \gls{bbo} for finding $\solnspace_\textrm{BBO}$.
We set $\thresh_t$ adaptively by specifying it as an empirical quantile $\tilde{Q}^t_\tar$ of the observed target values at round $t$,
\begin{align}
    \thresh_t = \tilde{Q}^t_\tar(\gamma_t{=}p_{t-1}^\eta)
    \label{eq:annealedthresh}
\end{align}
where $p_{t-1}$ is a percentile from the previous round, and $\eta \in [0, 1]$ is an annealing parameter for $\thresh_t$ that trades off exploration and exploitation.
%
Performance is measured by simple regret $\regretsim_t$, which quantifies the fitness gap between the globally optimal design and the best design found,
\begin{align}
    \regretsim_t = \tar^* - \max_\tar \{\tar_{bi}\}^{B,t}_{b=1, i=1}.
    \label{eq:simregret}
\end{align}
Here $\tar^*$ is the fitness value of the globally optimal sequence $\obs^*$.
We use the higher dimensional AAV ($\tar^*{=}19.54$), GFP ($\tar^*{=}4.12$) and Ehrlich functions ($\tar^*{=}1$) datasets/benchmarks to show that \gls{vsd} can scale to higher dimensional problems. $\obsspace$ of AAV and GFP is completely intractable to fully explore experimentally, and so we use a predictive oracle trained on all the original experimental data as the ground-truth black-box function.
We use the CNN-based oracles from \cite{kirjner2024improving} for these experiments. However, we note here that some oracles used in these experiments do not predict well out-of-distribution~\citep{surana2024overconfident}, which limits their real-world applicability. The Ehrlich functions \citep{stanton2024closed} are challenging biologically inspired closed-form simulations that cover all $\obsspace$. We compare against a \gls{ga}, \gls{cbas} and LaMBO-2~\citep{gruver2023protein} for sequences of length $M=\{15,32,64\}$ using the \textsc{poli} and \textsc{poli-baselines} benchmarks and baselines software \citep{gonzalez-duque2024poli}. For these experiments we use CNNs for the \gls{cpe}s -- all experimental settings are in Appendix \ref{app:bbo_settings}.

\begin{figure}[htb]
\centering
\rotatebox{90}{\hspace{1.5cm}GFP}
\includegraphics[width=.32\textwidth]{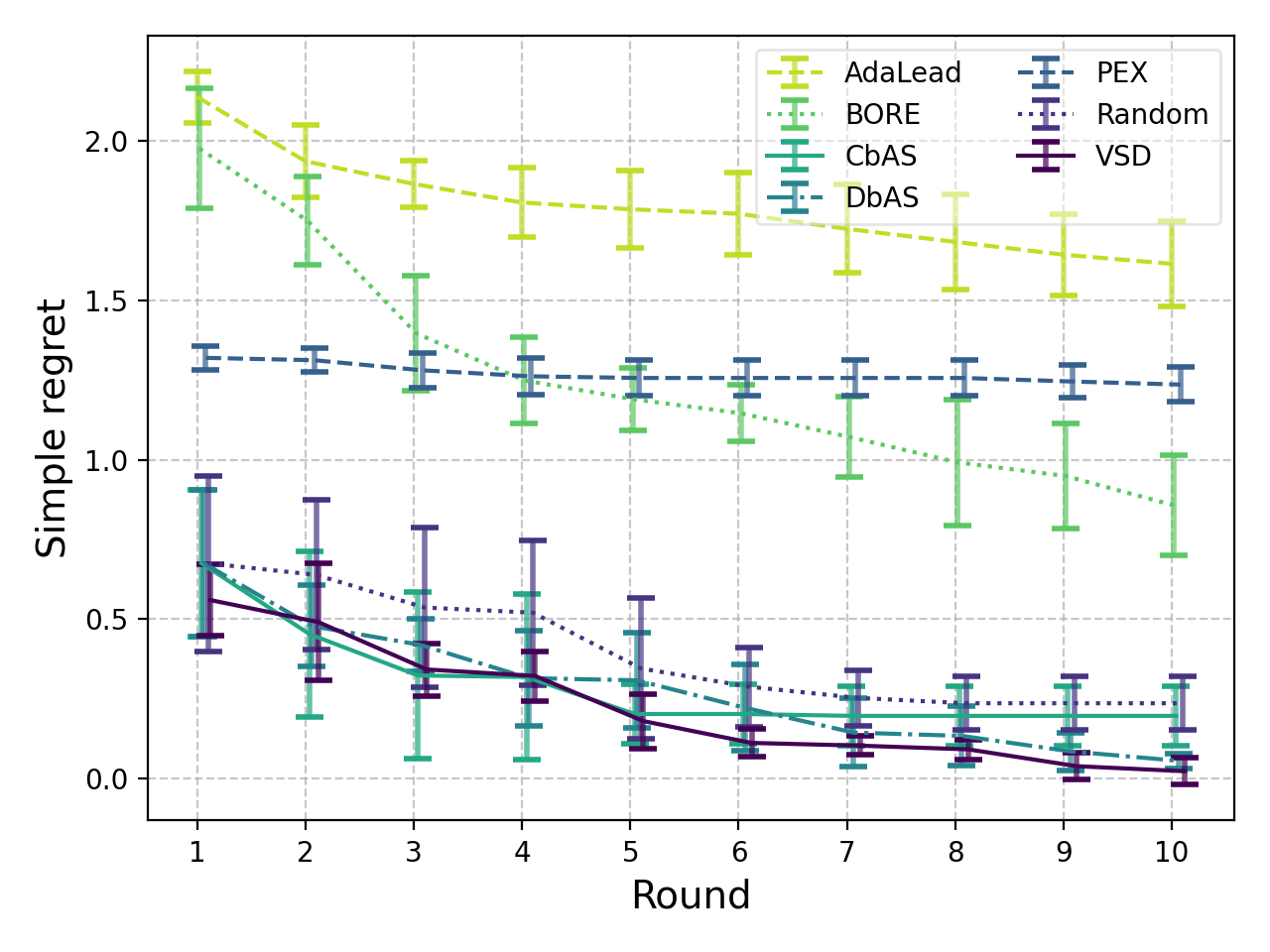}
\includegraphics[width=.32\textwidth]{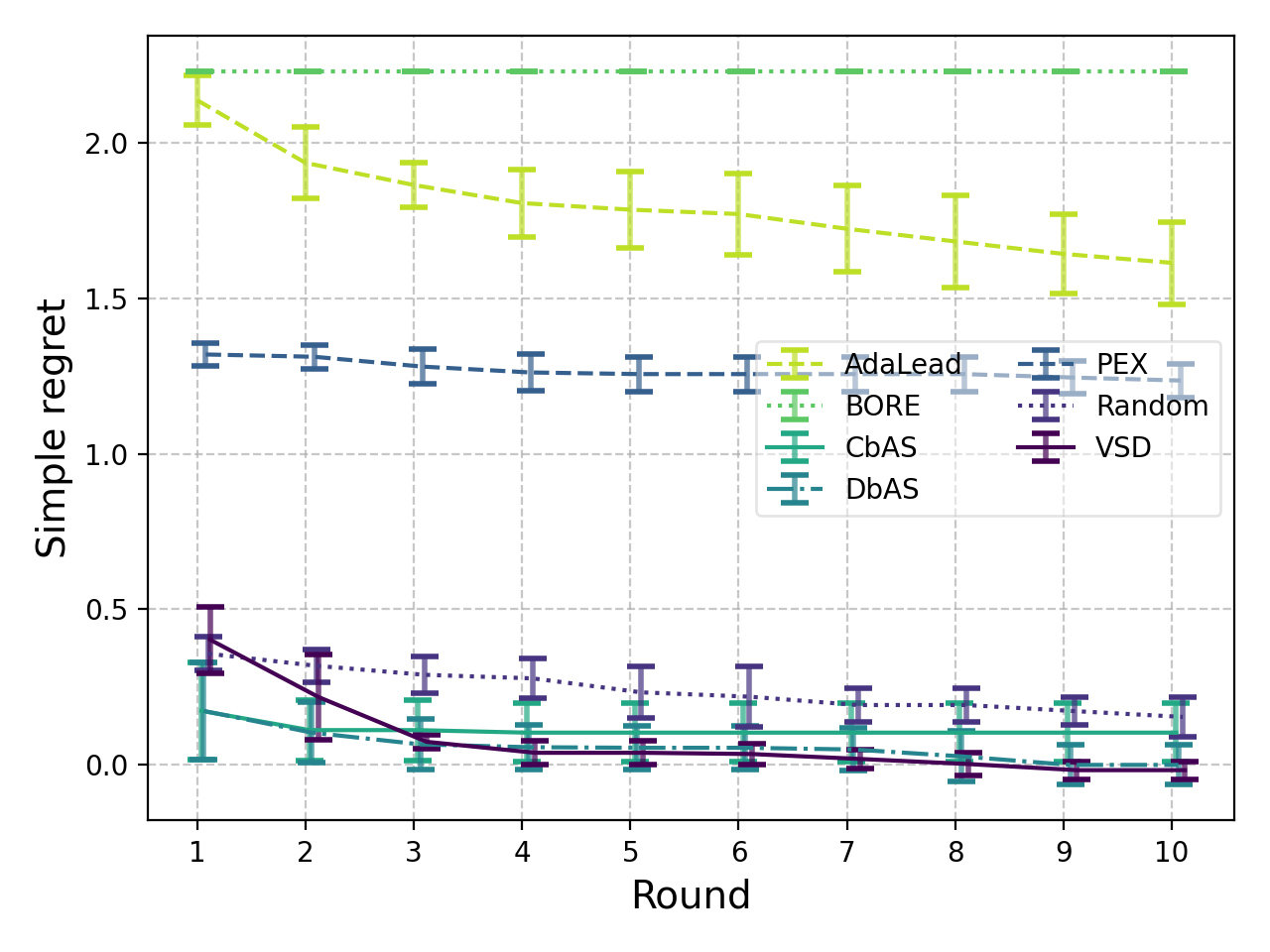}
\includegraphics[width=.32\textwidth]{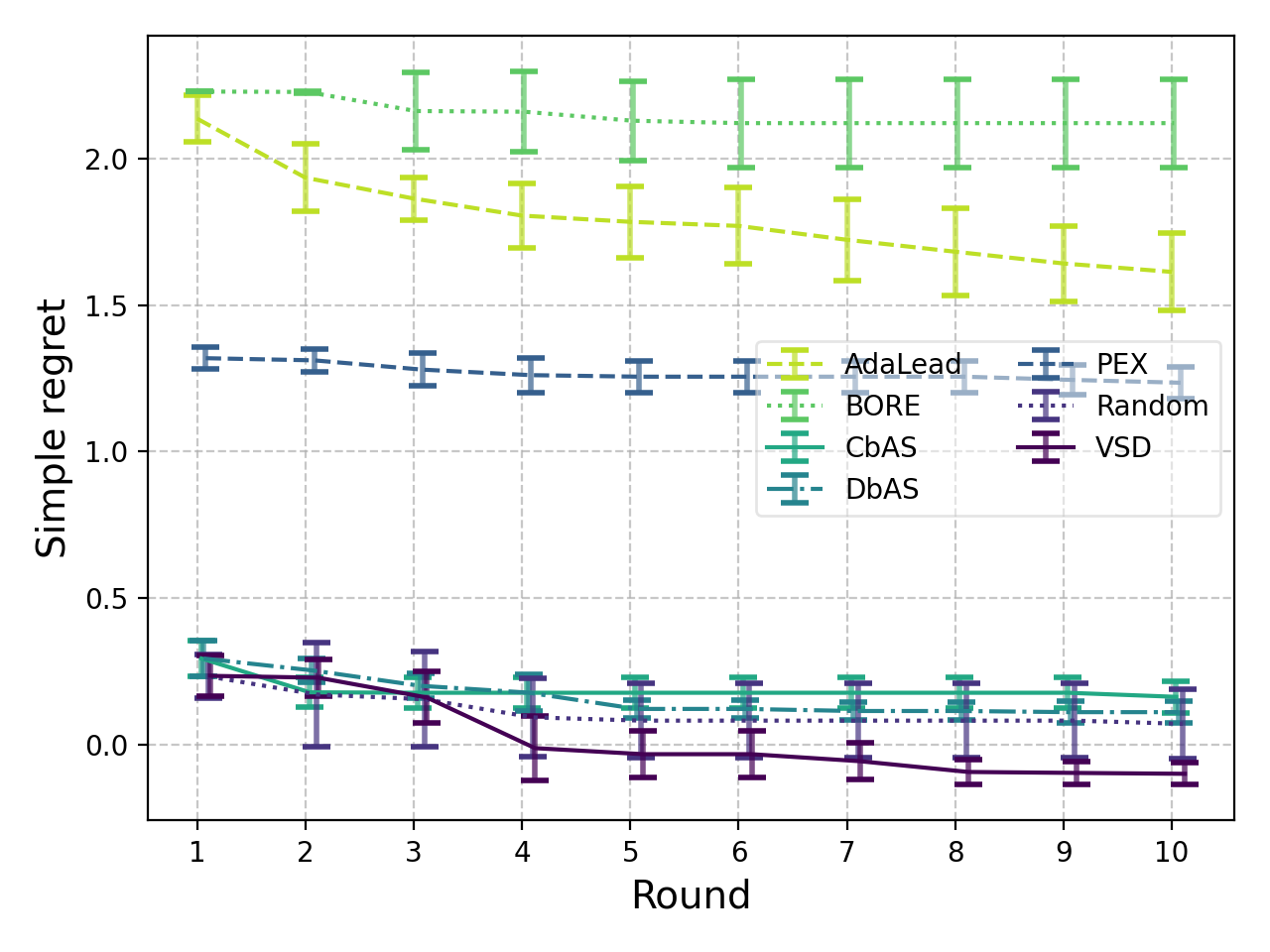} \\
\rotatebox{90}{\hspace{1.5cm}AAV}
\subcaptionbox{Independent\label{sfig:aav_sr_mc}}
    {\includegraphics[width=.32\textwidth]{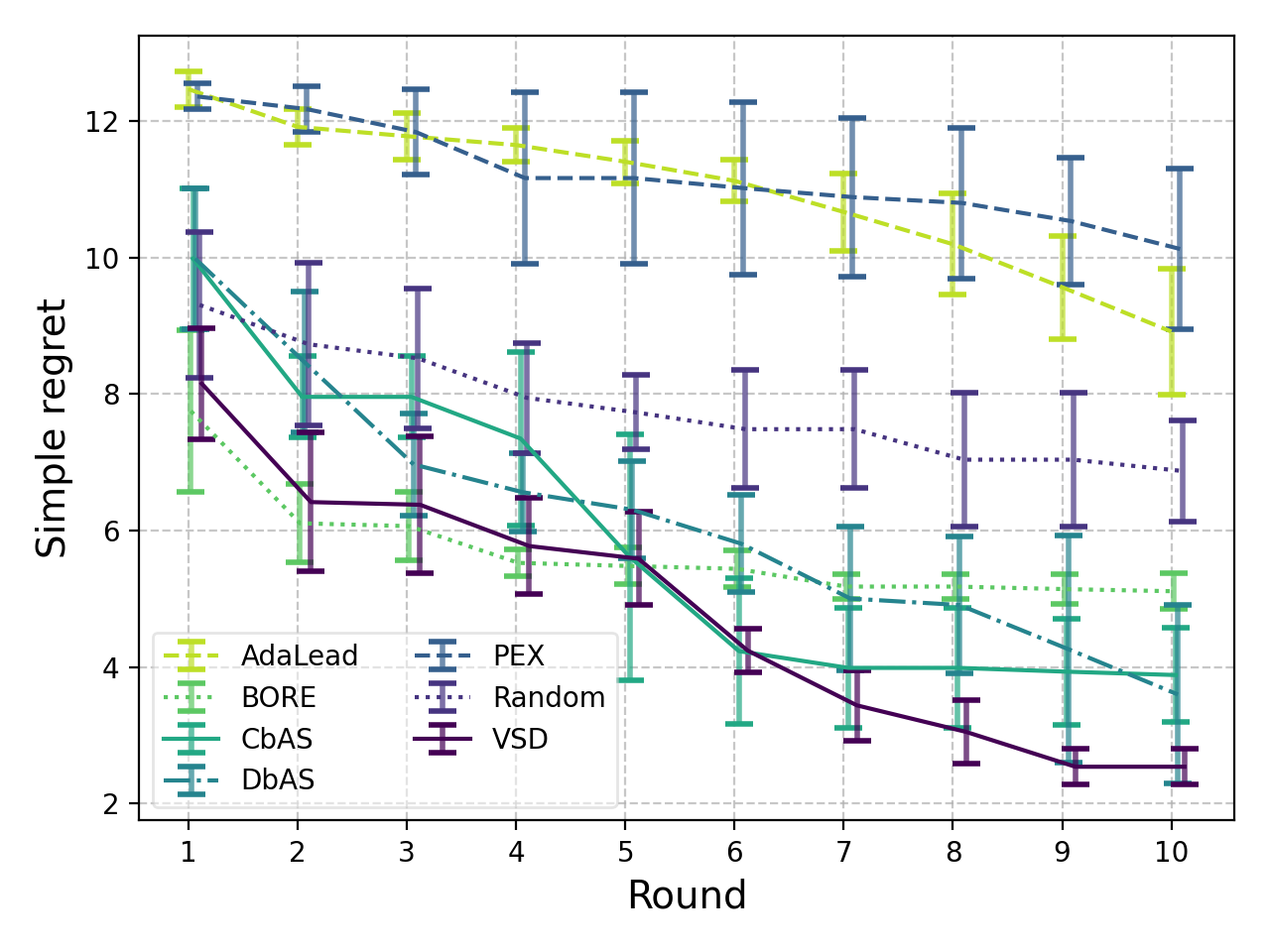}}
\subcaptionbox{LSTM\label{sfig:aav_sr_lstm}}
    {\includegraphics[width=.32\textwidth]{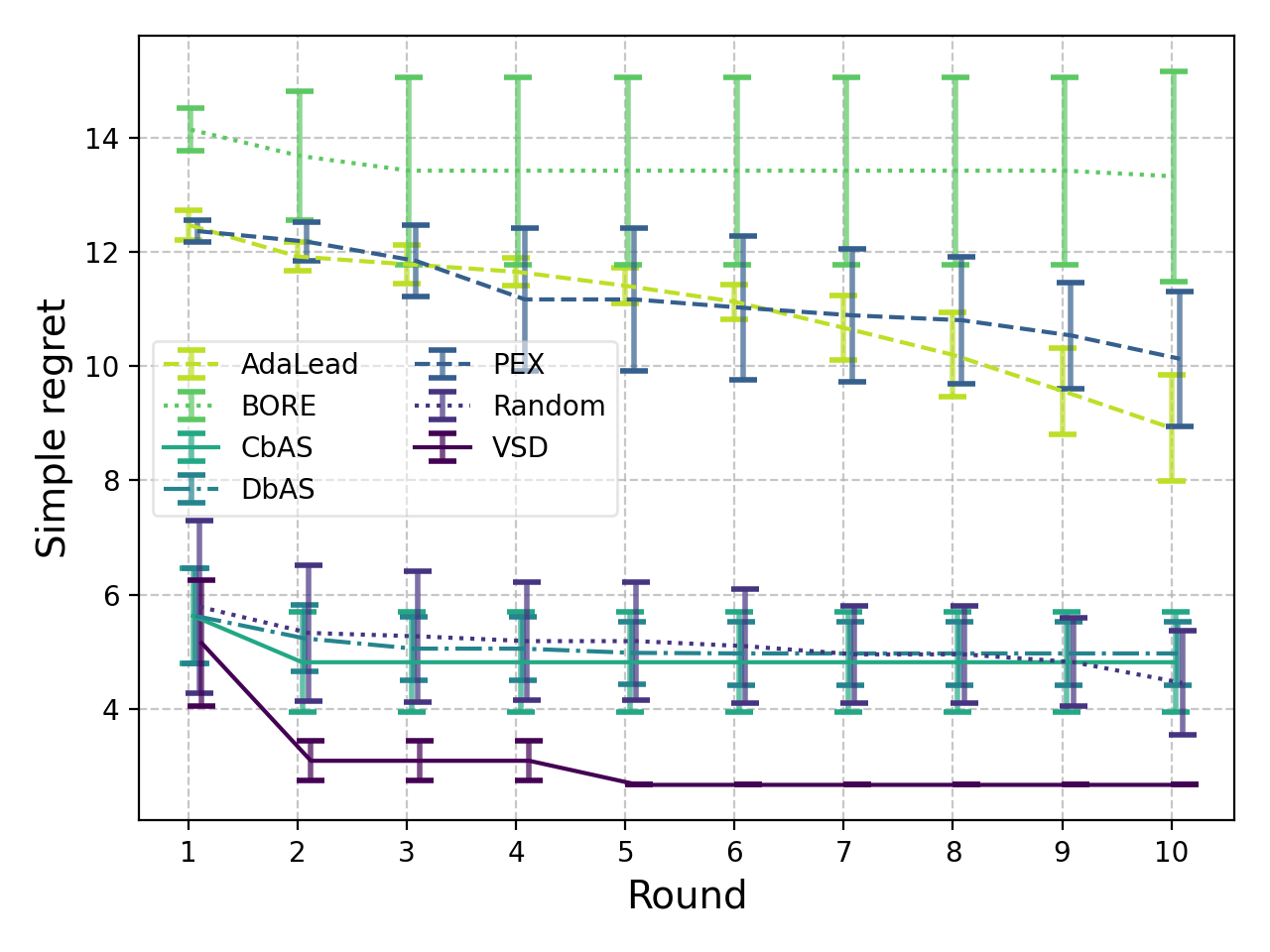}}
\subcaptionbox{Transformer\label{sfig:aav_sr_dtfm}}
    {\includegraphics[width=.32\textwidth]{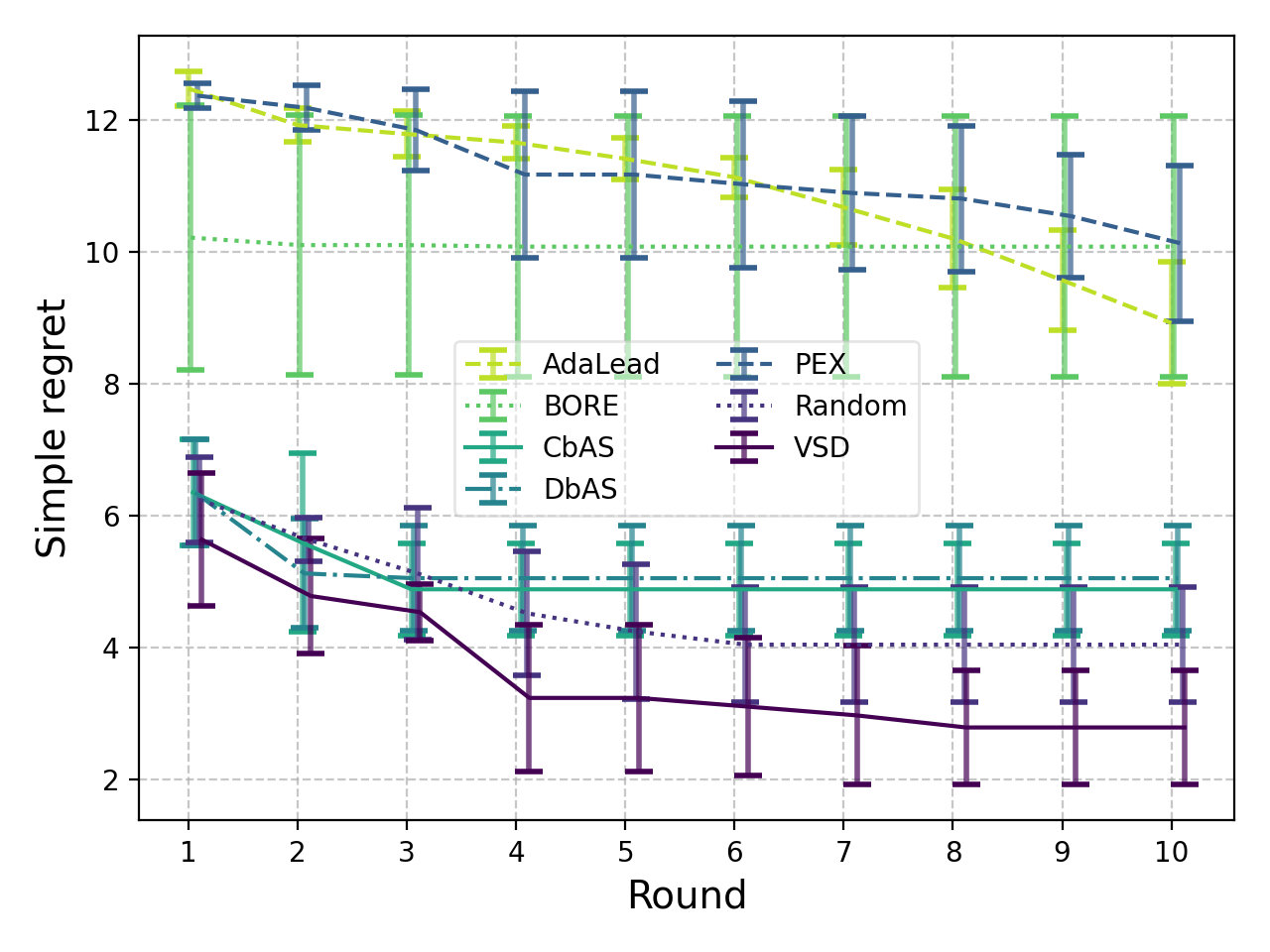}}
\caption{AAV \& GFP \gls{bbo} results. Simple regret (\autoref{eq:simregret}) -- lower is better -- on GFP and AAV with independent and auto-regressive variational distributions. The \gls{pex} and AdaLead results are replicated between the plots, since they are unaffected by choice of variational distribution.
}
\label{fig:bbo_res}
\vspace{-1em}
\end{figure}

\begin{figure}[htb]
    \centering
    \subcaptionbox{$M=15$\label{sfig:ehrlich-poli-15}}
    {\includegraphics[width=.329\textwidth]{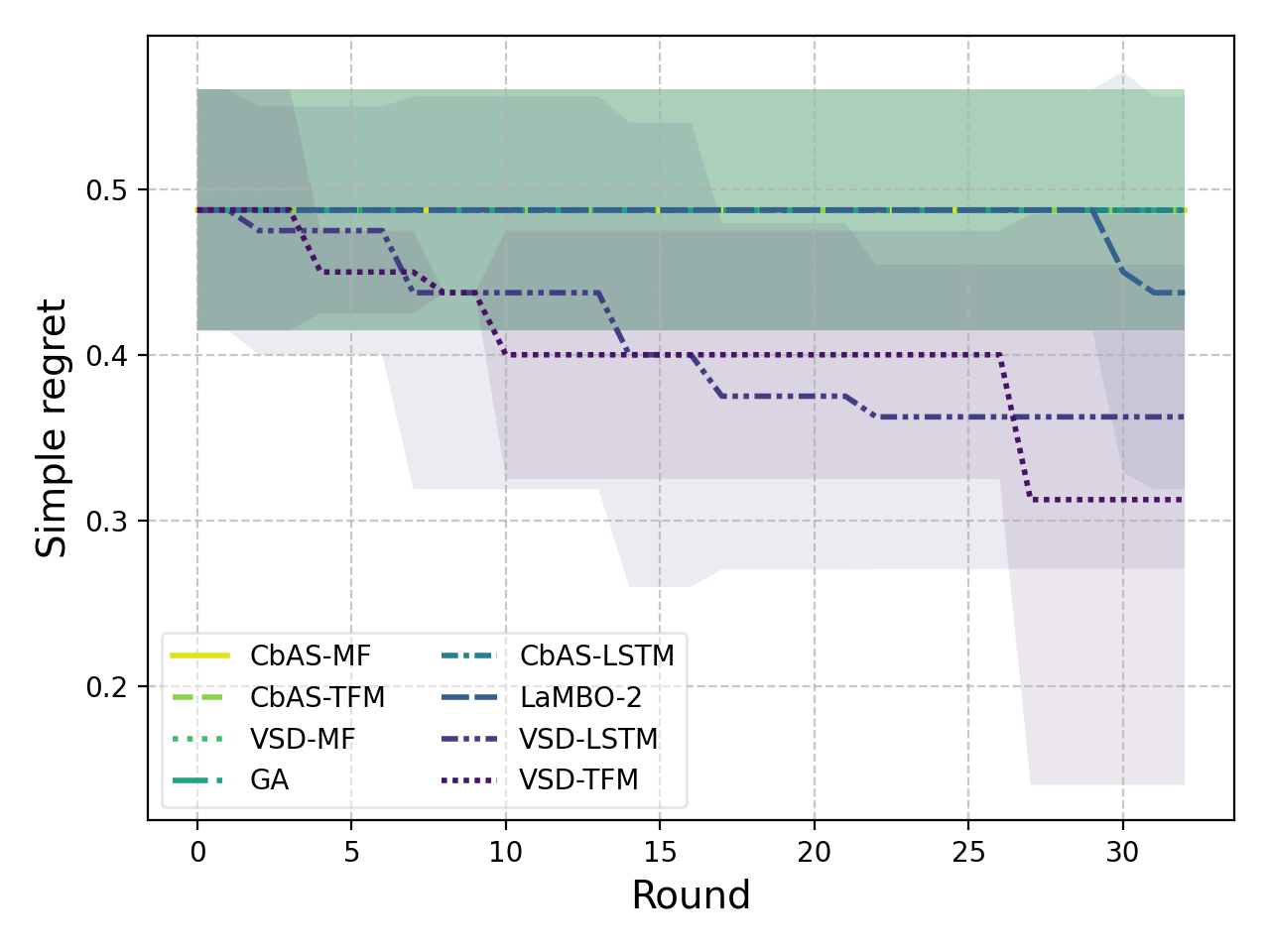}}
    \subcaptionbox{$M=32$\label{sfig:ehrlich-poli-32}}
    {\includegraphics[width=.329\textwidth]{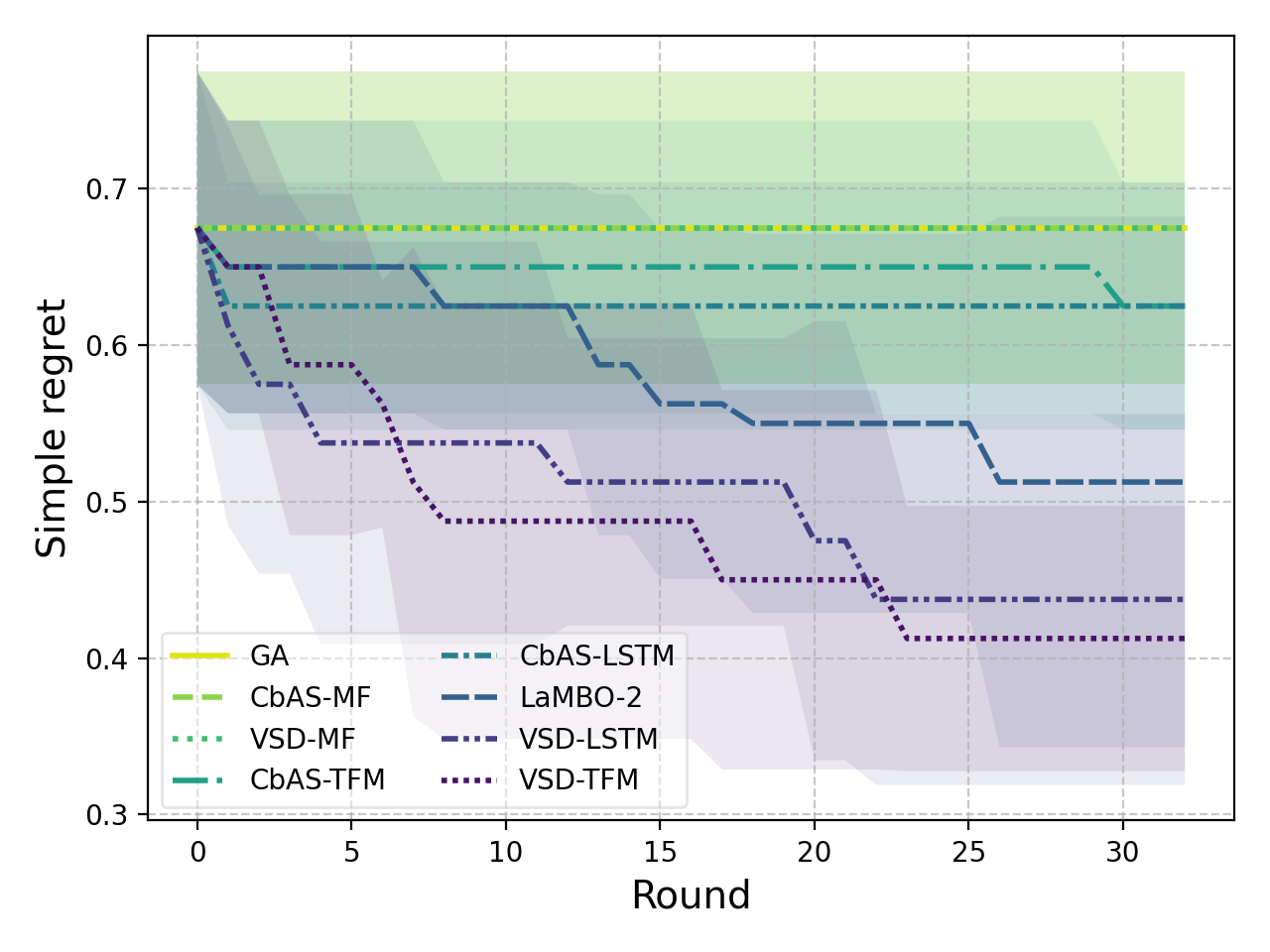}}
    \subcaptionbox{$M=64$\label{sfig:ehrlich-poli-64}}
    {\includegraphics[width=.329\textwidth]{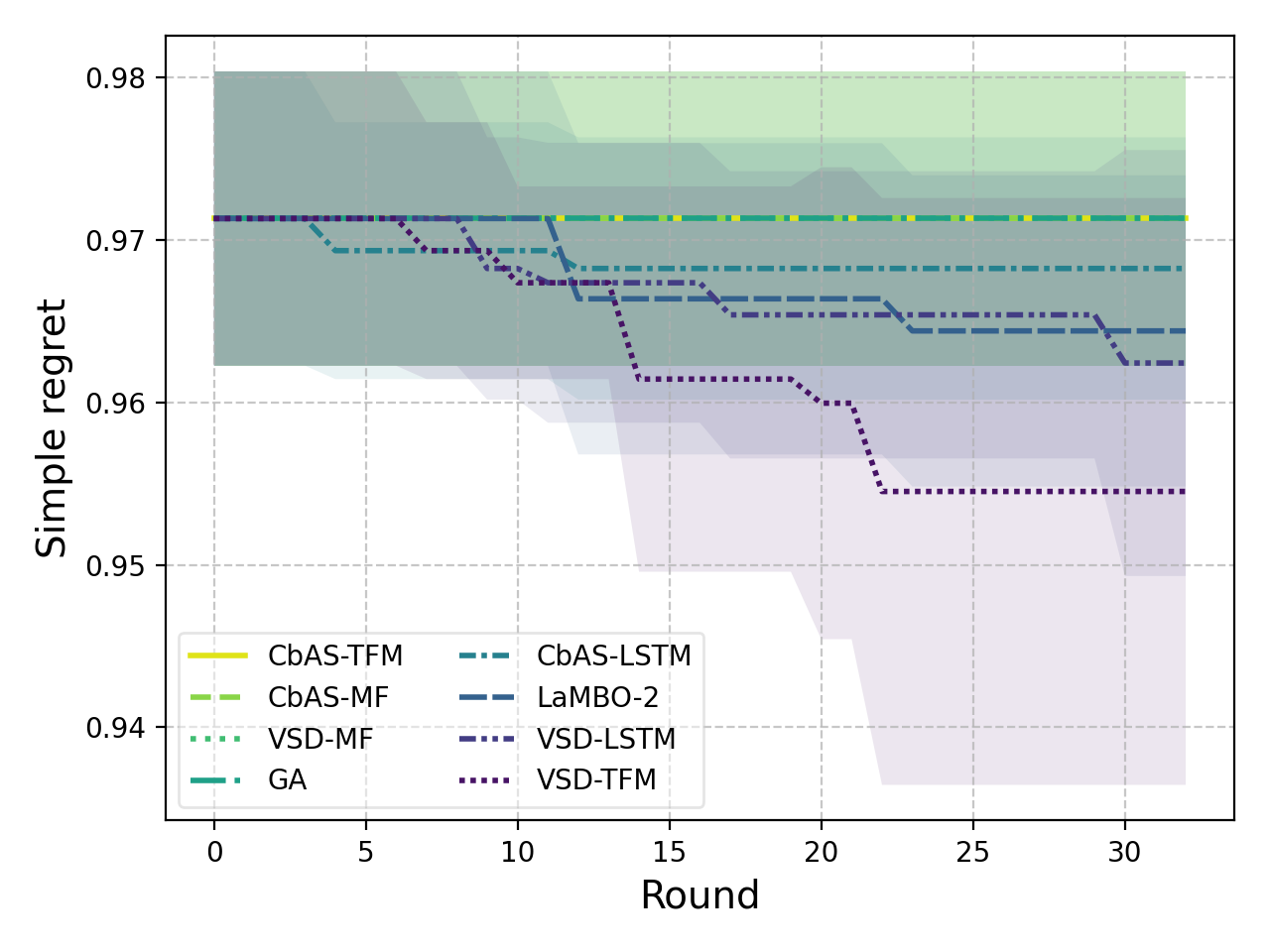}}
    \caption{Ehrlich function (\textsc{poli} implementation) \gls{bbo} results. \Gls{vsd} and \gls{cbas} with different variational distributions; mean field (MF), \gls{lstm} and transformer (TFM), compared against genetic algorithm (GA) and LaMBO-2 baselines.}
    \label{fig:ehrlich_poli_bbo}
    \vspace{-1em}
\end{figure}

The results are summarized in~\autoref{fig:bbo_res} and \ref{fig:ehrlich_poli_bbo}. Batch diversity scores for these experiments are presented in Appendix \ref{sub:diversity}, and for \textsc{holo} Ehrlich function implementation results see Appendix~\ref{app:ehrlich}. \gls{vsd} is among the leading methods for all experiments. \Gls{vsd} takes better advantage of the more complex variational distributions than \gls{cbas} and \gls{dbas} since it can sample from the adapted variational distribution while learning it. We can see that AdaLead, \gls{pex} and often \gls{bore} all perform worse than random for reasons previously mentioned. Simple regret can drop below zero for AAV \& GFP since an oracle is used as the black box function, but the global maximizer is taken from the experimental data. \gls{vsd} outperforms \gls{cbas} on the Ehrlich function benchmarks, and is competitive with LaMBO-2.
We also present an ablation study in Appendix \ref{sub:abl}. 

\section{Conclusion}

We have presented the problem of active generation --- sequentially learning a generative model for designs of a rare class by efficiently evaluating a black-box function --- and a method for efficiently generating samples which we call \acrfull{vsd}. Underpinned by variational inference, \Gls{vsd} satisfies critical requirements and important desiderata, and we show that \gls{vsd} converges asymptotically to the true level-set distribution at the same rate as a Monte-Carlo estimator with full knowledge of the true distribution. We showcased the benefits of our method empirically on a set of combinatorially complete and high dimensional sequential-design biological problems and show that it can effectively learn powerful generative models of fit designs.
There is a close connection between active generation and black-box optimization, and with the advent of powerful generative models we hope that our explicit framing of generation of fit sequences will lead to further study of this connection.
Finally, our framework can be generalized to more complex application scenarios, involving learning generative models over Pareto sets, $\solnspace_\textrm{Pareto}$, in a multi-objective setting, or other challenging combinatorial optimization problems \citep{bengio2021}, such as graph structures \citep{annadani2023bayesdag}, and mixed discrete-continuous variables. All of which are worth investigating as future work directions. For the code implementing the models and experiments in this paper, please see \url{https://github.com/csiro-funml/variationalsearch}.

\subsubsection*{Acknowledgments}
This work is funded by the CSIRO Science Digital and Advanced Engineering Biology Future Science Platforms, and was supported by resources and expertise provided by CSIRO IMT Scientific Computing. We would like to thank the anonymous reviewers, and especially reviewer \texttt{xamp} for the constructive feedback and advice, which greatly increased the relevance and quality of this work.

\bibliography{vsd}

\clearpage
\appendix

\renewcommand\thefigure{\thesection.\arabic{figure}}
\renewcommand\thetable{\thesection.\arabic{table}}
\setcounter{figure}{0}
\setcounter{table}{0}

\section{Acronyms}
\printglossary[type=\acronymtype]

\section{Depiction of Active Generation}

See \autoref{fig:diagrams} for graphical depictions of active generation as implemented by \gls{vsd}, compared to batch \gls{bo}. Active generation using \gls{vsd} follows \autoref{alg:optloop} to sequentially approximate $\probc{\obs}{\tar > \thresh}$. It uses samples from the current learned approximation of this distribution, $\qrobc{\obs}{\phi^*_t}$, for proposing candidates to evaluate each round. Unmodified batch \gls{bo} in the discrete setting without any specialization  requires a list of candidates, from which a batch of candidates are selected per round using a surrogate model with a batch acquisition function, e.g.~see \cite{wilson2017reparameterization}. The surrogate model's hyper-parameters, $\mparam$, are estimated by minimizing negative log marginal likelihood, $\mathcal{L}_\textrm{NLML}(\mparam, \data_N)$. Mechanistically, active generation learns a generative model of valuable candidates to circumvent the requirement of having to select candidates from a list. This is especially important for searching over the space of sequences, $\obsspace$, where enumerating feasible candidates is often intractable. Furthermore, active generation naturally lends itself to large batch sizes, while explicit batch optimization is often computationally intensive and is limited to $\bigo(10)$ candidates per batch \cite{wilson2017reparameterization}. As mentioned in \autoref{sec:related}, alternatives to active generation for specializing batch \gls{bo} to the discrete domain also include \gls{lso} \citep{tripp2020sample} and amortized \gls{bo} \citep{swersky2020amortized} among others.

\begin{figure}[htb]
    \centering
    \subcaptionbox{\gls{vsd} Active Generation\label{sfig:ag}}
    {\includegraphics[height=4.1cm]{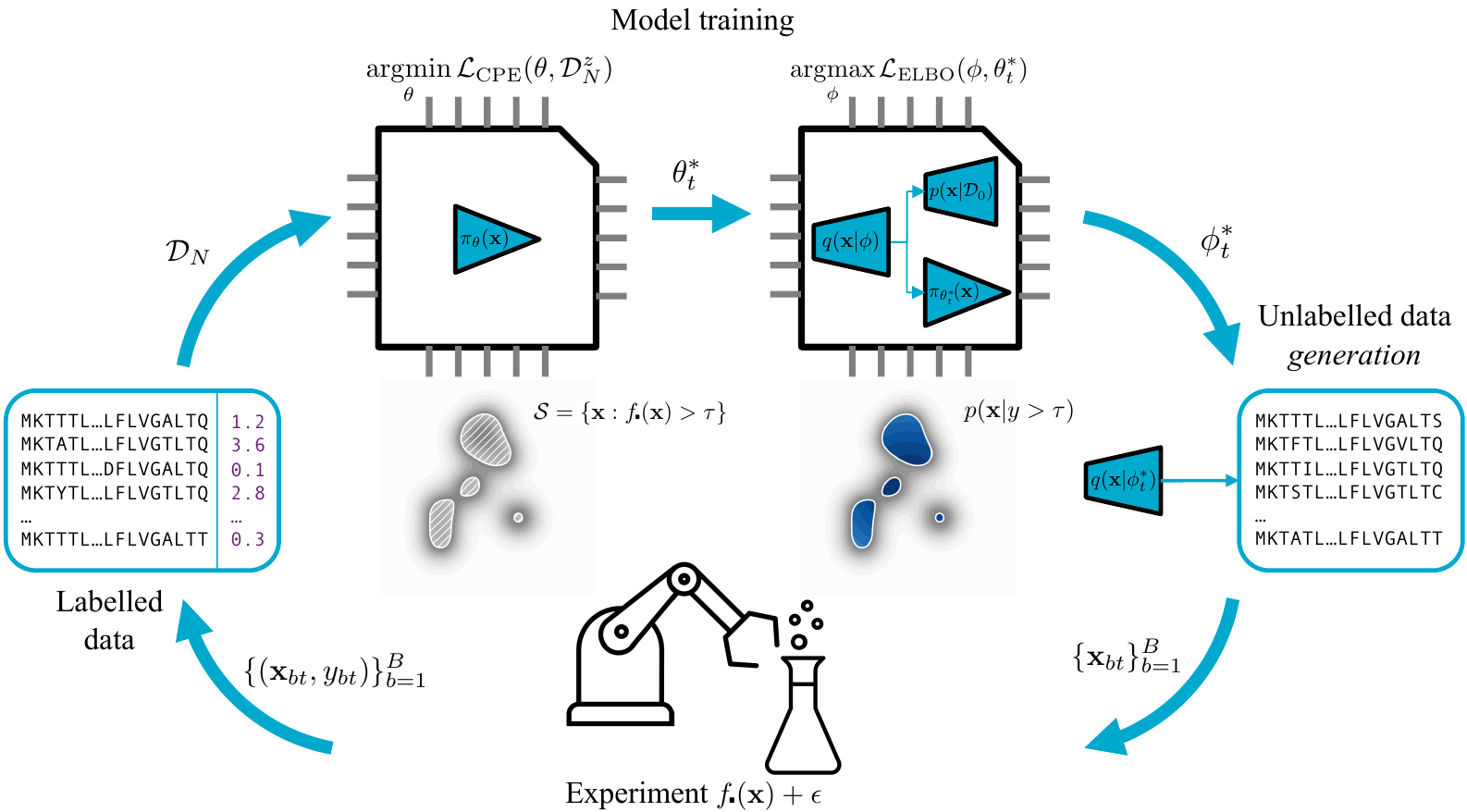}}
        \subcaptionbox{Batch \gls{bo}\label{sfig:batchbo}}
    {\includegraphics[height=4.1cm]{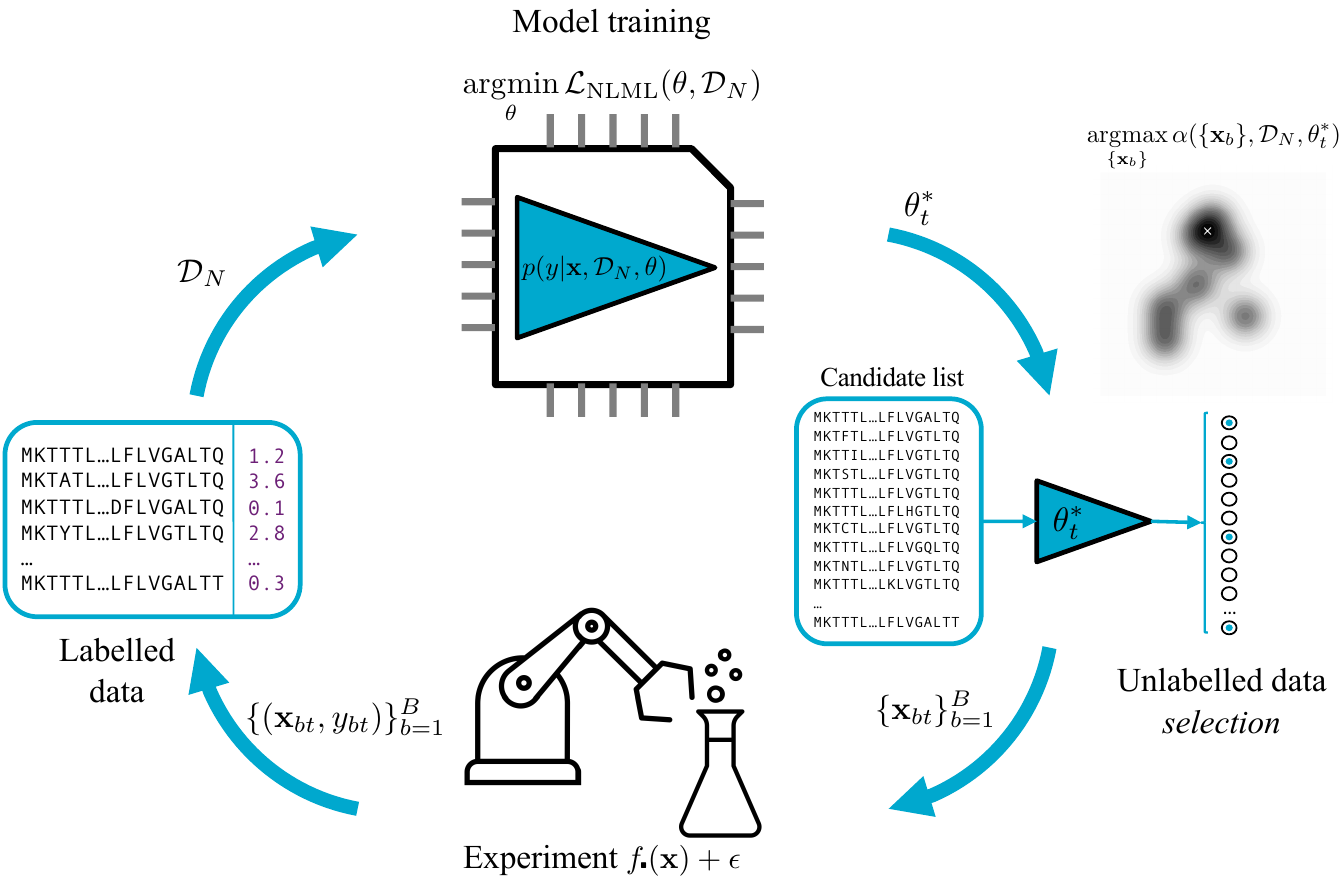}}
    \caption{Depictions of (\subref{sfig:ag}) active generation as implemented by \gls{vsd}, and (\subref{sfig:batchbo}) batch Bayesian optimization as applied to discrete sequences. Please see the text for a discussion of the differences between these approaches.}
    \label{fig:diagrams}
\end{figure}

\section{Experimental Details}
\label{app:expdet}

We use three well established datasets; a green fluorescent protein (GFP) from Aequorea Victoria~\citep{sarkisyan2016local}, an adeno-associated virus (AAV)~\cite{bryant2021deep}; and DNA binding activity to a human transcription factor (TFBIND8)~\citep{trabucco2022design, barrera2016survey}. These datasets have been used variously by~\citet{brookes2018design, brookes2019conditioning, angermueller2019model, kirjner2024improving, jain2022biological} among others. The GFP task is to maximize fluorescence, this protein consists of 238 amino acids, of which 237 can mutate. The AAV task us to maximize the genetic payload that can be delivered, and the associated protein has 28 amino acids, all of which can mutate. A complete combinatorial assessment is infeasible for these tasks, and so we use the convolution neural network oracle presented in~\cite{kirjner2024improving} as \emph{in-silico} ground truth. TFBIND8 contains a complete combinatorial assessment of the effect of changing 8 nucleotides on binding to human transcription factor SIX6 REF R1~\citep{barrera2016survey}. The dataset we use contains all 65536 sequences prepared by~\cite{trabucco2022design}.

We also use two novel datasets from recent works that experimentally assess the (near) complete combinatorial space of short sequences. The first dataset measures the antibiotic resistance of Escherichia coli metabolic gene folA, which encodes dihydrofolate reductase (DHFR)~\citep{papkou2023rugged}. Only a sub-sequence of this gene is varied (9 nucleic acids which encode 3 amino acids), and so a near-complete (99.7\%) combinatorial scan is available. For variants that have no fitness (resistance) data available, we give a score of $-1$. The next dataset is near-complete combinatorial scan of four interacting amino acid residues near the active site of the enzyme tryptophan synthase (TrpB)~\citep{johnston2024combinatorially}, with 159,129 unique sequences and fitness values, we use $-0.2$ for the missing fitness values (we do not use the authors' imputed values). These residues are explicitly shown to exhibit epistasis -- or non-additive effects on catalytic function -- which makes navigating this landscape a more interesting challenge from an optimization perspective.

Finally, we use the recently proposed Ehrlich functions \citep{stanton2024closed} benchmark. These functions are challenging closed form biological analogues, specifically designed to test \gls{bbo} methods on high dimensional sequence design tasks without having to resort to physical experimentation or machine learning oracles. We use the \textsc{poli} and \textsc{poli-baselines} software package for the benchmark and baselines \citep{gonzalez-duque2024poli}, and test on both the original \textsc{holo} implementation \citep{stanton2024closed} as well as the native \textsc{poli} implementation of these functions.

The properties of these datasets and benchmarks are presented in \autoref{tab:dataprops}.

\begin{table}[htb]
    \footnotesize
    \centering
    \begin{tabular}{r|c c c c}
         \textbf{Dataset} & $\card{\acidspace}$ & $M$ & $\card{\obsspace_\textrm{available}}$ & $\card{\obsspace}$ \\
         \hline
         TFBIND8 & 4 & 8 & 65,536 & 65,536 \\
         TrpB & 20 & 4 & 159,129 & 160,000 \\
         DHFR & 4 & 9 & 261,333 & 262,144 \\
         AAV & 20 & 28 & 42,340 & $20^{28}$ \\
         GFP & 20 & 237 & 51,715 & $20^{237}$ \\
         Ehrlich-15 & 20 & 15 & $20^{15}$ & $20^{15}$ \\
         Ehrlich-32 & 20 & 32 & $20^{32}$ & $20^{32}$ \\
         Ehrlich-64 & 20 & 64 & $20^{64}$ & $20^{64}$ \\
    \end{tabular}
    \caption{Alphabet size, sequence length, and number of available sequences for each of the datasets we use in this work.}
    \label{tab:dataprops}
\end{table}

We optimize \gls{vsd}, \gls{cbas}, \gls{dbas} and \gls{bore} for a minimum of 3000 iterations each round (5000 for all experiments but the Ehrlich functions) using Adam \citep{kingma2014adam}. When we use a \gls{cpe}, AdaLead's $\kappa$ parameter is set to 0.5 since the \gls{cpe} already incorporates the appropriate threshold.

\subsection{Fitness Landscapes Settings}
\label{app:fl_settings}

For the DHFR and TrpB experiments we set maximum fitness in the training dataset to be that of the wild type, and $\thresh$ to be slightly below the wild type fitness value (so we have $\sim\!10$ positive examples to train the \gls{cpe} with). We use a randomly selected $N_\textrm{train} = 2000$ below the wild-type fitness to initially train the \gls{cpe}, we also explicitly include the wild-type. The thresholds and wild-type fitness values are; DHRF: $\thresh = -0.1$, $\tar_\textrm{wt} = 0$, TrpB: $\thresh = 0.35$, $\tar_\textrm{wt} = 0.409$. We follow the same procedure for the TFBIND8 experiment, however, there is no notion of a wild-type sequence in this data, and so we set $\thresh = 0.75$, and $\tar_\textrm{train max} = 0.85$.
We use a uniform prior over sequences, $\prob{\obs} = \prod^M_{m=1} \categc{\acid_m}{\mathbf{1} \cdot |\acidspace|^{-1}}$, since these are relatively small search spaces, and the sub-sequences of nucleic/amino acids have been specifically selected for their task. Similarly, we find that relatively simple independent (mean-field) variational distributions of the form in \autoref{eq:proposal} and MLP based \gls{cpe}s work best for these experiments (details in \autoref{sub:cpe}).

\subsection{Black-box Optimization Settings}
\label{app:bbo_settings}

We follow \cite{kirjner2024improving} in the experimental settings for the AAV and GFP datasets, but we modify the maximum fitness training point and training dataset sizes to make them more amenable to a sequential optimization setting. The initial percentiles, schedule, and max training fitness values are; AAV: $p_0 = 0.8$, $\eta=0.7$, $\tar_\textrm{max} = 5$, GFP: $p_0 = 0.8$, $\eta = 0.7$ $\tar_\textrm{max} = 1.9$. We aim for $p_T = 0.99$. The edit distance between $\obs^*$ and the fittest sequence in the \gls{cpe} training data is 8 for GFP, and 13 for AAV. We again use a random $N_\textrm{train} = 2000$ for training the \gls{cpe}s, which in this case are CNNs -- architecture specifics are in \autoref{sub:cpe}.

For the Ehrlich function experiment, we use sequence lengths of $M=\{15, 32, 64\}$ with 2 motifs for the shorter sequence lengths, and 8 motifs for $M=64$. All use a motif length of 4 and a quantization of 4. $B=128$, $T=32$ and \emph{only} 128 random samples of the function are used for $\data_N$ -- these are resampled for each seed. As before, 5 different random seeds are used for these trials, and for VSD we use an the same scheduling function for $\thresh_t$ as in \autoref{eq:annealedthresh}, with $p_0 = 0.5$ and $\eta=0.87$ (so $p_T = 0.99$). The lower initial percentile is used since the training dataset is much smaller than in the other experiments, and we find allowing for more exploration initially improves \gls{vsd}'s performance.

In these higher dimensional settings, we find that performance of the methods heavily relies on using an informed prior (in the case of \gls{vsd} and \gls{cbas}), or initial variational distribution (in the case of \gls{dbas} and \gls{bore}). To this end, we follow \citet{brookes2019conditioning} and fit the initial variational distribution to the \gls{cpe} training sequences (regardless of fitness), but we use maximum likelihood. For the more complex variational distributions (\gls{lstm} and transformer), we have to be careful not to over-fit -- so we implement early stopping and data augmentation techniques. Then for \gls{vsd} and \gls{cbas} we copy this distribution and fix its parameters for the remainder of the experiment for use as a prior. We also use this prior for the Random method, but AdaLead and \gls{pex} use alternative generative heuristics. For these experiments we use the simple independent variational distribution and the same \gls{lstm} and causal decoder-only transformer models from \autoref{sub:cde}.

\subsection{Variational Distributions}
\label{app:var_dist}

In this section we summarize the main variational distribution architectures considered for \gls{vsd}, \gls{bore}, \gls{cbas} and \gls{dbas}, and the sampling distributions for the Random baseline method. Somewhat surprisingly, we find that we often obtain good results for the biological sequence experiments using a simple independent (or mean-field) variational distribution, especially in lower dimensional settings,
\begin{align}
    \qrobc{\obs}{\qparam} = \prod^M_{m=1} \categc{\acid_m }{\softmax{\qparam_m}},
    \label{eq:proposal}
\end{align}
where $\acid_m \in \acidspace$ and $\qparam_m \in \real^{\card{\acidspace}}$. However, this simple mean-field distribution was not capable of generating convincing handwritten digits or, in some cases, higher-dimensional sequences. We  have also tested a variety of transition variational distributions,
\begin{align}
    \qrobc{\obs_t}{\obs_{t-1}, \qparam} = \prod^M_{m=1} \categc{\acid_{tm}}{\softmax{\textrm{NN}_m(\obs_{t-1}, \qparam)}},
    \label{eq:transition-proposal}
\end{align}
where $\textrm{NN}_m(\obs_{t-1}, \qparam)$ is the $m^\textrm{th}$ vector output of a neural network that takes a sequence from the previous round, $\obs_{t-1}$, as input. We have implemented multiple neural net encoder/decoder architectures for $\textrm{NN}_m(\obs_{t-1}, \qparam)$, but we did not consider architectures of the form $\textrm{NN}_m(\qparam)$ since the variational distribution in \autoref{eq:proposal} can always learn a $\phi_m = \textrm{NN}_m(\qparam')$. We found that none of these transition architectures significantly outperformed the mean-field distribution (\autoref{eq:proposal}) when it was initialized well (e.g.~fit to the \gls{cpe} training sequences), see \autoref{sub:abl} for results. We also implemented auto-regressive variational distributions of the form,
\begin{align}
    \qrobc{\obs}{\qparam} &= \categc{\acid_1}{\softmax{\qparam_1}} \prod^M_{m=2} \qrobc{\acid_m}{\acid_{1:m-1}, \qparam_{1:m}} \quad \textrm{where,}\label{eq:autoregressive-proposal} \\
    \qrobc{\acid_m}{\acid_{1:m-1}, \qparam_{1:m}} &=
    \begin{cases}
        \categc{\acid_m}{\softmax{\mathrm{LSTM}(\acid_{m-1}, \qparam_{m-1:m})}}, \nonumber \\
        \categc{\acid_m}{\softmax{\mathrm{DTransformer}(\acid_{1:m-1}, \qparam_{1:m})}}. \nonumber
    \end{cases}
\end{align}

For a \gls{lstm} \gls{rnn} and a decoder-only transformer with a causal mask, for the latter see \citet[Algorithm 10 \& Algorithm 14]{phuong2022formal} for maximum likelihood training and sampling implementation details respectively. We list the configurations of the \gls{lstm} and transformer variational distributions in \autoref{tab:qsettings}. We use additive positional encoding for all of these models. When using these models for priors or initialization of variational distributions, we find that over-fitting can be an issue. To circumvent this, we use early stopping for  larger training datasets, or data augmentation techniques for smaller training datasets (as in the case of the Ehrlich functions).

\begin{table}[htb]
    \centering
    \begin{tabular}{r r|c c c c c c}
        & Configuration & Digits & AAV & GFP & Ehrlich 15 & Ehrlich 32 & Ehrlich 64 \\
        \hline
        LSTM & Layers & 5 & 4 & 4 & 3 & 3 & 3 \\
             & Network size & 128 & 32 & 32 & 32 & 32 & 64 \\
             & Embedding size & 4 & 10 & 10 & 10 & 10 & 10 \\
        \hline
        Transformer & Layers & 4 & 1 & 1 & 2 & 2 & 2 \\
                    & Network Size & 256 & 64 & 64 & 32 & 64 & 128 \\
                    & Attention heads & 8 & 2 & 2 & 1 & 2 & 3 \\
                    & Embedding size & 32 & 20 & 20 & 10 & 20 & 30
    \end{tabular}
    \caption{\Gls{lstm} and transformer network configuration.}
    \label{tab:qsettings}
\end{table}

\subsection{Class Probability Estimator Architectures}
\label{sub:cpe}

For the fitness landscape experiments on the smaller combinatorially complete datasets we use a two-hidden layer MLP, with an input embedding layer. The architecture is given in \autoref{fig:cpe-arch} (a). For the larger dimensional AAV and GFP datasets and Ehrlich function benchmark, we use the convolutional architecture given in \autoref{fig:cpe-arch} (b). On all but the Ehrlich benchmark, five fold cross validation was used to select the hyper parameters before the \gls{cpe}s are trained on the whole training set for use in the subsequent experimental rounds. For the Ehrlich benchmark we do not use cross-validation to select the \gls{cpe} hyper parameters -- but we do use an additive ensemble of 10 randomly initialized CNNs for the \gls{cpe} following LaMBO-2. Model updates are performed by retraining on the whole query set.

\begin{figure}[htb]
    \centering
\begin{minipage}{.45\textwidth}%
\begin{lstlisting}
Sequential(
    Embedding(
        num_embeddings=A,
        embedding_dim=8
    ),
    Dropout(p=0.2),
    Flatten(),
    LeakyReLU(),
    Linear(
        in_features=8 * M,
        out_features=32
    ),
    LeakyReLU(),
    Linear(
        in_features=32,
        out_features=1
    ),
)
\end{lstlisting}
\centering
\vspace{6.35cm}
(a) MLP architecture
\end{minipage}%
\begin{minipage}{.45\textwidth}%
\begin{lstlisting}
Sequential(
    Embedding(
        num_embeddings=A,
        embedding_dim=10
    ),
    Dropout(p=0.2),
    Conv1d(
        in_channels=10,
        out_channels=16,
        kernel_size=3 or 7,
    ),
    LeakyReLU(),
    MaxPool1d(
        kernel_size=2 or 4,
        stride=2 or 4,
    ),
    Conv1d(
        in_channels=16,
        out_channels=16,
        kernel_size=7,
    ),
    LeakyReLU(),
    MaxPool1d(
        kernel_size=2 or 4,
        stride=2 or 4,
    ),
    Flatten(),
    LazyLinear(
        out_features=128
    ),
    LeakyReLU(),
    Linear(
        in_features=128,
        out_features=1
    ),
)
\end{lstlisting}
\centering
(b) CNN architecture
\end{minipage}%
    \caption{\Gls{cpe} architectures used for the experiments in PyTorch syntax. $\texttt{A} = \card{\acidspace}$, $\texttt{M} = M$, GFP uses a max pooling kernel size and stride of 4, all other datasets and benchmarks use 2. The Ehrlich function benchmark uses and ensemble of 10 randomly initialized CNNs that are additively combined. The Ehrlich-15 functions use a kernel size of 3,
    all other \gls{bbo} experiments use a kernel size of 7. LaMBO-2 uses the same kernel size as our CNNs for the Ehrlich functions.}
    \label{fig:cpe-arch}
\end{figure}

\section{Additional Experimental Results}
\label{sec:aresults}

\subsection{Fitness Landscapes -- Gaussian Process Probability of Improvement}
\label{sub:gp_fl}

Here we present additional fitness landscape experimental results, where we have used a \gls{gp} as a surrogate model for $\probc{\tar}{\obs, \data_N}$ in conjunction with a complementary Normal CDF as the \gls{pi} acquisition function. This is one of the main frameworks supported by our theoretical analysis.
\Gls{vsd}, \gls{dbas}, \gls{cbas} and \gls{bore} can make use of the \gls{gp}-\gls{pi} acquisition function, and so \gls{bore} is \gls{bopr}~\citep{daulton2022bayesian} in this instance since we are not using a \gls{cpe}. \Gls{pex} and AdaLead only use the \gls{gp} surrogate, as per their original formulation.
The \gls{gp} uses a simple categorical kernel with automatic relevance determination from \cite{balandat2020botorch},
\begin{align}
    \gpkernel(\obs, \obs') = \sigma \exp\!\left(-\frac{1}{M}\sum^M_{m=1}\frac{\indic{\acid_m = \acid'_m}}{l_m}\right),
\end{align}
where $\sigma$ and $l_m$ are hyper-parameters controlling scale and length-scale respectively. See \autoref{fig:fl_res_gp} for the results.

\begin{figure}[tb]
\centering
\includegraphics[width=.32\textwidth]{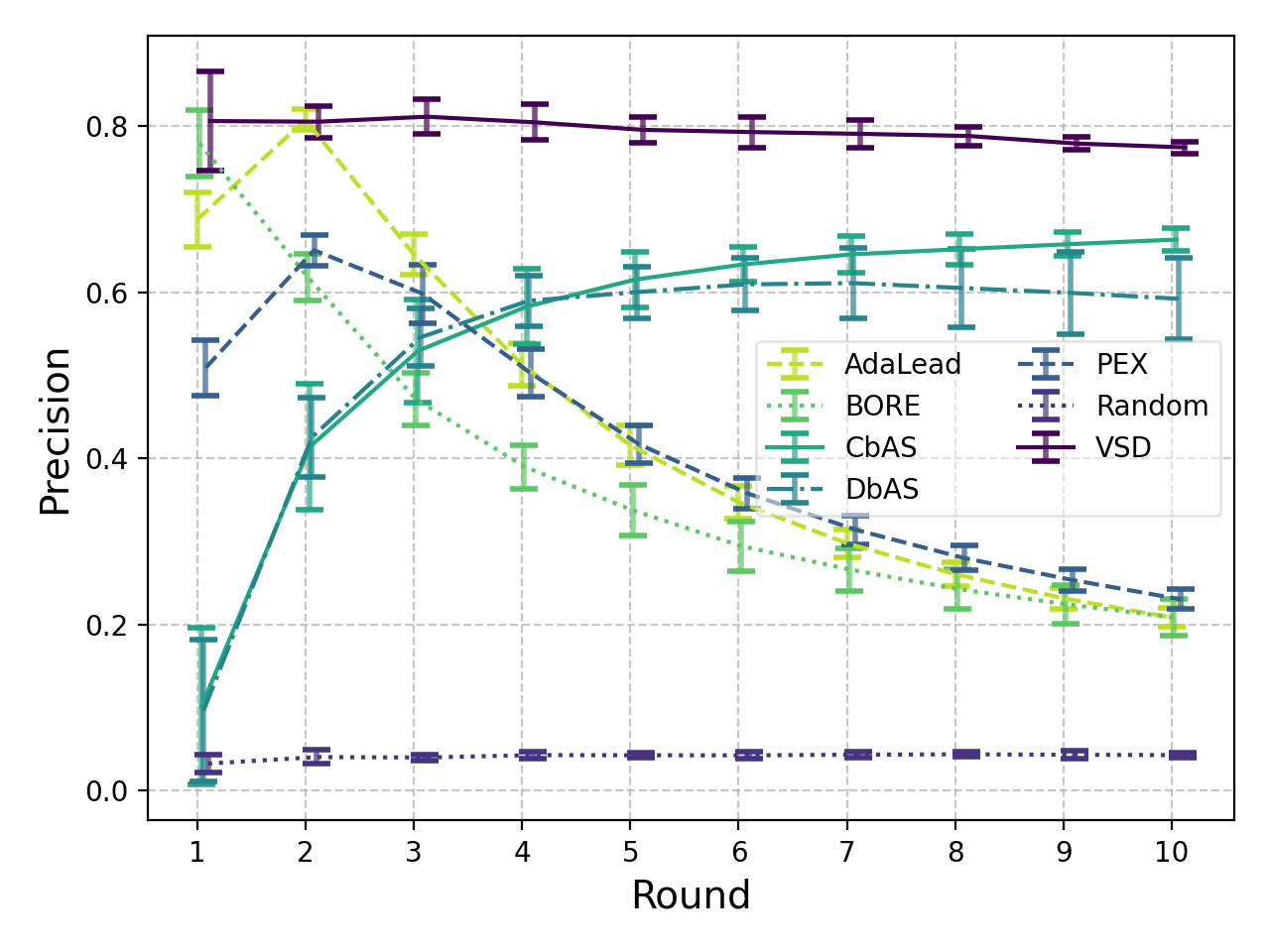}
\includegraphics[width=.32\textwidth]{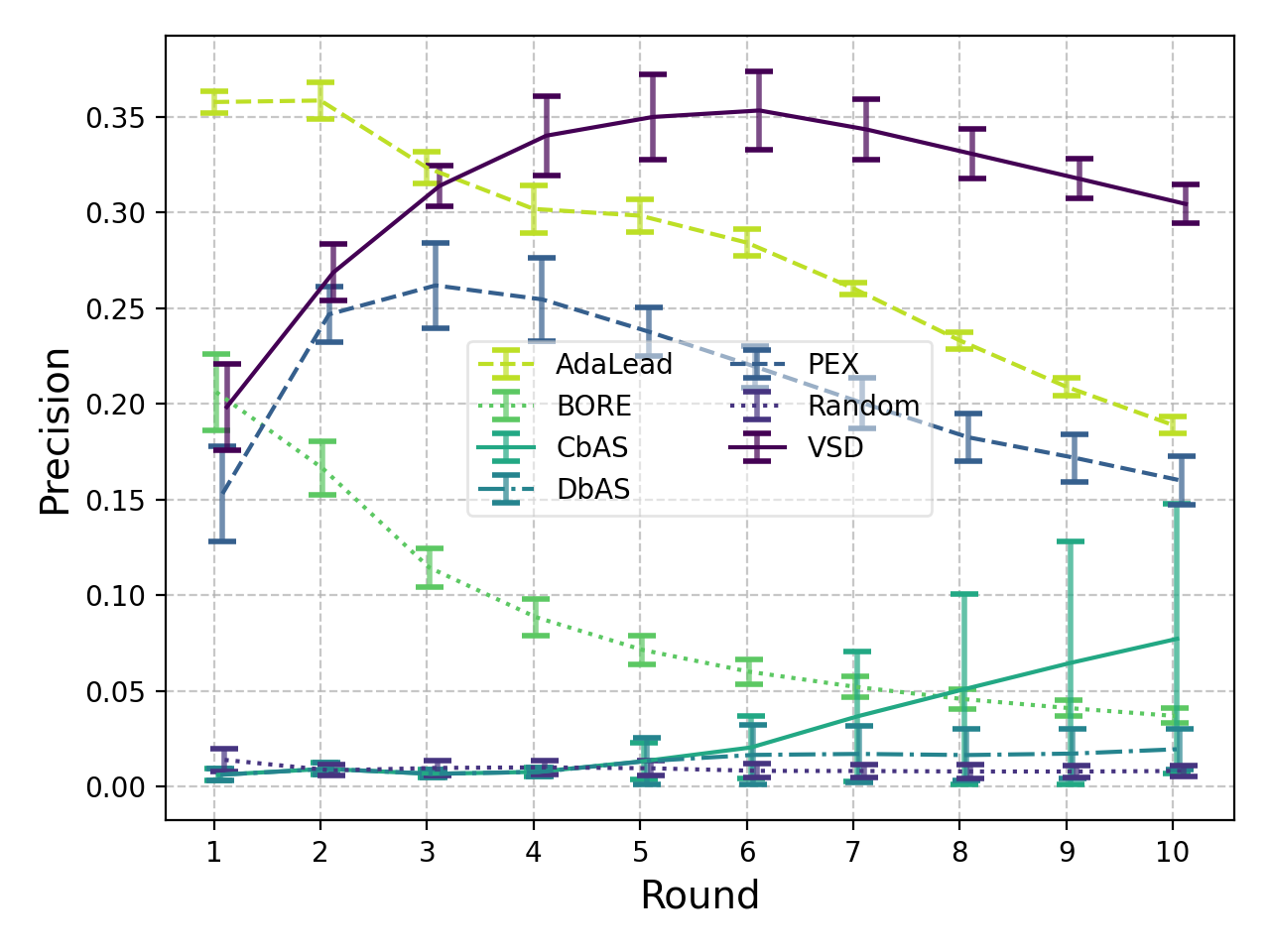}
\includegraphics[width=.32\textwidth]{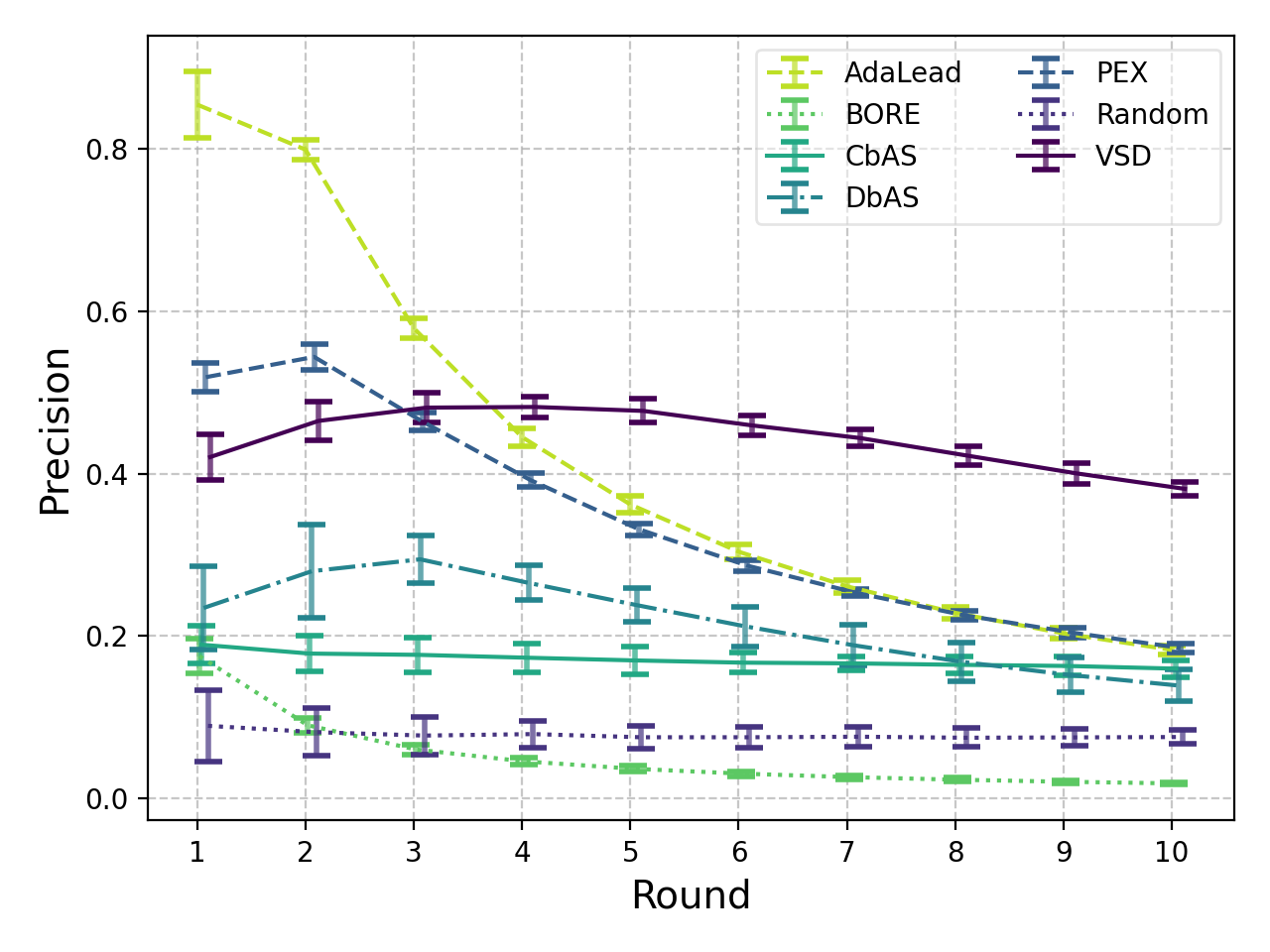} \\
\includegraphics[width=.32\textwidth]{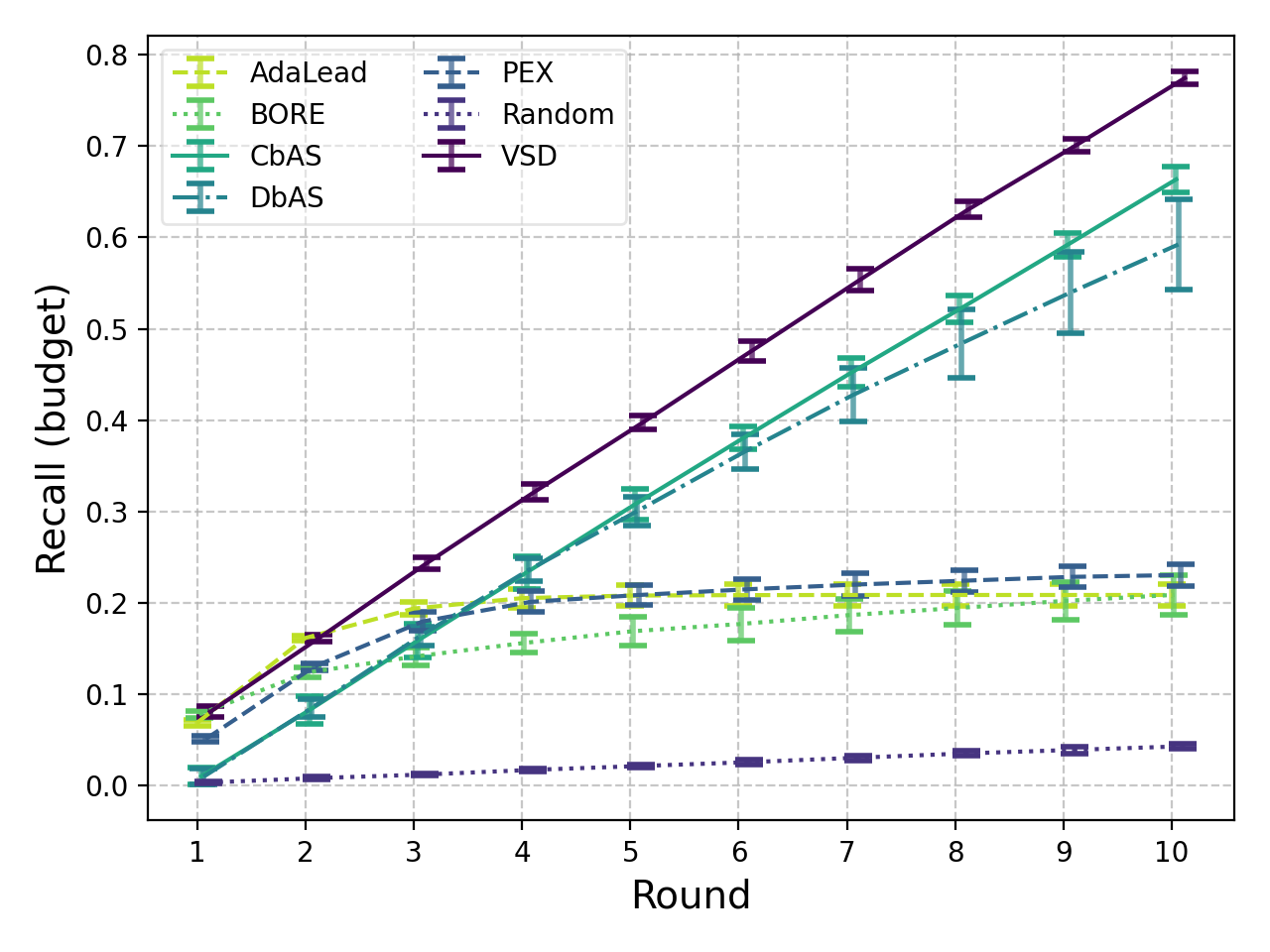}
\includegraphics[width=.32\textwidth]{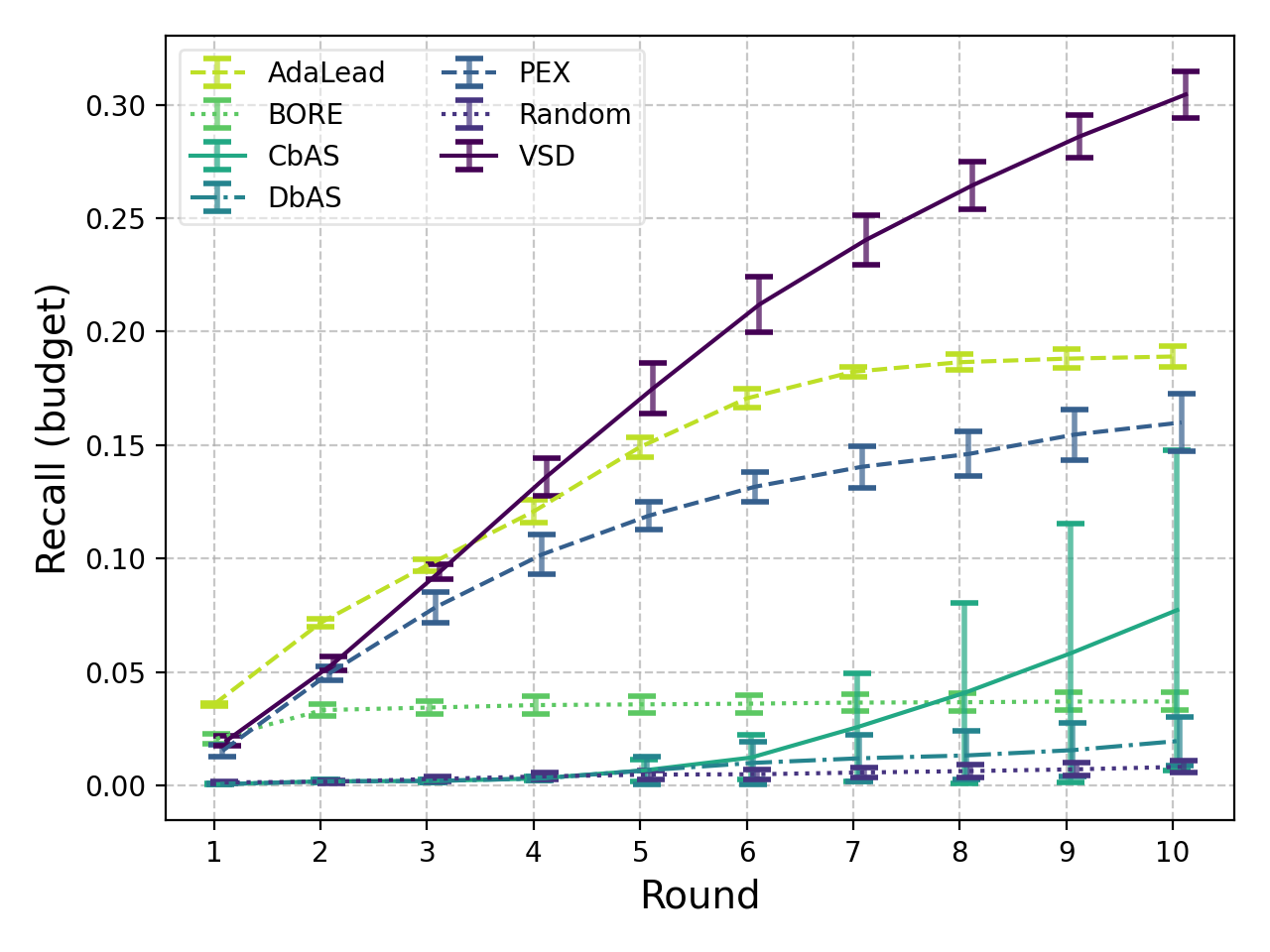}
\includegraphics[width=.32\textwidth]{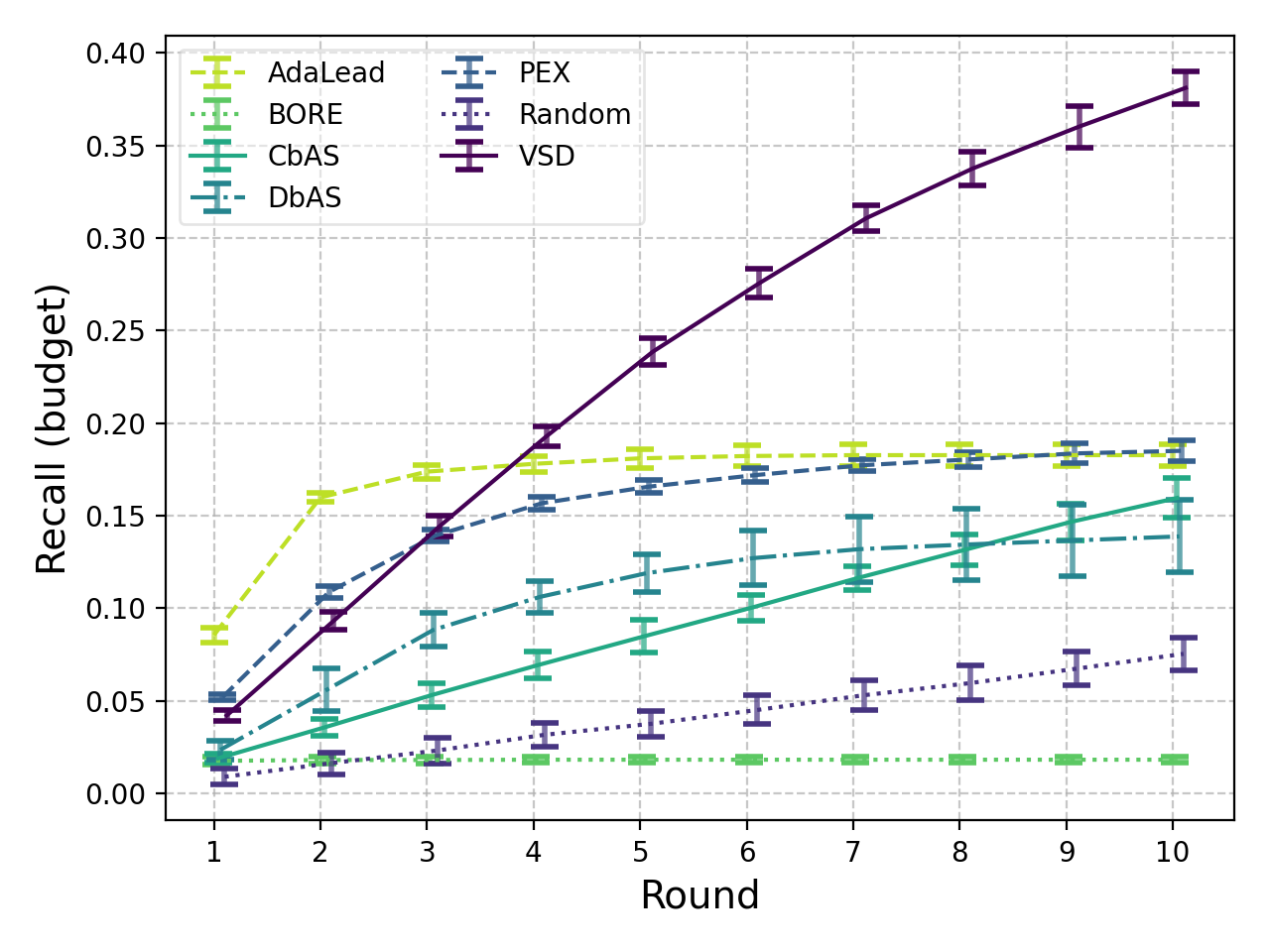} \\
\subcaptionbox{DHFR\label{sfig:dhfr_fl_gp}}
    {\includegraphics[width=.32\textwidth]{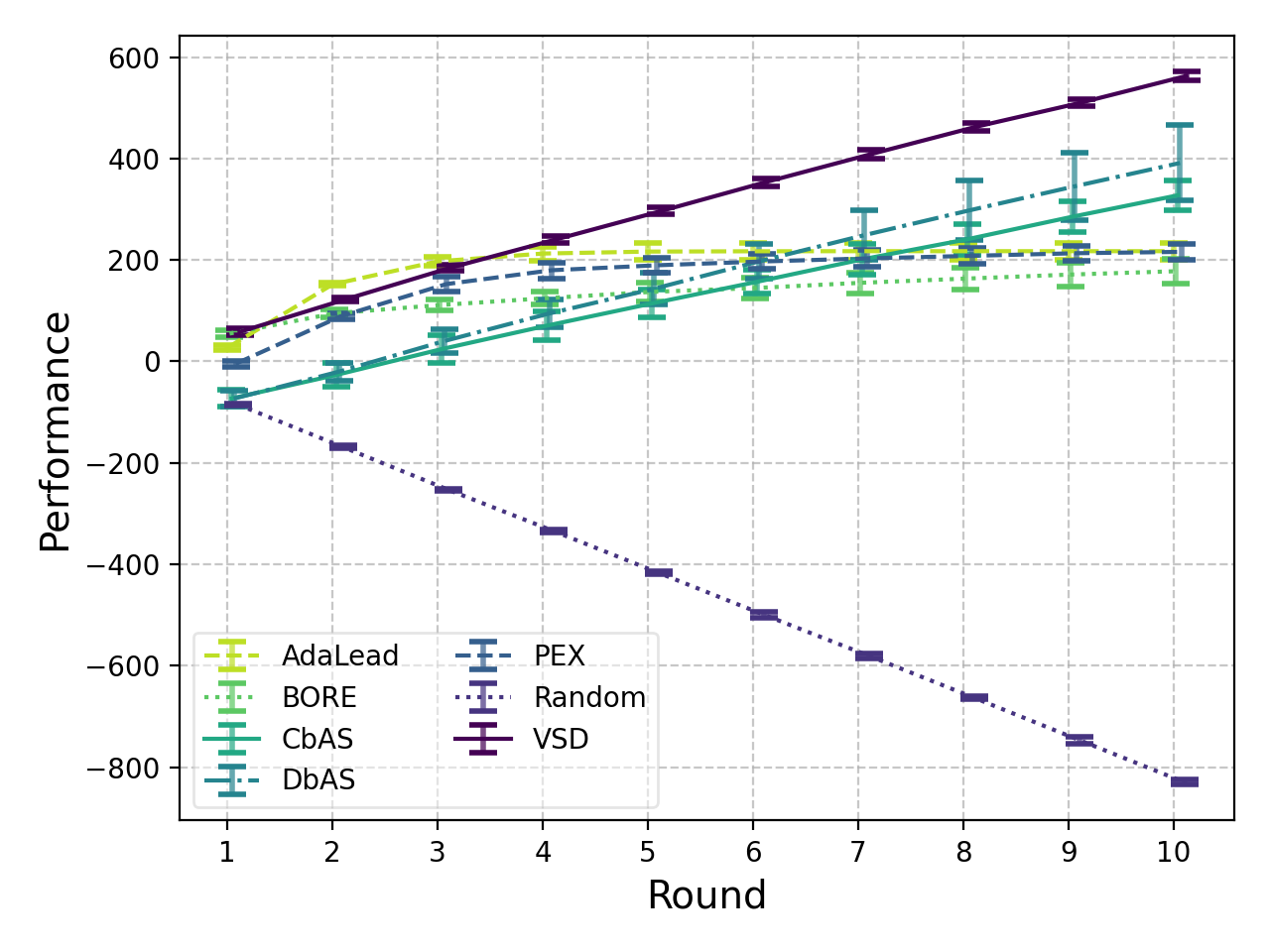}}
\subcaptionbox{TrpB\label{sfig:trpb_fl_gp}}
    {\includegraphics[width=.32\textwidth]{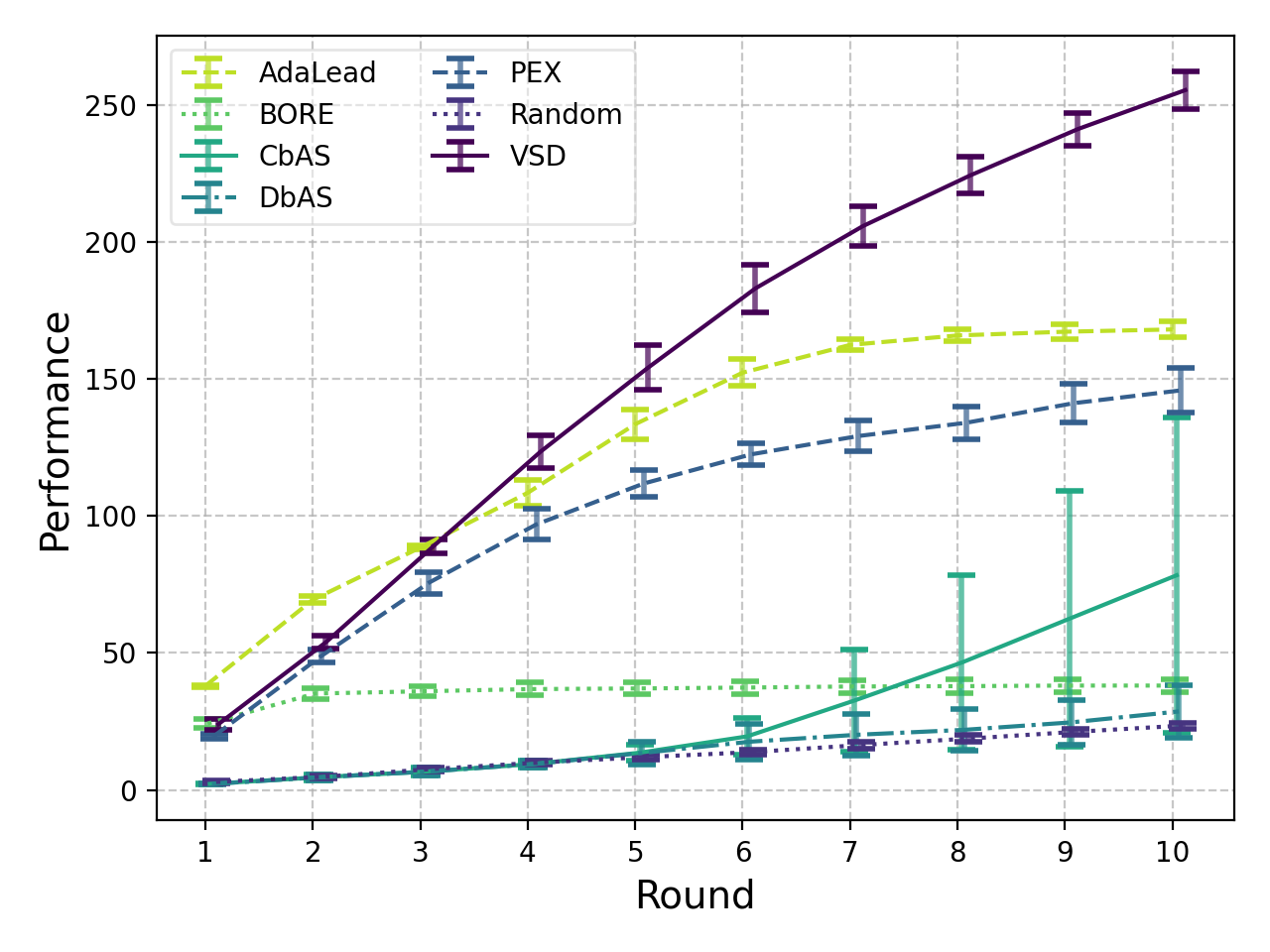}}
\subcaptionbox{TFBIND8\label{sfig:tfbind8_fl_gp}}
    {\includegraphics[width=.32\textwidth]{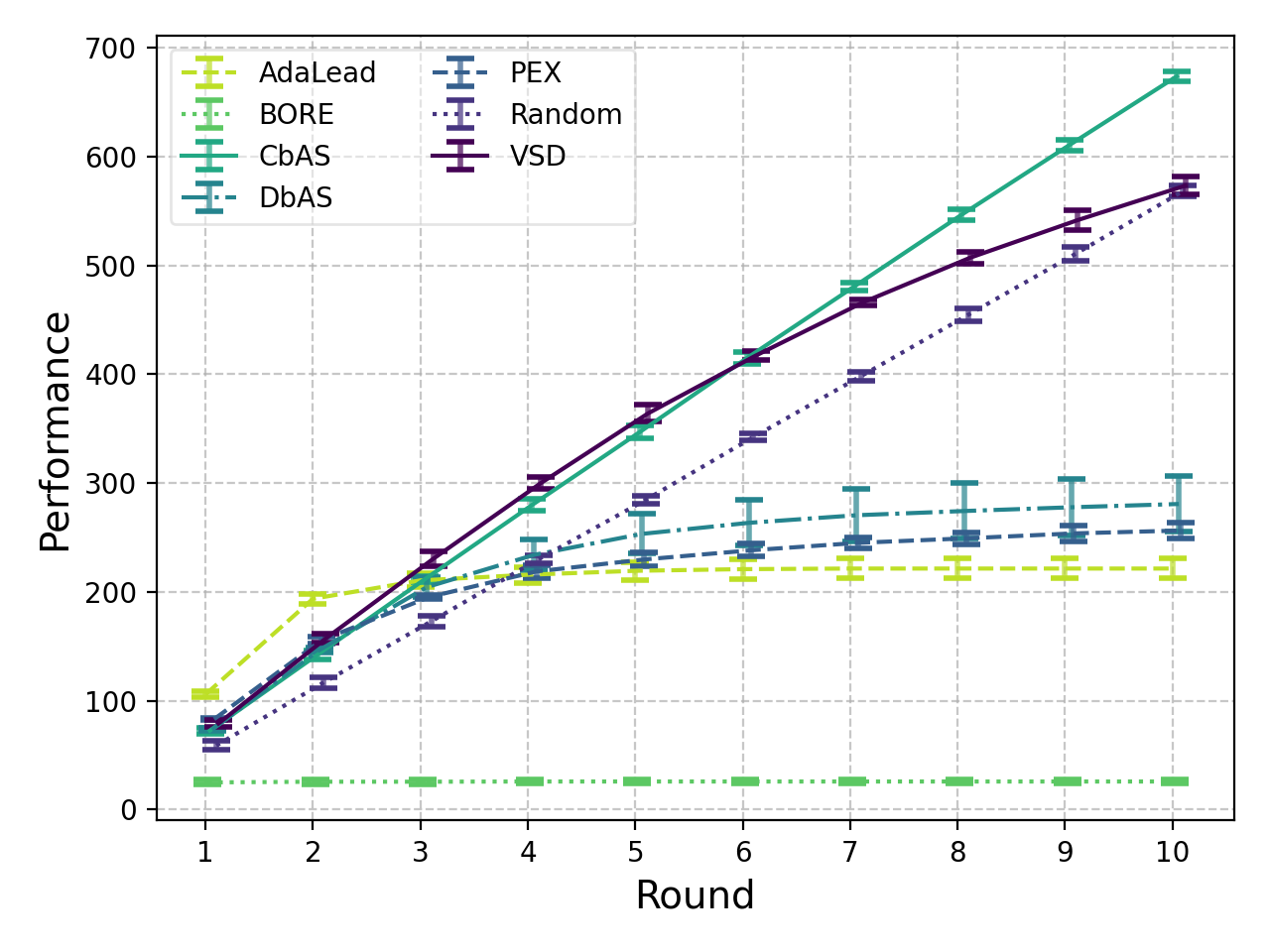}}
\caption{Fitness landscape results using \gls{gp}-\gls{pi}. Precision (\autoref{eq:prec}), recall (\autoref{eq:rec}) and performance (\autoref{eq:perf}) -- higher is better -- for the combinatorially (near) complete datasets, DHFR and TrpB and TFBIND8. The random method is implemented by drawing $B$ samples uniformly.}
\label{fig:fl_res_gp}
\end{figure}

\subsection{Ehrlich Function HOLO Results}
\label{app:ehrlich}

See \autoref{fig:ehrlich_holo_bbo} for \gls{bbo} results on the original \textsc{holo} Ehrlich function implementation~\citep{stanton2024closed}. We present additional diversity scores for these and the \textsc{poli} implementation in \autoref{sub:diversity}.

\begin{figure}[htb]
    \centering
    \subcaptionbox{$M=15$\label{sfig:ehrlich-holo-15}}
    {\includegraphics[width=.329\textwidth]{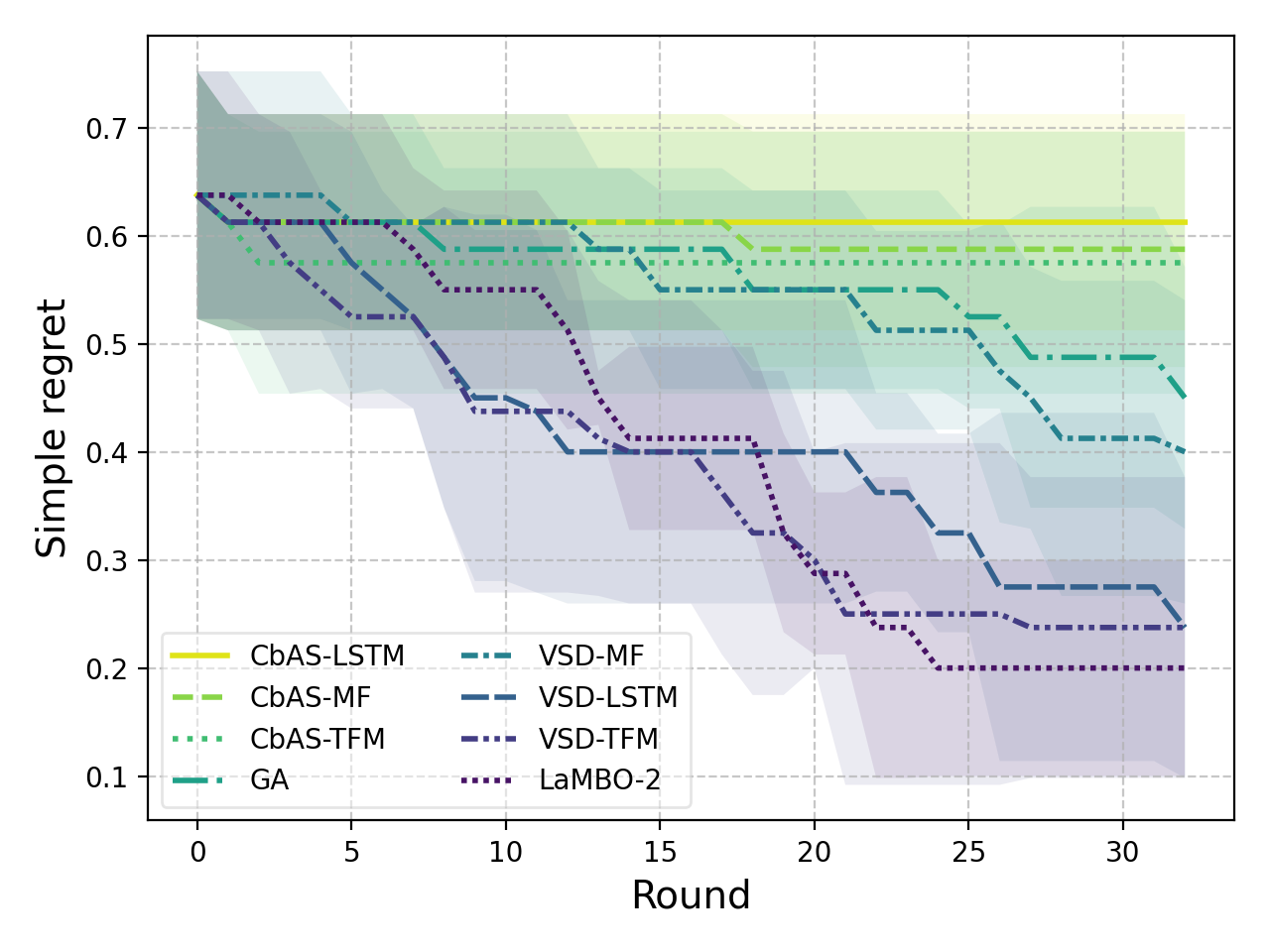}}
    \subcaptionbox{$M=32$\label{sfig:ehrlich-holo-32}}
    {\includegraphics[width=.329\textwidth]{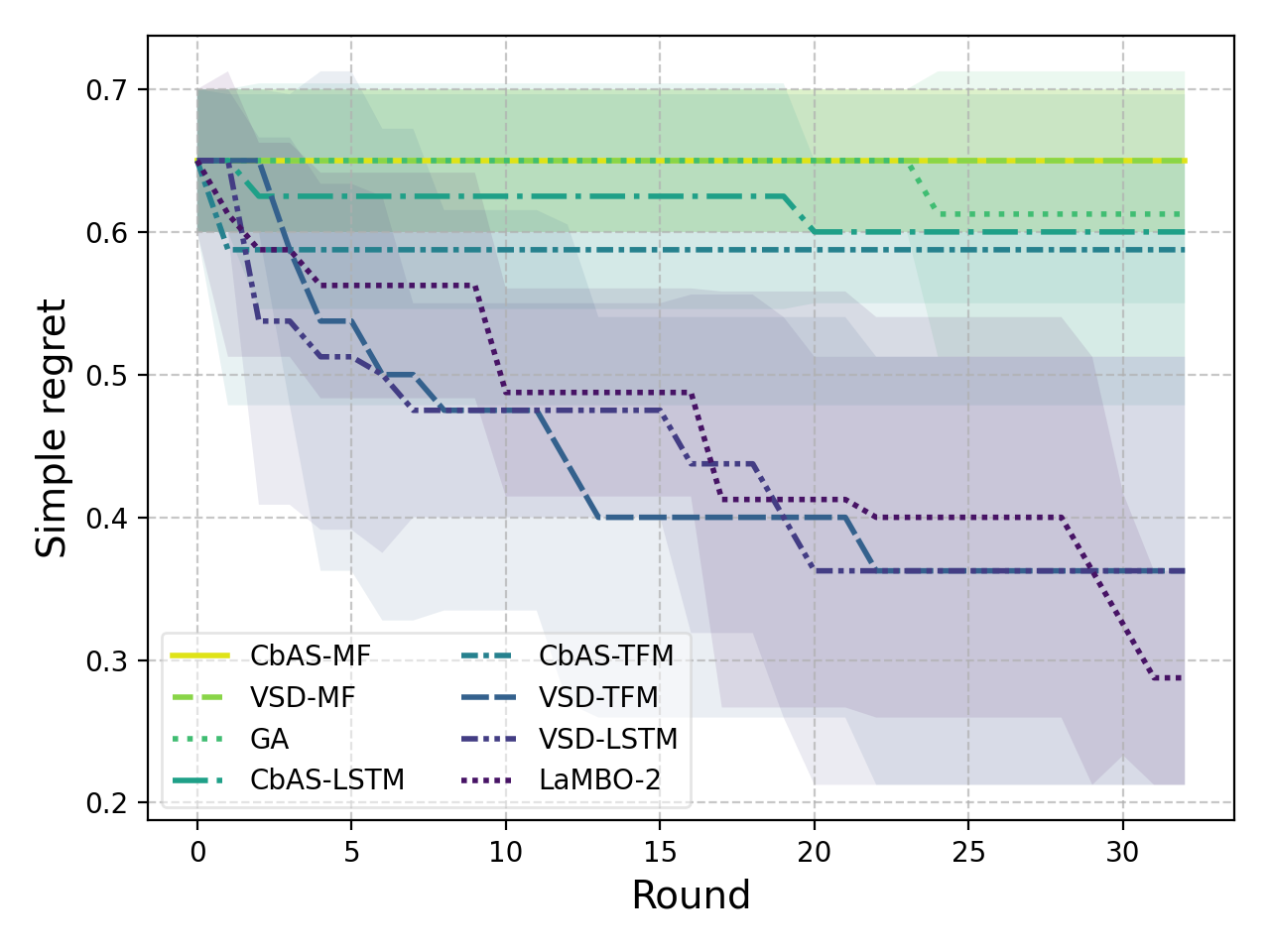}}
    \subcaptionbox{$M=64$\label{sfig:ehrlich-holo-64}}
    {\includegraphics[width=.329\textwidth]{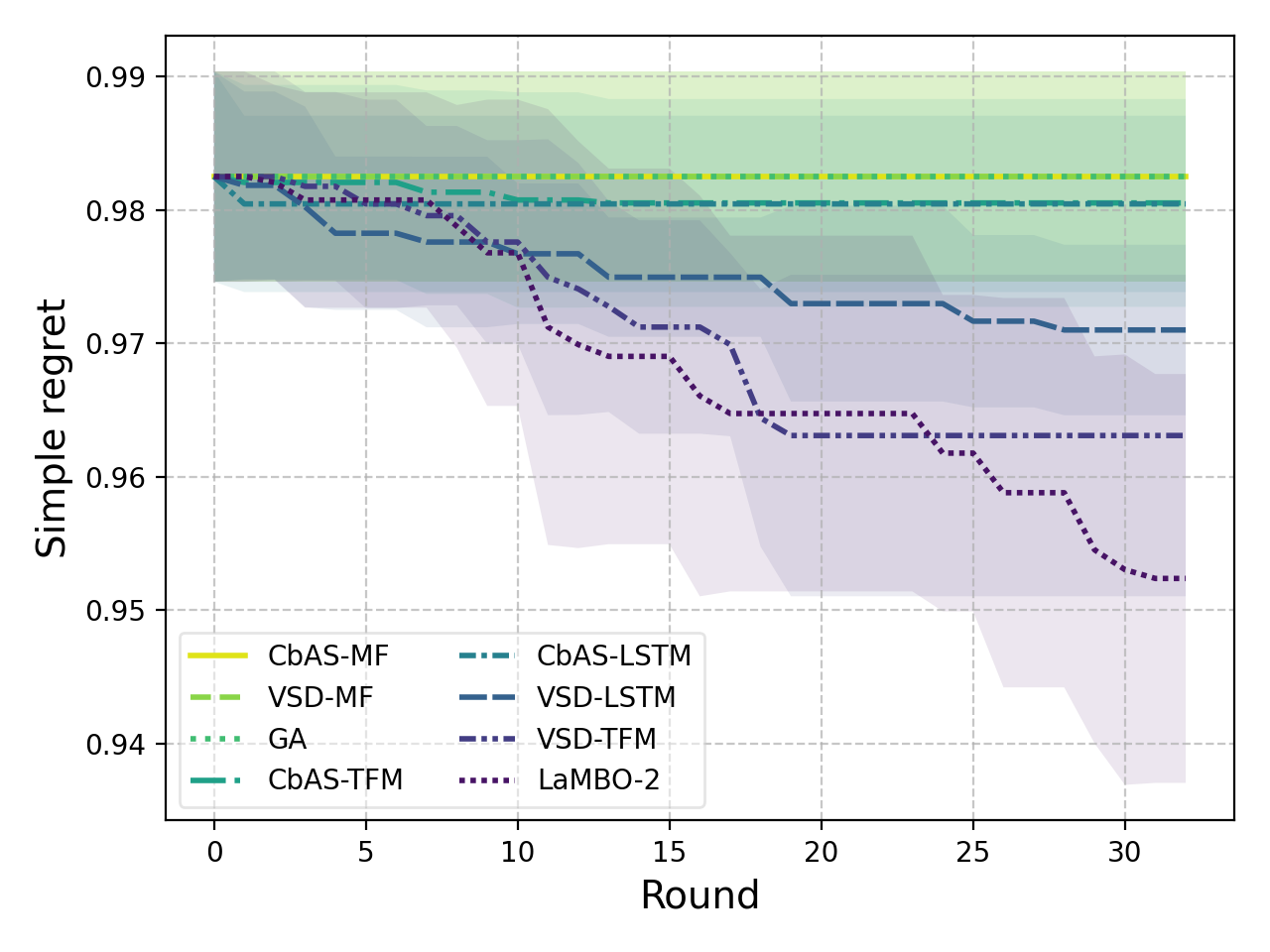}}
    \caption{Ehrlich function (\textsc{holo} implementation) \gls{bbo} results. \Gls{vsd} and \gls{cbas} with different variational distributions; mean field (MF), \gls{lstm} and transformer (TFM), compared against genetic algorithm (GA) and LaMBO-2 baselines.}
    \label{fig:ehrlich_holo_bbo}
\end{figure}

\subsection{Diversity Scores}
\label{sub:diversity}

The diversity of batches of candidates is a common thing to report in the literature, and to that end we present the diversity of our results here. We have taken the definition of pair-wise diversity from \citep{jain2022biological} as,
\begin{align}
    \mathrm{Diversity}_t = \frac{1}{B(B - 1)} \sum_{\obs_i \in \data_{Bt}} \sum_{\obs_j \in \data_{Bt} \setminus \{\obs_i\}} \mathrm{Lev}(\obs_i, \obs_j),
    \label{eq:divers}
\end{align}
where $\mathrm{Lev} : \obsspace \times \obsspace \to \N_0$ is the Levenshtein distance. We caution the reader as to the interpretation of these results however, as more diverse batches often do not lead to better performance, precision, recall or simple regret (as can be seen from the Random method results). Though insufficient diversity can also explain poor performance, as in the case of \gls{bore}. Results for the fitness landscape experiment are presented in \autoref{fig:fl_res_div}, and black-box optimization for AAV \& GFP in \autoref{fig:bbo_res_div} and Ehrlich functions in \autoref{fig:bbo_ehrlich_div}.

\begin{figure}[htb]
\centering
\subcaptionbox{DHFR\label{sfig:dhfr_fl_div}}
    {\includegraphics[width=.32\textwidth]{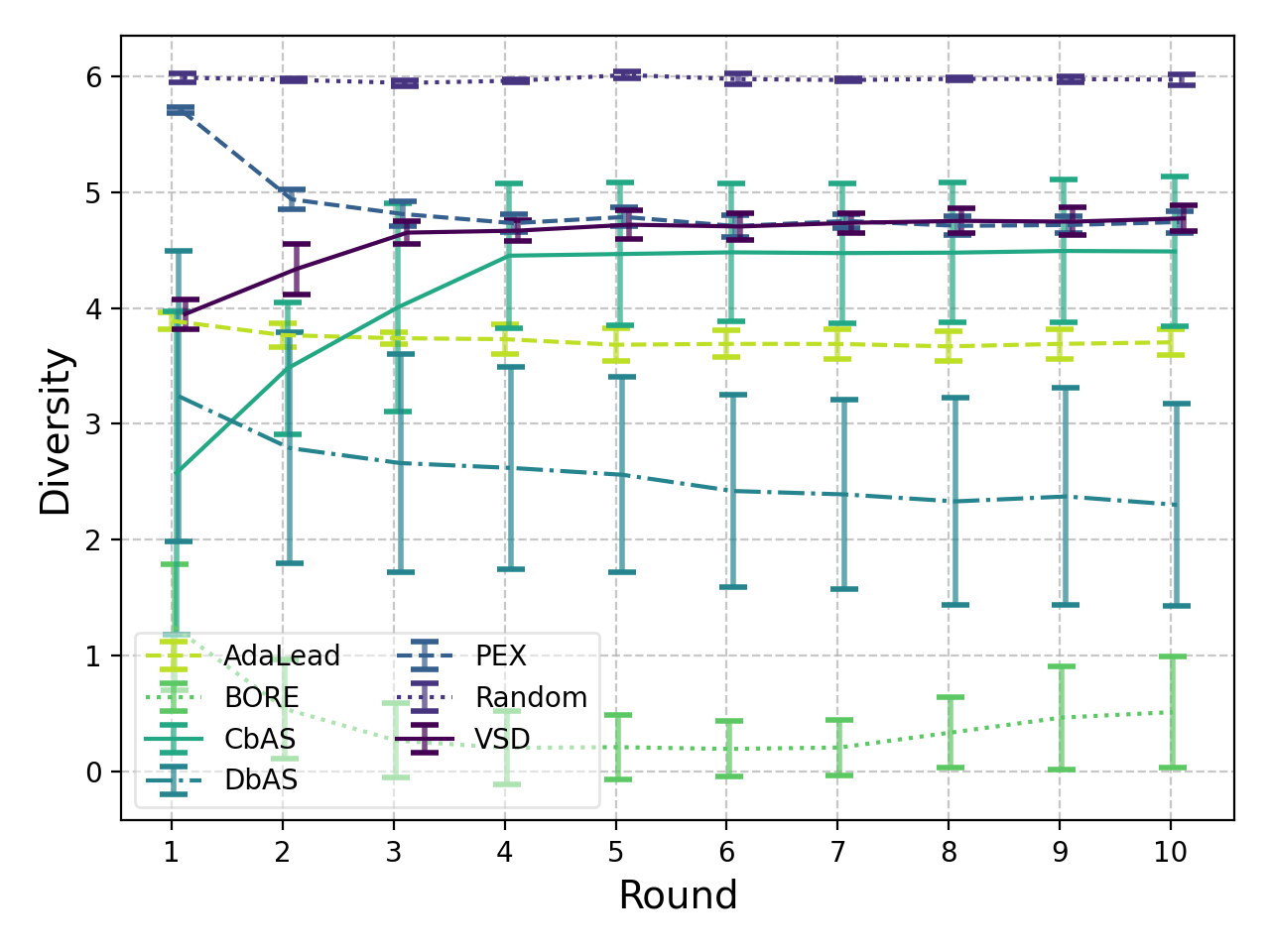}}
\subcaptionbox{TrpB\label{sfig:trpb_fl_div}}
    {\includegraphics[width=.32\textwidth]{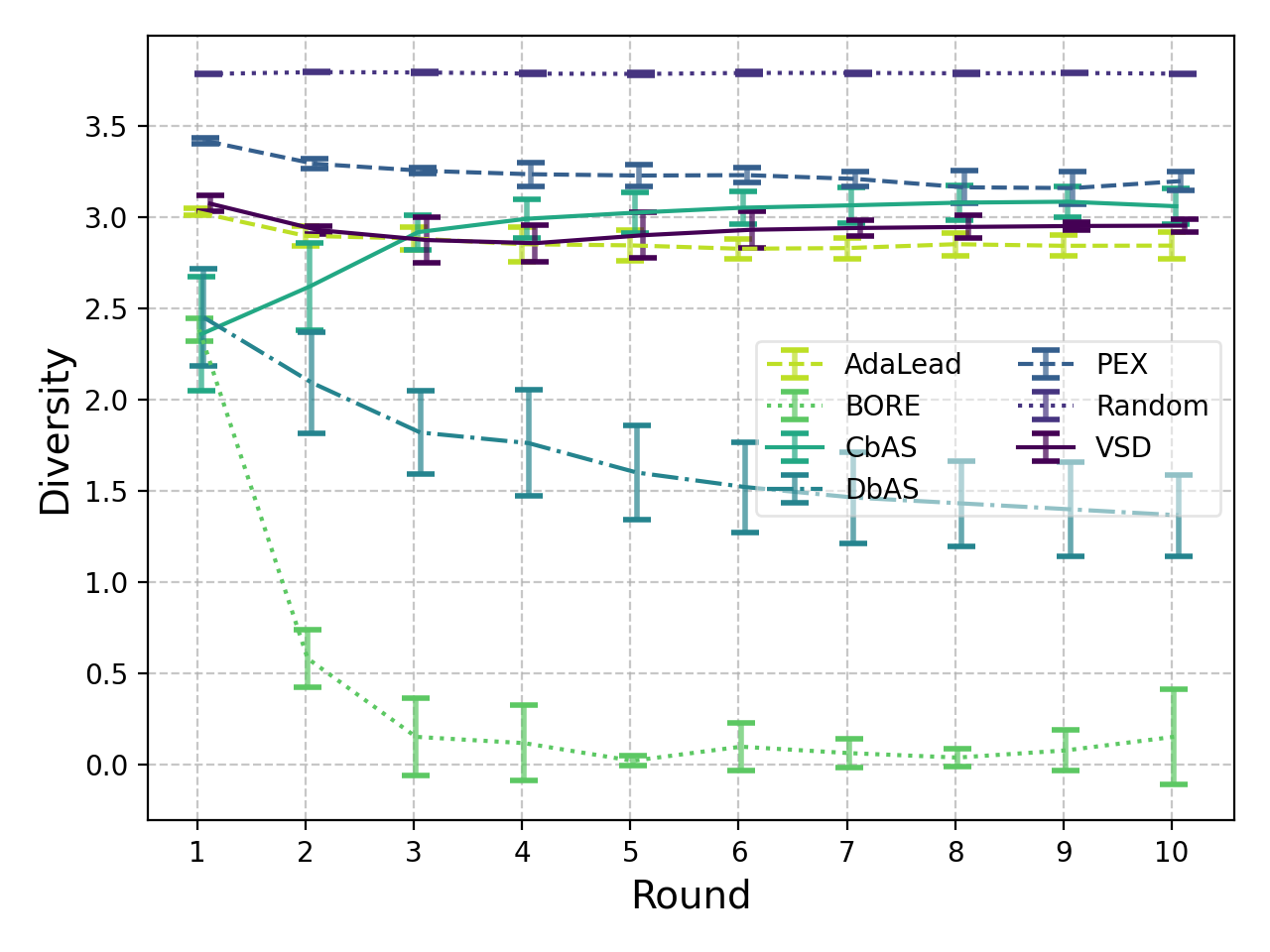}}
\subcaptionbox{TFBIND8\label{sfig:tfbind8_fl_div}}
    {\includegraphics[width=.32\textwidth]{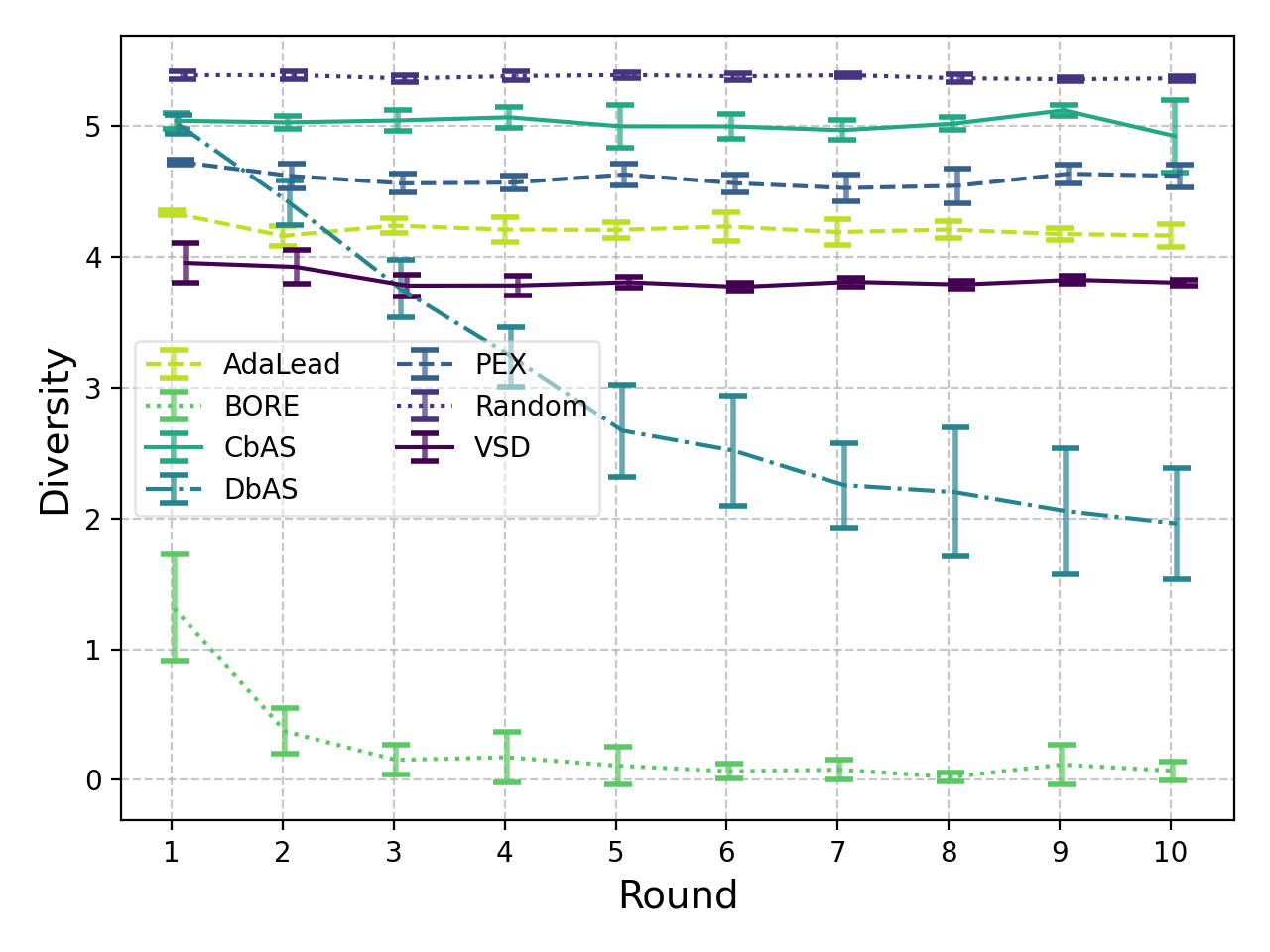}}
\caption{Fitness landscape diversity results. Higher is more diverse, as defined by \autoref{eq:divers}.}
\label{fig:fl_res_div}
\end{figure}

\begin{figure}[htb]
\centering
\rotatebox{90}{\hspace{1.5cm}GFP}
\includegraphics[width=.32\textwidth]{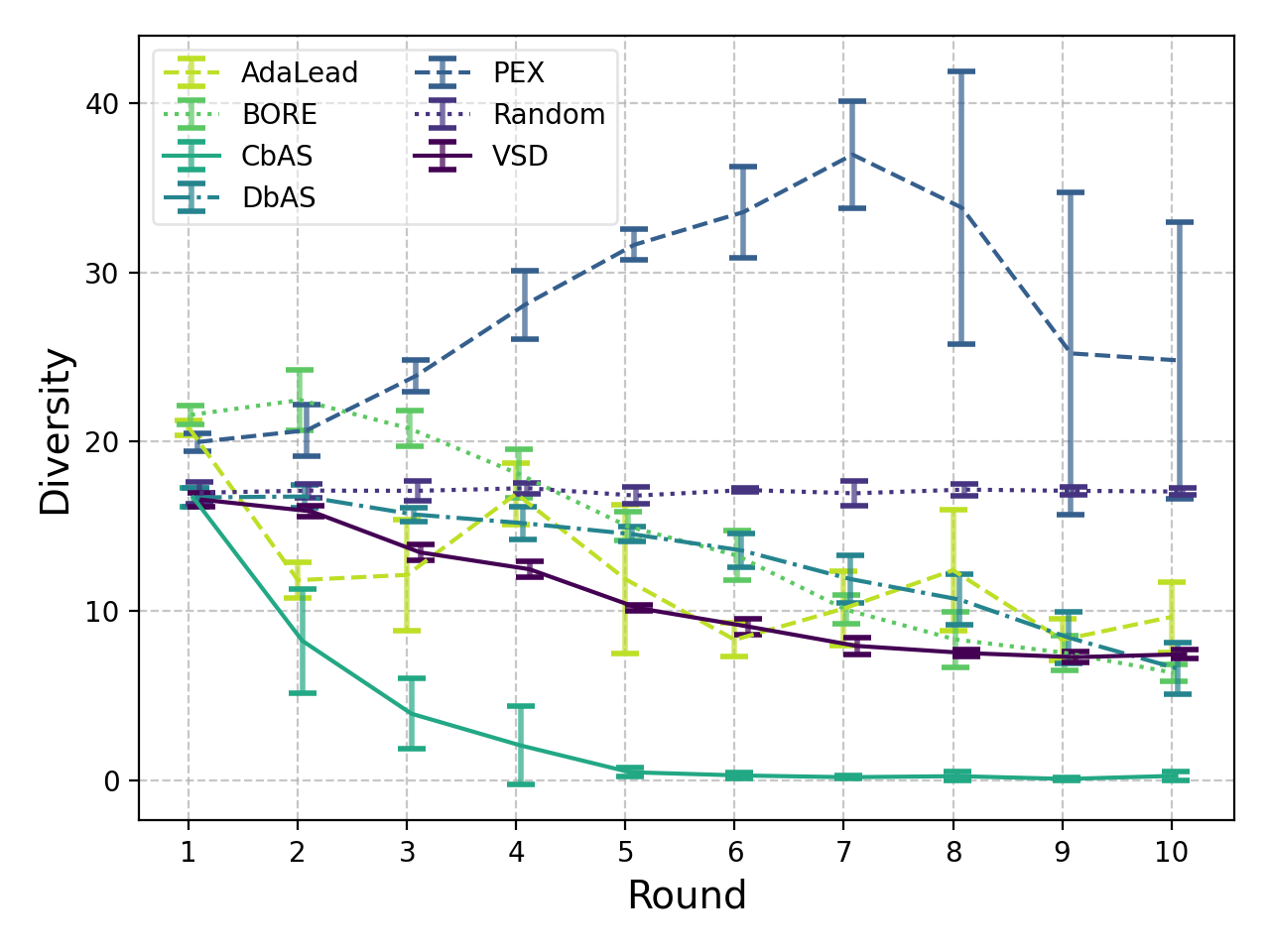}
\includegraphics[width=.32\textwidth]{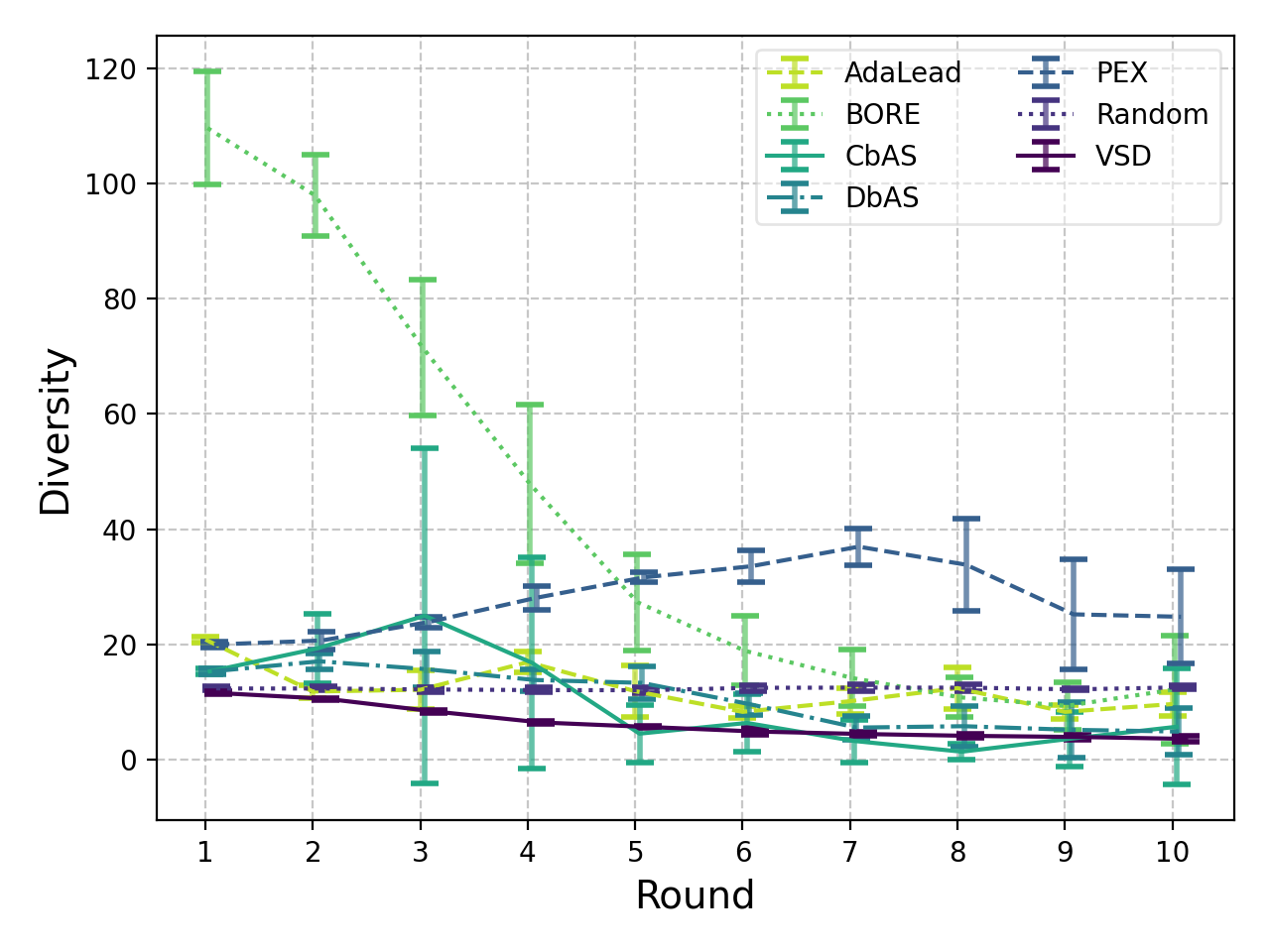}
\includegraphics[width=.32\textwidth]{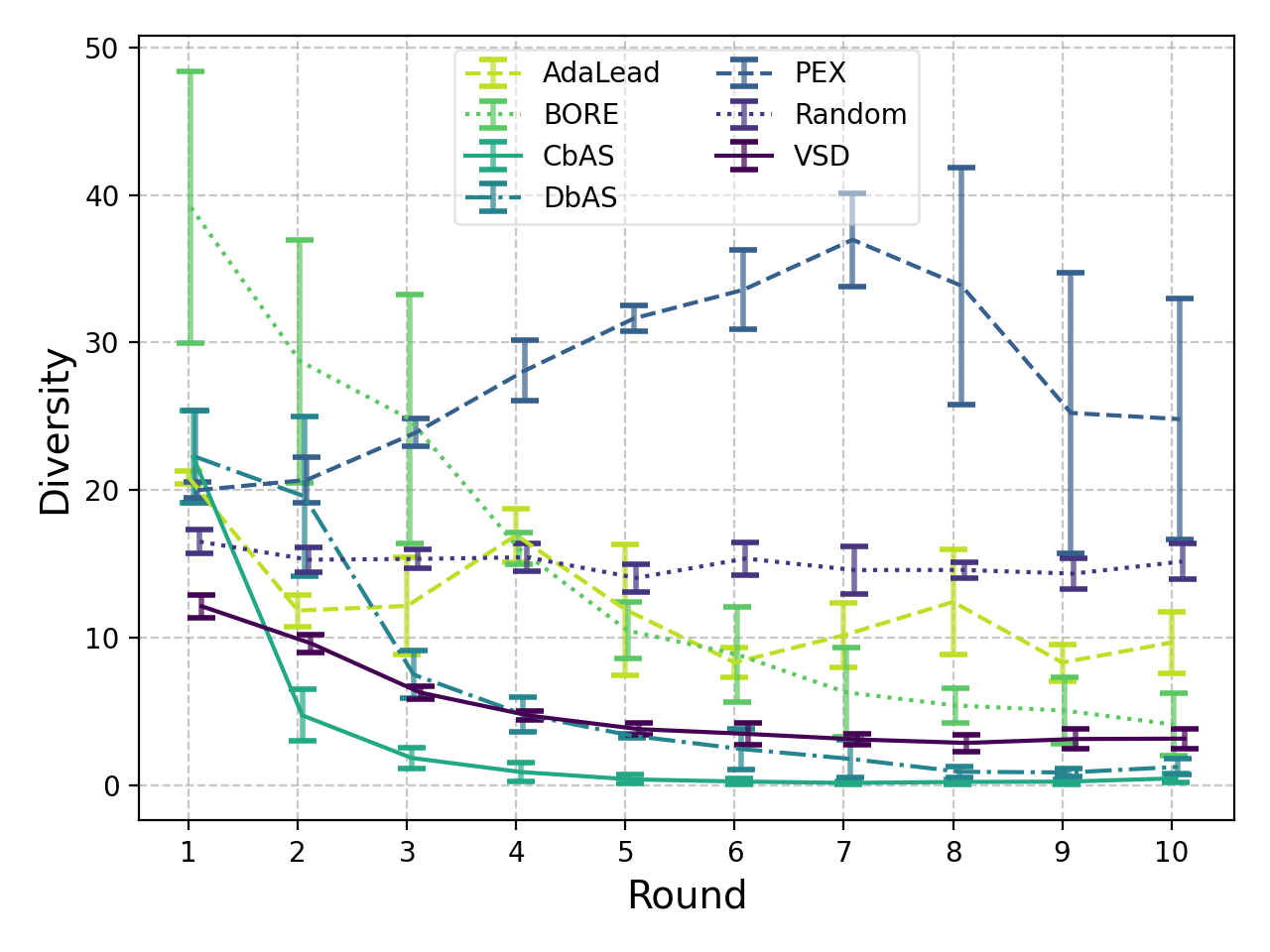} \\
\rotatebox{90}{\hspace{1.5cm}AAV}
\subcaptionbox{Independent\label{sfig:aav_div_mc}}
    {\includegraphics[width=.32\textwidth]{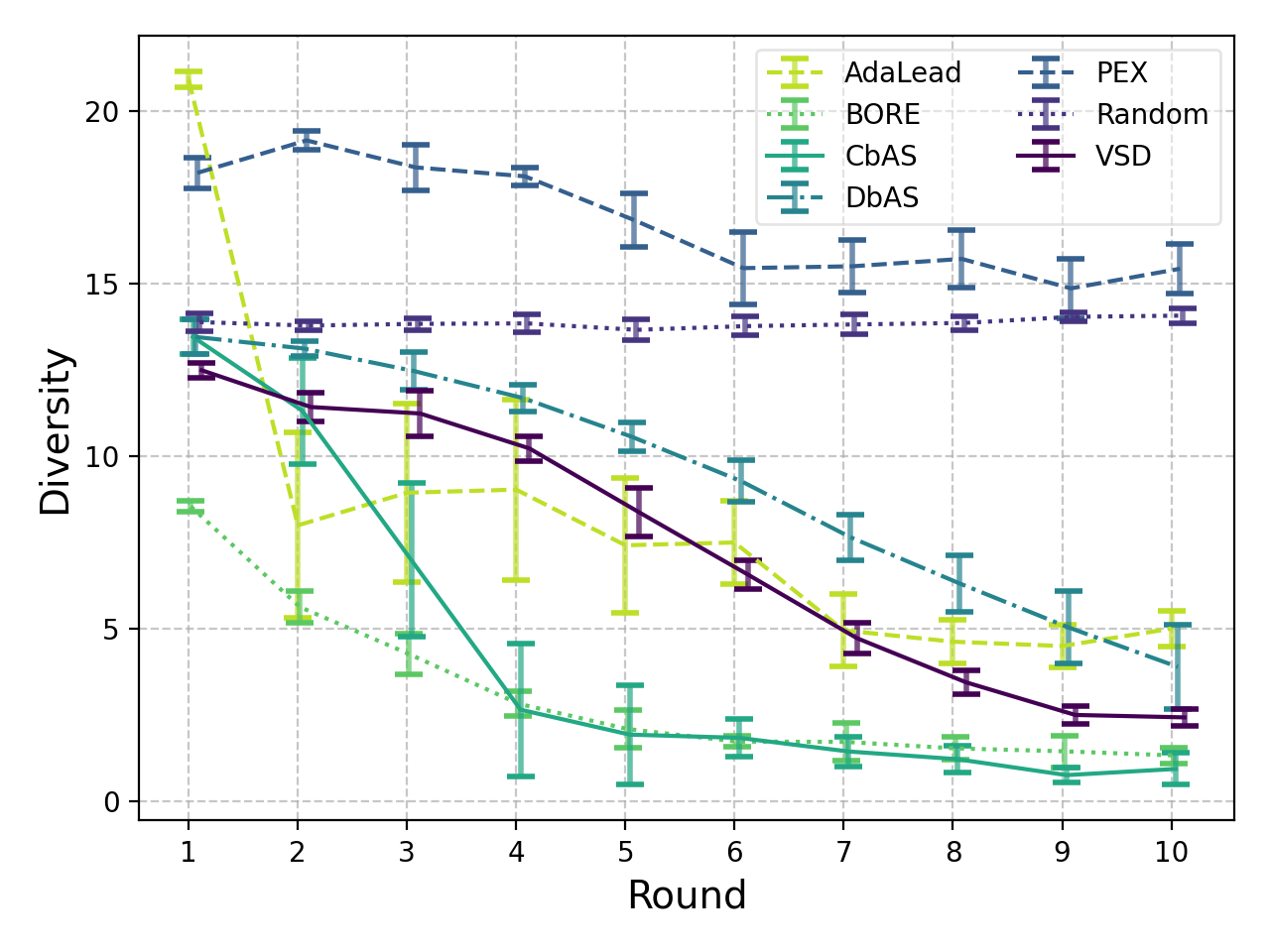}}
\subcaptionbox{LSTM\label{sfig:aav_div_lstm}}
    {\includegraphics[width=.32\textwidth]{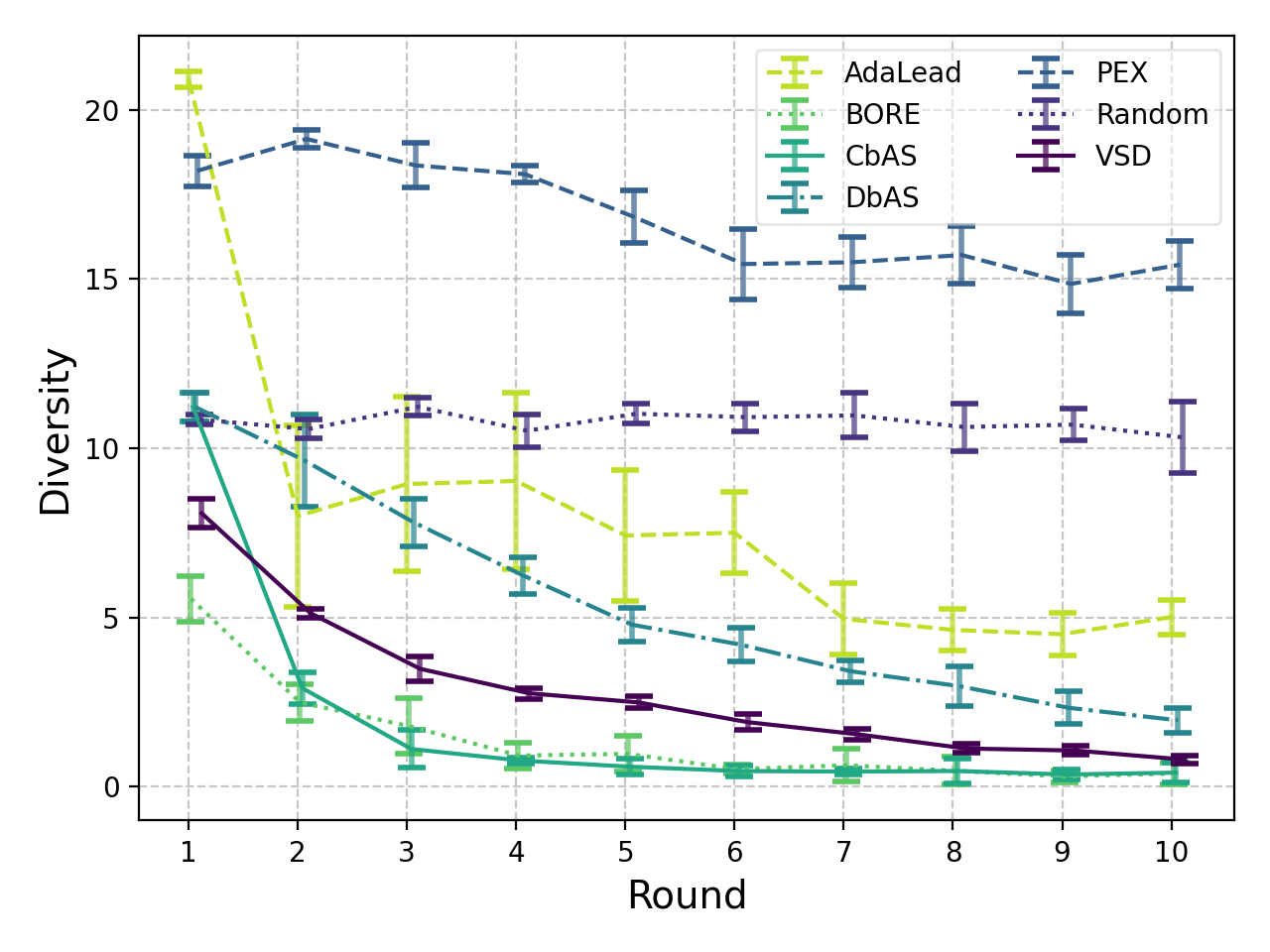}}
\subcaptionbox{Transformer\label{sfig:aav_div_dtfm}}
    {\includegraphics[width=.32\textwidth]{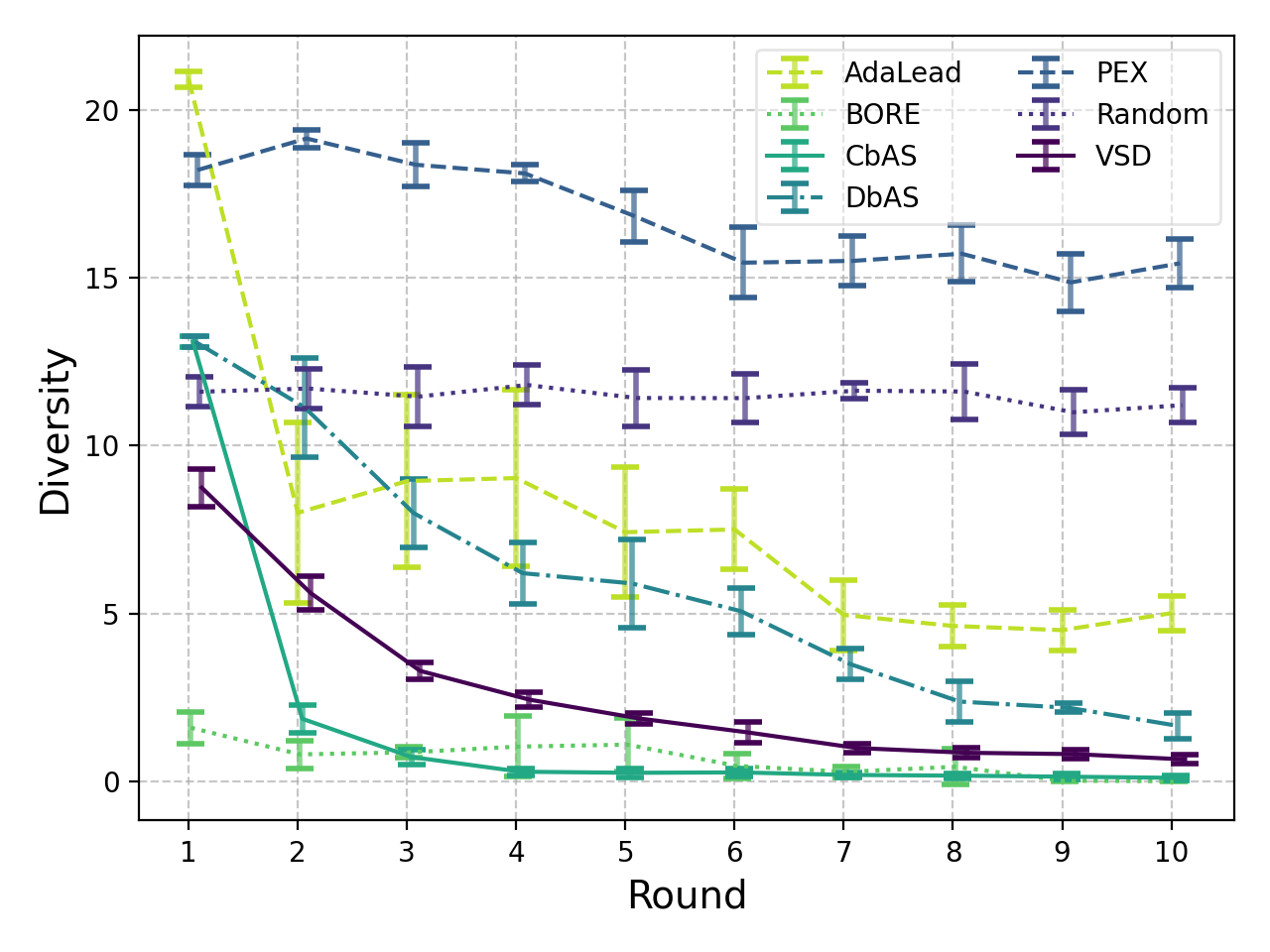}}
\caption{Black-box optimization results for diversity on GFP and AAV with independent and auto-regressive variational distributions. Higher is more diverse, as defined by \autoref{eq:divers}. The \gls{pex} and AdaLead results are replicated between the plots, since they are unaffected by choice of variational distribution.}
\label{fig:bbo_res_div}
\end{figure}

\begin{figure}[htb]
\centering
\rotatebox{90}{\hspace{1.5cm}\textsc{poli}}
\includegraphics[width=.32\textwidth]{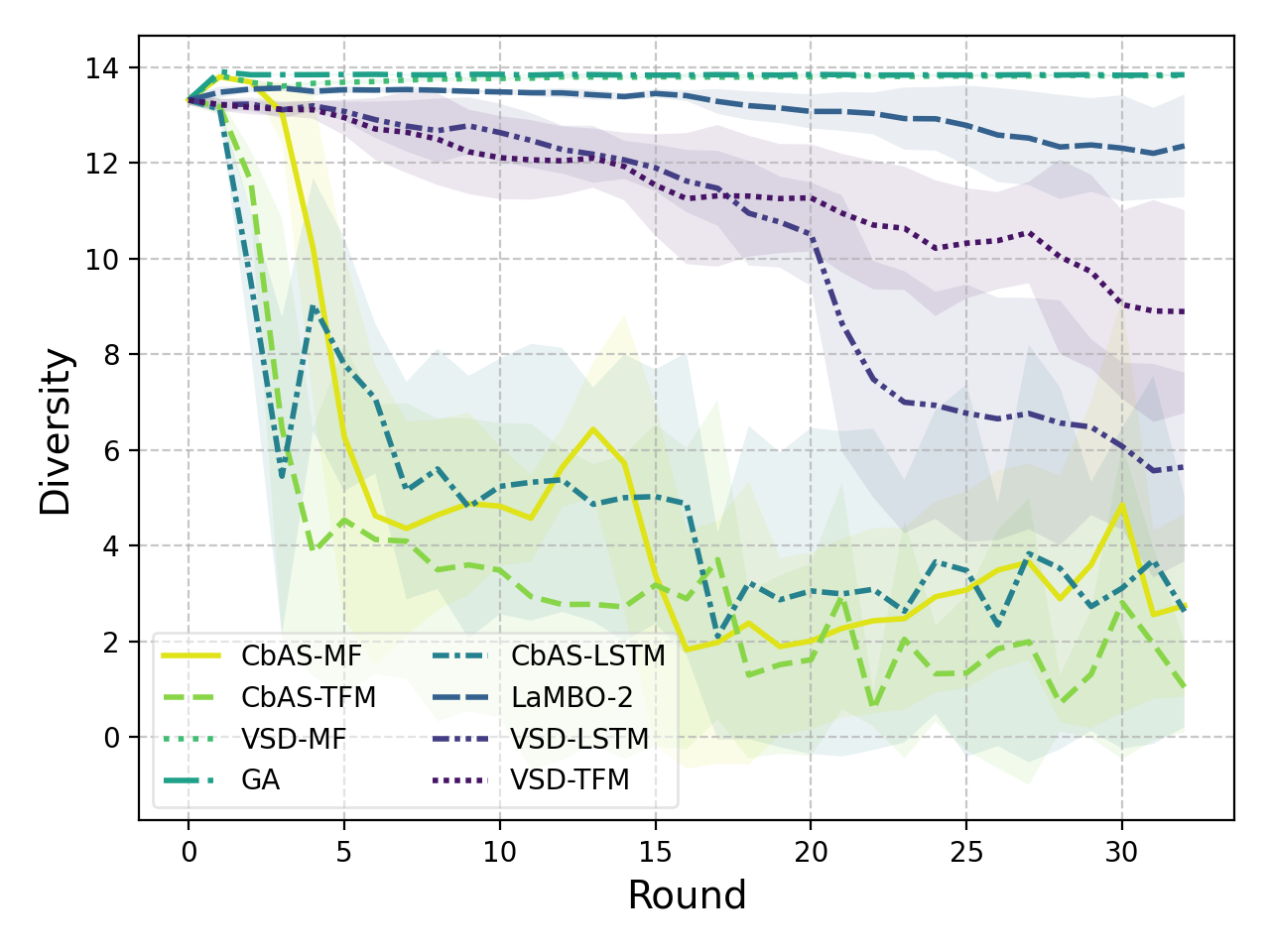}
\includegraphics[width=.32\textwidth]{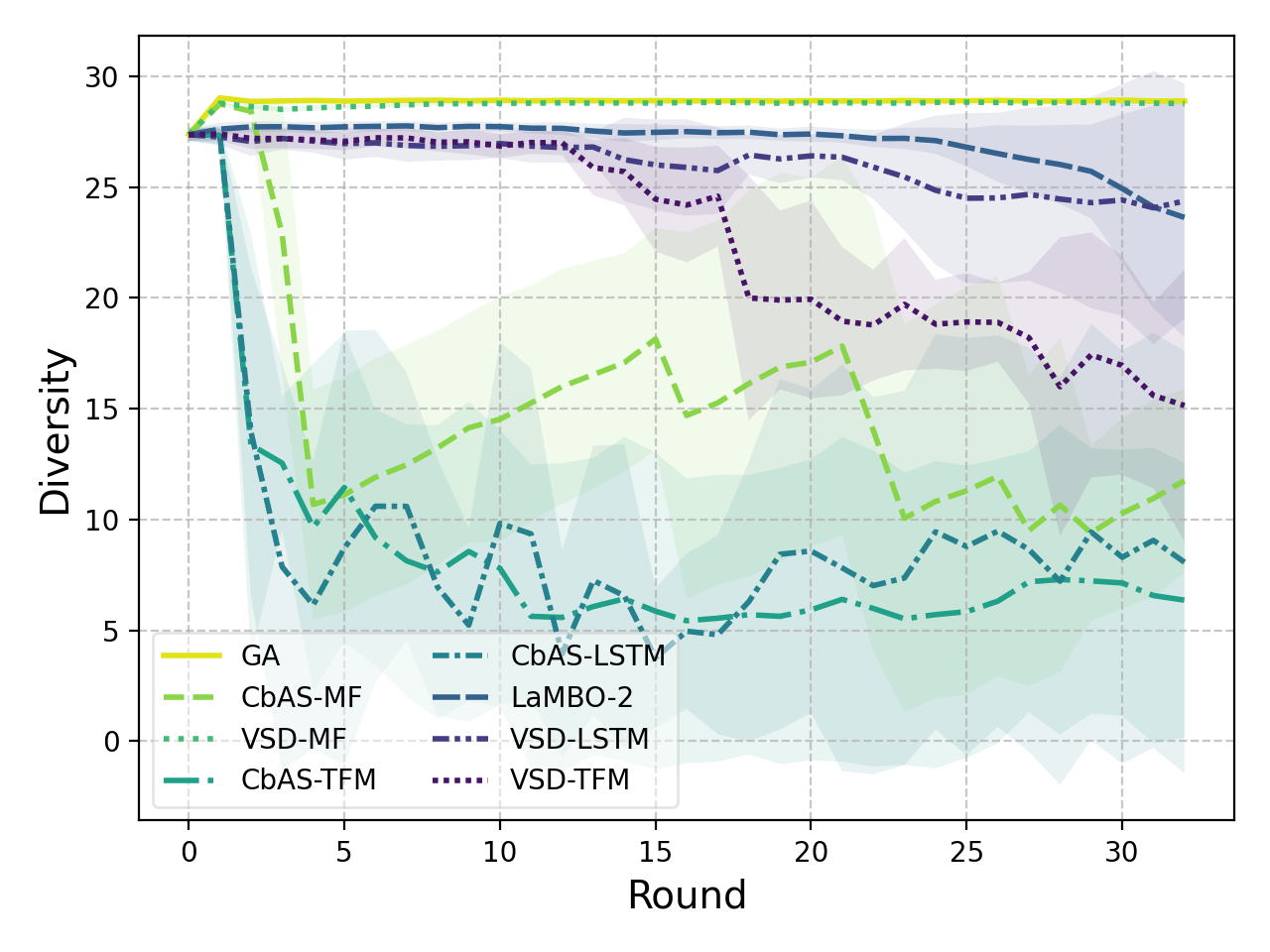}
\includegraphics[width=.32\textwidth]{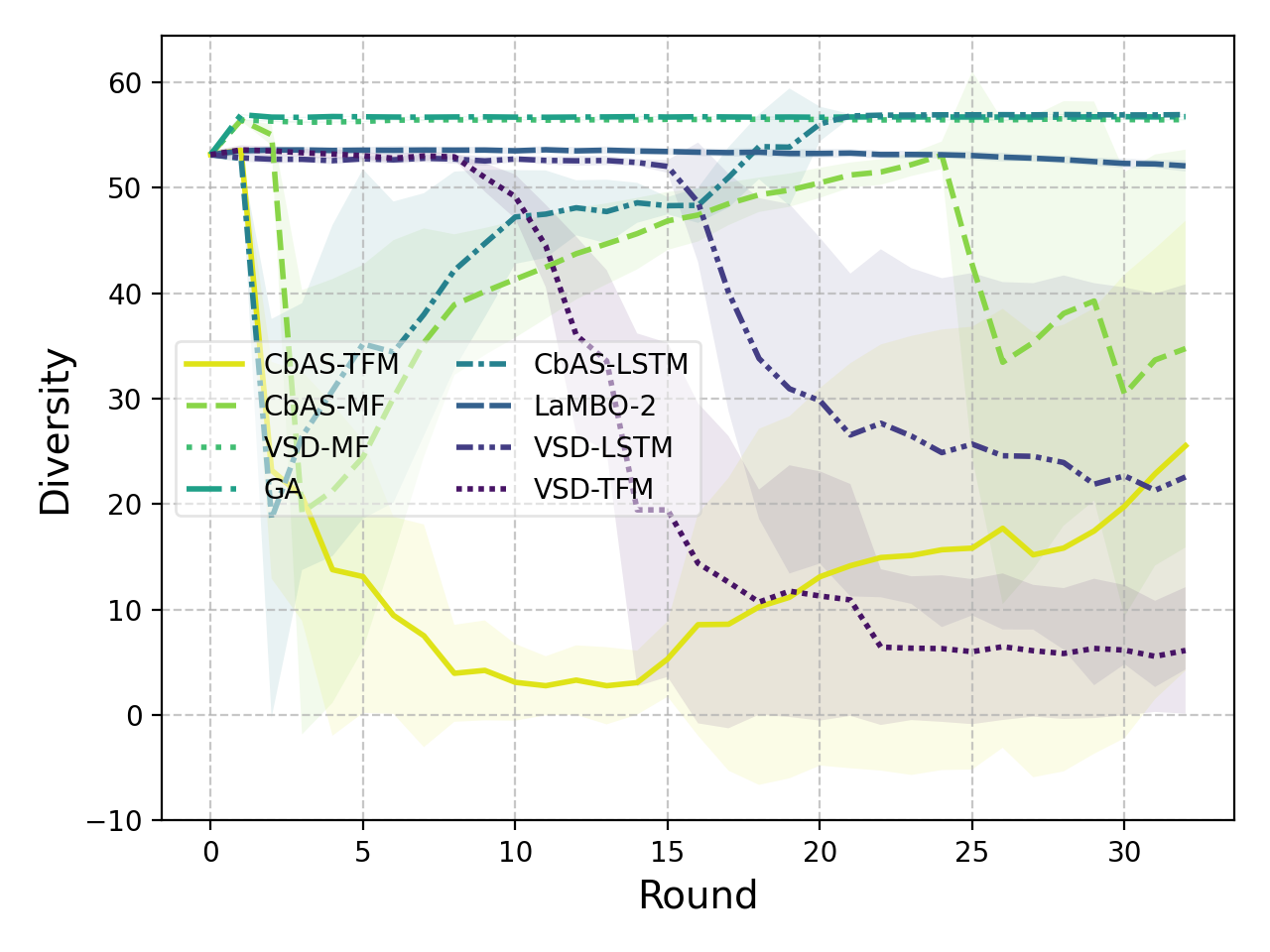} \\
\rotatebox{90}{\hspace{1.5cm}\textsc{holo}}
\subcaptionbox{$M=15$\label{sfig:ehrlich_div_15}}
    {\includegraphics[width=.32\textwidth]{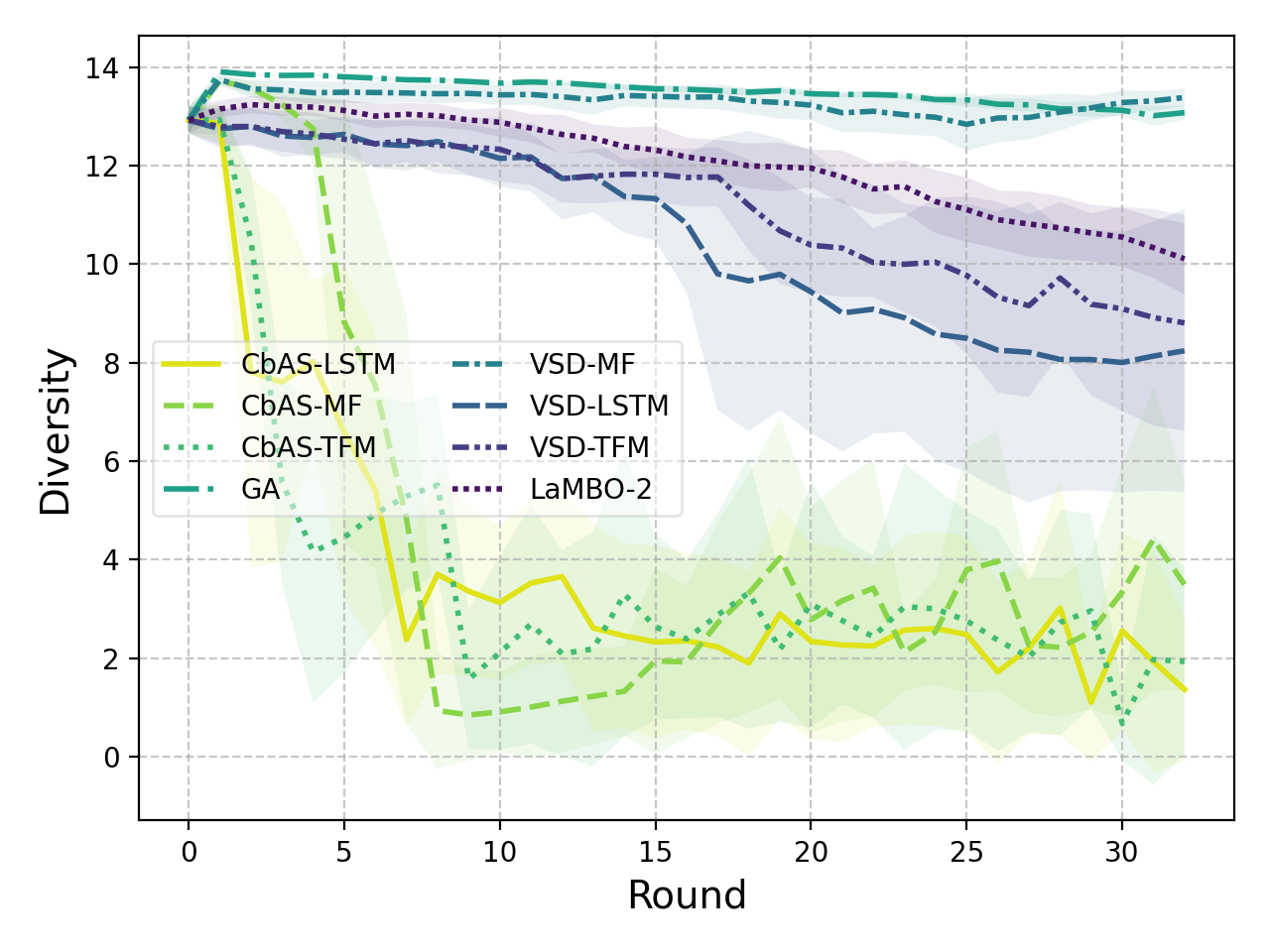}}
\subcaptionbox{$M=32$\label{sfig:ehrlich_div_32}}
    {\includegraphics[width=.32\textwidth]{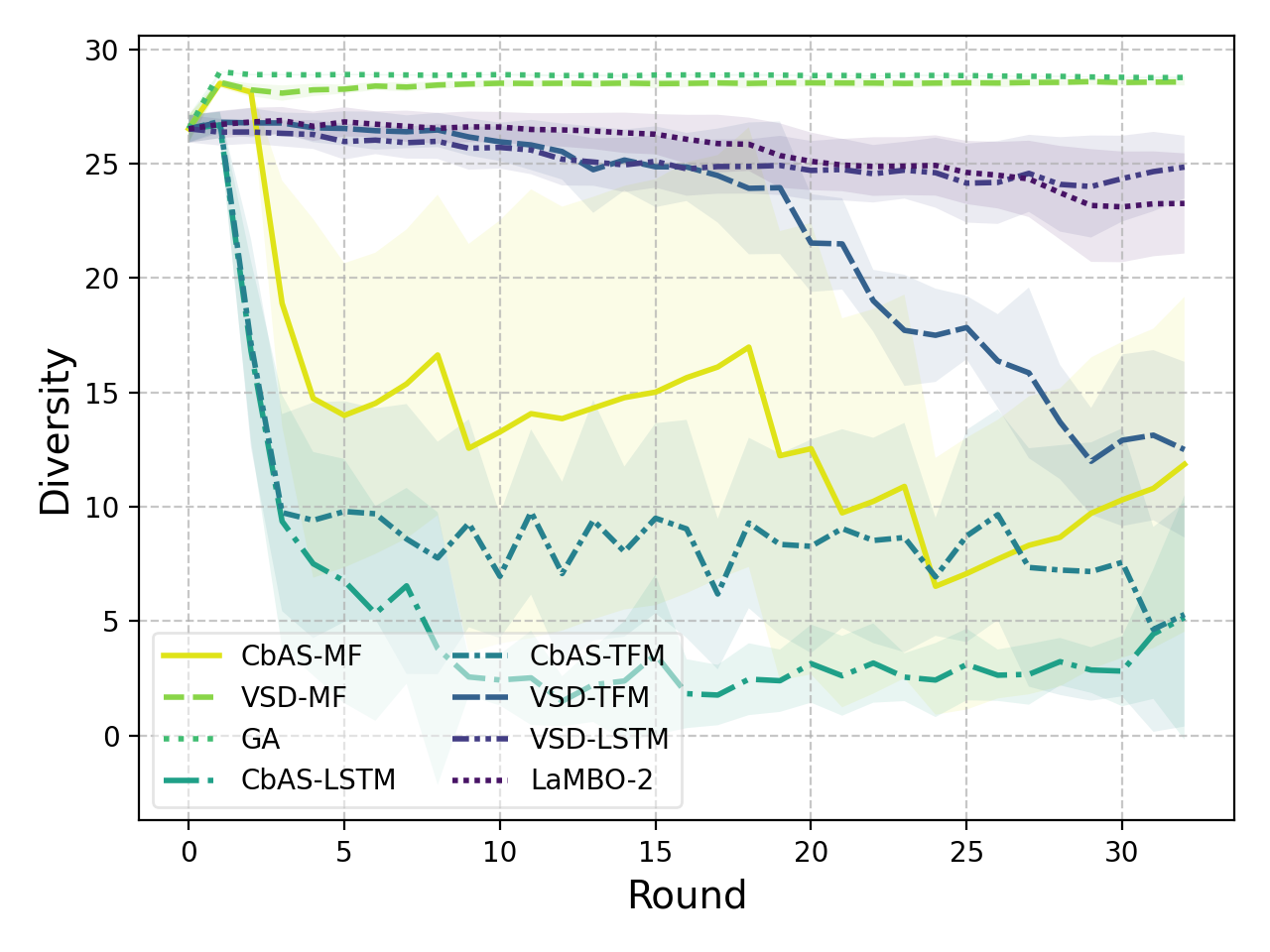}}
\subcaptionbox{$M=64$\label{sfig:ehrlich_div_64}}
    {\includegraphics[width=.32\textwidth]{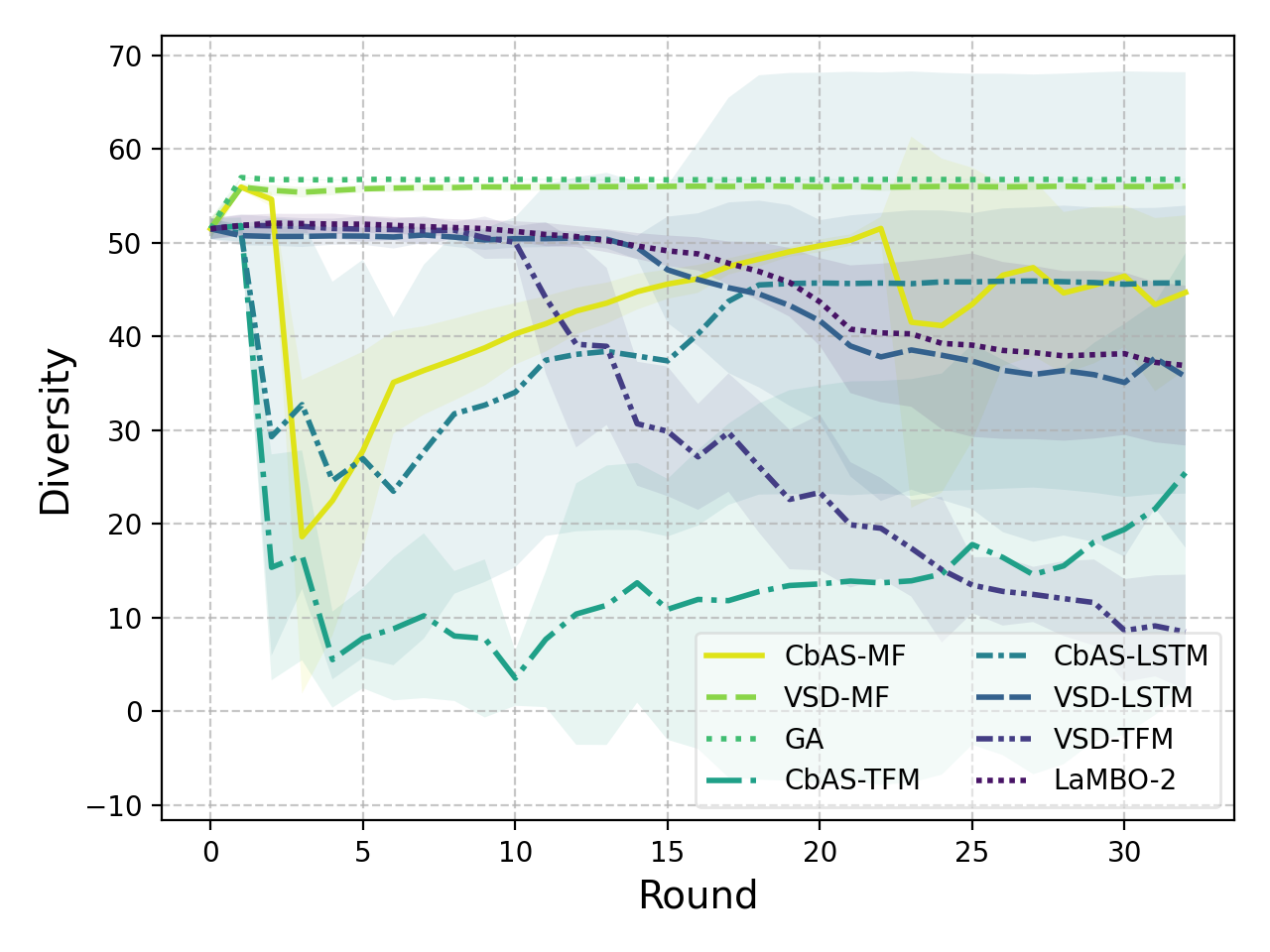}}
\caption{Black-box optimization results for diversity on the \textsc{poli} and \textsc{holo} implementations of the Ehrlich functions. Higher is more diverse, as defined by \autoref{eq:divers}.}
\label{fig:bbo_ehrlich_div}
\end{figure}

\subsection{Ablations -- Variational and Prior Distributions}
\label{sub:abl}

In \autoref{fig:ablation} we present ablation results for \gls{vsd} using different priors and variational distributions. We use the \gls{bbo} experimental datasets for this task as they are higher-dimensional and so more sensitive to these design choices. We test the following prior and variational posterior distributions:
\begin{description}[leftmargin=!,labelwidth=\widthof{\bfseries DTFM}]
    \item[IU] Independent categorical variational posterior distribution of the form in \autoref{eq:proposal}, and a uniform prior distribution, $\prob{\obs} = \prod^M_{m=1} \categc{\acid_m}{\mathbf{1} \cdot |\acidspace|^{-1}}$.
    \item[I] Independent categorical prior and variational posterior of the form in \autoref{eq:proposal}. The prior is fit using \gls{ml} on the initial \gls{cpe} training data.
    \item[LSTM] \gls{lstm} prior and variational posterior of the form \autoref{eq:autoregressive-proposal}. The prior is fit using \gls{ml} on the initial \gls{cpe} training data.
    \item[DTFM] Decoder-only causal transformer prior and variational posterior of the form \autoref{eq:autoregressive-proposal}. The prior is fit using \gls{ml} on the initial \gls{cpe} training data.
    \item[TAE] Independent categorical prior and a transition-style auto-encoder variational posterior of the form \autoref{eq:transition-proposal}, where we use two-hidden layer MLPs for the encoder and decoder. The prior is fit using \gls{ml} on the initial \gls{cpe} training data.
    \item[TCNN] Independent categorical prior and a transition-style convolutional auto-encoder variational posterior of the form \autoref{eq:transition-proposal}, where we use a convolutional encoder, and transpose convolutional decoder. The prior is fit using \gls{ml} on the initial \gls{cpe} training data.
\end{description}
We use the informed-independent priors with the transition variational distributions since they are somewhat counter-intuitive to use as priors themselves.

\begin{figure}[htb]
\centering
\includegraphics[width=.49\textwidth]{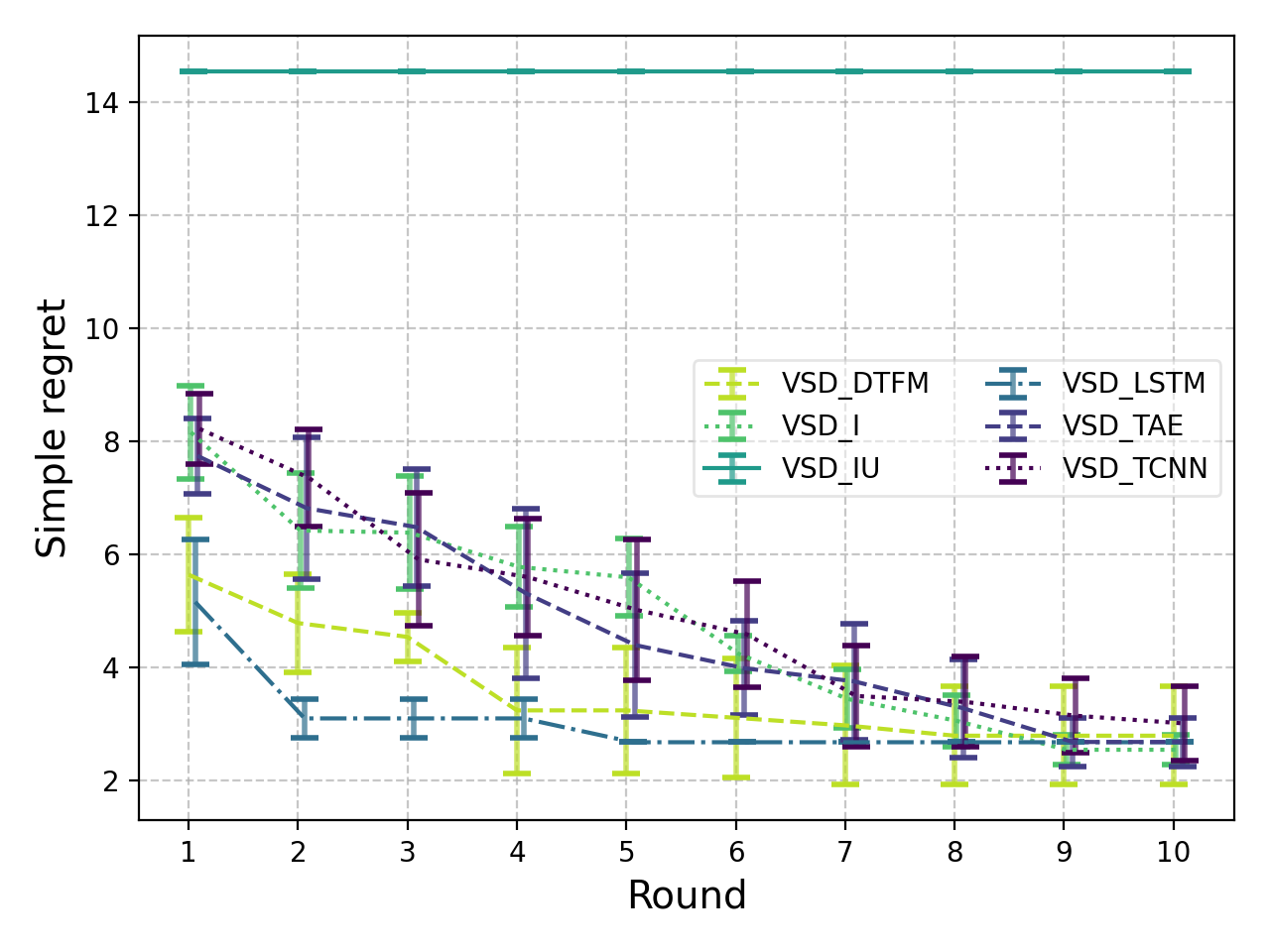}
\includegraphics[width=.49\textwidth]{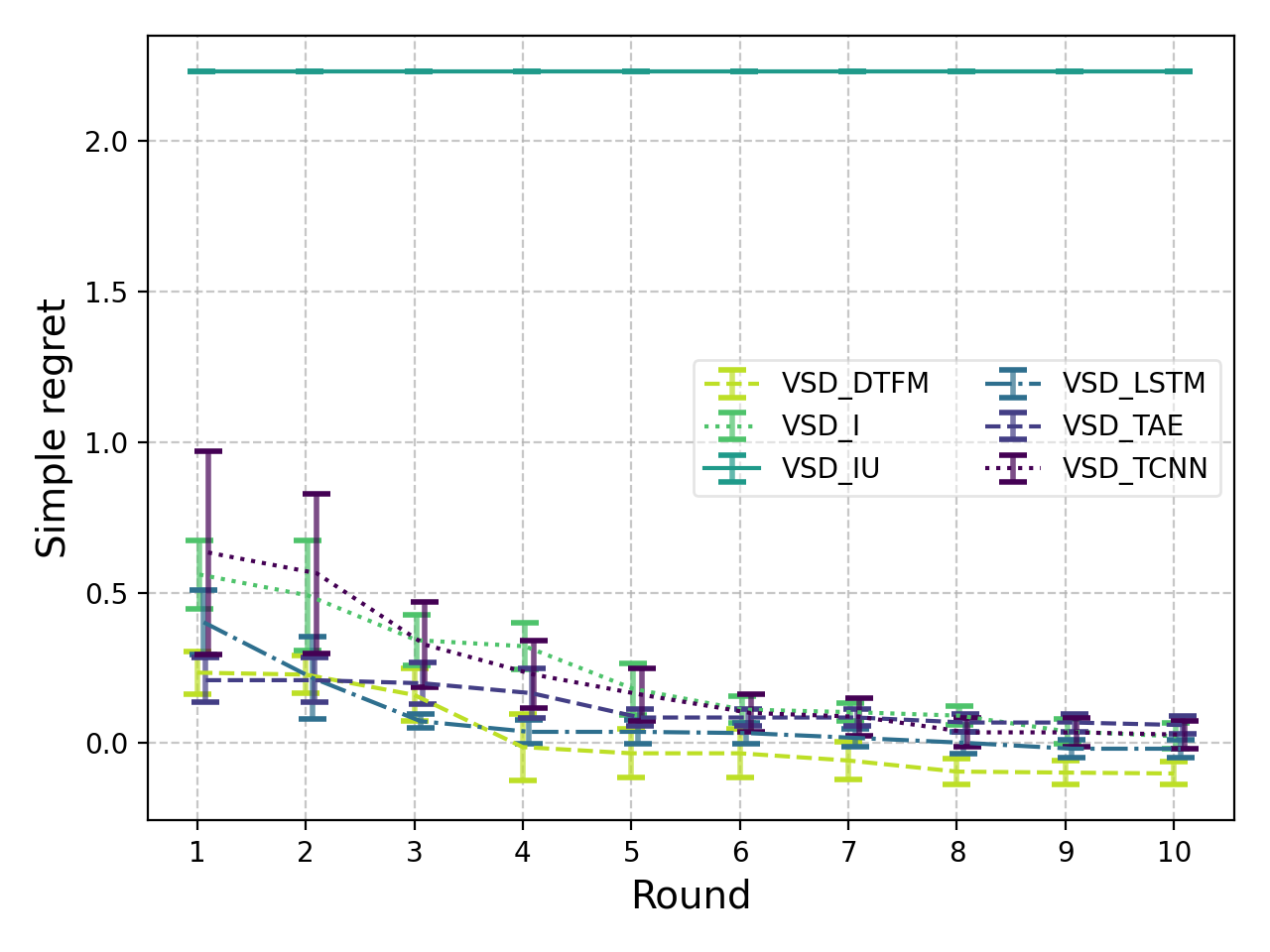} \\
\subcaptionbox{AAV\label{sfig:aav_abl}}
    {\includegraphics[width=.49\textwidth]{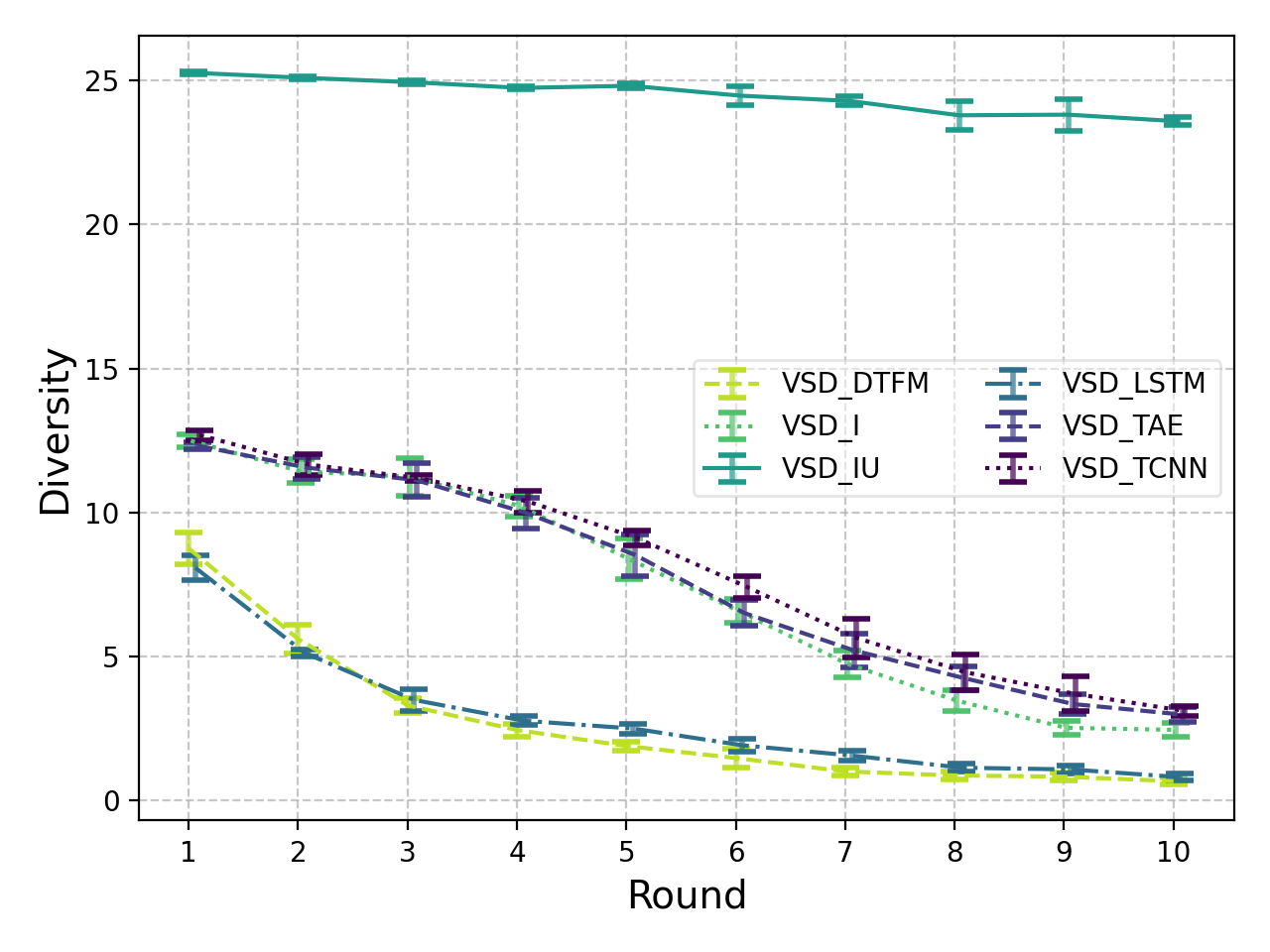}}
\subcaptionbox{GFP\label{sfig:gfp_abl}}
    {\includegraphics[width=.49\textwidth]{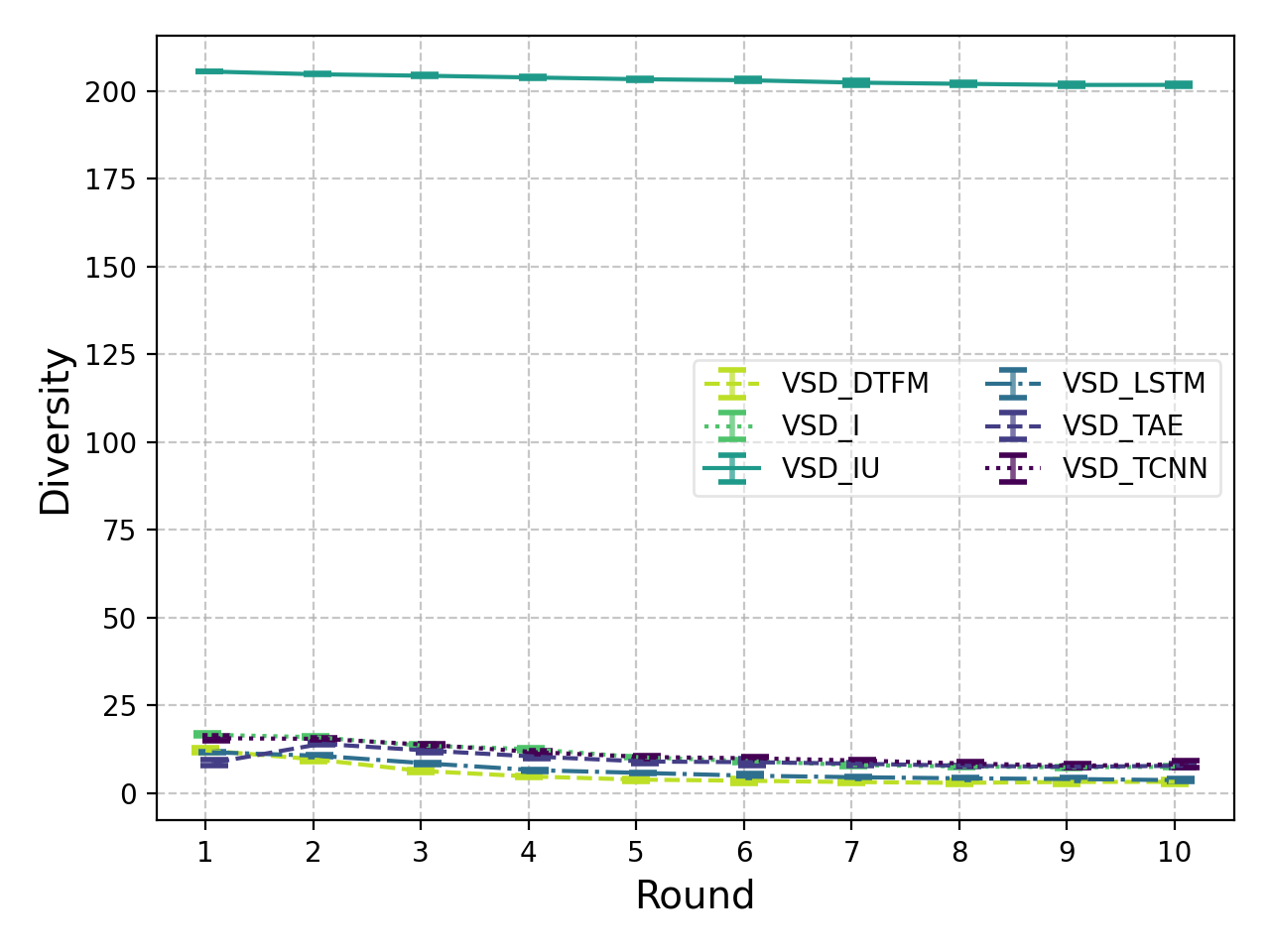}}
\caption{Ablation results for the AAV and GFP \gls{bbo} experiments. \Gls{vsd} is trialed with different prior and variational posterior combinations, ``I'' indicates a simple independent informed prior and posterior, ``IU'' is the same but with a uniform prior, ``LSTM'' and ``DTFM'' are the \gls{lstm} and decoder only transformer prior and posteriors, ``TCNN'' and ``TAE'' are transition convolutional encoder-decoder and auto-encoder posteriors, with informed independent priors. See text for details.}
\label{fig:ablation}
\end{figure}

From \autoref{fig:ablation} we can see that while using an uninformative prior works in the lower-dimensional fitness landscape experiments, using an informative prior is crucial for these higher dimensional problems. We found a similar result when using this uninformative prior with \gls{cbas}, or using a uniform initialization with \gls{dbas} and \gls{bore}. The methods are not able to make any significant progress within the experimental budget given. The independent and transition variational distributions achieve similar performance, whereas the auto-regressive models generally outperform all others. This is because of the \gls{lstm} and transformer's superior generalization performance when generating sequences -- measured both when training the priors (on held-out sequences) and during \gls{vsd} adaptation.

\section{Theoretical analysis for GP-based CPEs}
\label{app:theory}

In this section, we present theoretical results concerning \gls{vsd} and its estimates when equipped with Gaussian process regression models \citep{rasmussen2006}. We show that \gls{vsd} sampling distributions converge to a target distribution that characterizes the level set given by $\thresh$. The approximation error mainly depends on the predictive uncertainty of the probabilistic model with respect to the true underlying function $\bbf$. For the analysis, we will assume that $\bbf$ is drawn from a Gaussian process, i.e., $\bbf\sim\gp(0, \gpkernel)$, with a positive-semidefinite covariance (or kernel) function $\gpkernel: \obsspace \times \obsspace \to \real$. 
In this case, we can show that the predictive uncertainty of the model converges (in probability) to zero as the number of observations grows. From this result, we prove asymptotic convergence guarantees for VSD equipped with GP-PI-based CPEs. These results form the basis for our analysis of CPEs based on neural networks (\autoref{app:ntk-theory}).

\subsection{Gaussian process posterior}
\label{sec:gp-posterior}
Let $\bbf \sim \gp(0, \gpkernel)$ be a zero-mean Gaussian process with a positive-semidefinite covariance function $\gpkernel: \obsspace \times \obsspace \to \real$. Assume that we are given a set $\data_\nobs := \{(\obs_i, \tar_i)\}_{i=1}^\nobs$ of $\nobs \geq 1$ observations $\tar_i = \bbf(\obs_i) + \obsnoise_i$, where $\obsnoise \sim \normal{0, \sigma_\obsnoise^2}$ and $\obs_i \in \obsspace$. The GP posterior predictive distribution at any $\obs\in\obsspace$ is then given by \citep{rasmussen2006}:
\begin{align}
    \bbf(\obs)|\data_\nobs &\sim \normal{\gpmean_\nobs(\obs), \sigma_\nobs^2(\obs)}\\
    \gpmean_\nobs(\obs) &= \vec{\gpkernel}_\nobs(\obs)^\transpose (\mat{\gpkernel}_\nobs + \sigma_\obsnoise^2\eye)^{-1}\vec{\tar}_\nobs\\
    \gpkernel_\nobs(\obs, \obs') &= \gpkernel(\obs,\obs') - \vec{\gpkernel}_\nobs(\obs)^\transpose (\mat{\gpkernel}_\nobs + \sigma_\obsnoise^2\eye)^{-1}\vec{\gpkernel}_\nobs(\obs')\\
    \sigma_\nobs^2(\obs) &= \gpkernel_\nobs(\obs,\obs)\,,
\end{align}
where $\vec{\gpkernel}_\nobs(\obs) := [\gpkernel(\obs,\obs_i)]_{i=1}^\nobs \in \real^\nobs$, $\mat{\gpkernel}_\nobs := [\gpkernel(\obs_i, \obs_j)]_{i,j=1}^{\nobs,\nobs} \in \real^{\nobs\times\nobs}$, and $\vec{\tar}_\nobs := [\tar_i]_{i=1}^\nobs \in \real^\nobs$.

\paragraph{Batch size.} In the following, we will assume a batch of size $B = 1$ to keep the proofs simple. With this assumption, at every iteration $\iterIdx\geq 1$, we have $\nobs = \iterIdx$ observations available in the dataset. We would, however, like to emphasize that sampling a batch of multiple observations, instead of a single observation, per iteration should only improve the convergence rates by a constant (batch-size-dependent) multiplicative factor. Therefore, our results remain valid as an upper bound for the convergence rates of VSD in the batch setting.

\subsection{Background}
We will consider an underlying probability space $(\SampleSpace, \EventsAlgebra, \pMeasure)$, where $\SampleSpace$ is the sample space, $\EventsAlgebra$ denotes the $\sigma$-algebra of events, and $\pMeasure$ is a probability measure. For any event $\anyevent\in\EventsAlgebra$, we have that $\Probm{\anyevent} \in [0,1]$ quantifies the probability of that event. For events involving a random variable, e.g., $\indrv: (\SampleSpace, \EventsAlgebra) \to (\real, \borel_\real)$, where $\borel_\real$ denotes the Borel $\sigma$-algebra of the real line with its usual topology, we will let:
\begin{equation}
	\Probm{\indrv > 0} = \Probm{\{\sample\in\SampleSpace : \indrv(\sample) > 0\}}\,.
\end{equation}
We will also use conditional expectations, i.e., given a $\sigma$-sub-algebra $\collection{S}$ of $\EventsAlgebra$, the conditional expectation $\expectation[\indrv | \collection{S}]$ is a $\collection{S}$-measurable random variable such that:
\begin{equation}
	\forall \anyevent \in \collection{S}\,, \quad \int_\anyevent \expectation[\indrv | \collection{S}] \diff\pMeasure = \int_\anyevent \indrv \diff\pMeasure = \expectation[\indrv | \anyevent]\,.
\end{equation}
We will denote by $\{\filtration_\iterIdx\}_{\iterIdx=0}^\infty$ an increasing filtration on $\EventsAlgebra$. For instance, we could set $\filtration_\iterIdx$ as the $\sigma$-algebra generated by the random variables in the algorithm (i.e., the candidates, target observations, etc.) at time $\iterIdx$. For more details on the measure-theoretic definition of probability, we refer the reader to classic textbooks in the area \citep[e.g.][]{bauer1981,durrett2019}

We will use the following well known notation for asymptotic convergence results. For a given strictly positive function $\anyfunction:\N\to\real$, we define $\bigo(\anyfunction(\iterIdx))$ as the set of functions asymptotically bounded by $\anyfunction$ (up to a constant factor) as:
\begin{equation}
	\bigo(\anyfunction(\iterIdx)) := \left\lbrace \anotherfunction: \N\to\real \:\middle|\: \limsup_{\iterIdx\to\infty} \frac{|\anotherfunction(\iterIdx)|}{\anyfunction(\iterIdx)} < \infty \right\rbrace ,
	\label{eq:big-o}
\end{equation}
and for convergence in probability we use its stochastic counterpart:
\begin{equation}
	\bigo_\pMeasure(\anyfunction(\iterIdx)) := \left\lbrace \anyrv: \N \times (\SampleSpace, \EventsAlgebra) \to (\real, \borel_\real) \:\middle|\: \lim_{\constant\to\infty}\limsup_{\iterIdx\to\infty} \Probm{\frac{|\anyrv(\iterIdx)|}{\anyfunction(\iterIdx)} > \constant} = 0 \right\rbrace \,,
	\label{eq:big-op}
\end{equation}
which is equivalent to:
\begin{equation}
	\forall \varepsilon > 0, \quad \exists \constant_\varepsilon \in (0, \infty): \quad \Probm{\frac{|\anyrv_\iterIdx|}{\anyfunction(\iterIdx)} > \constant_\varepsilon} \leq \varepsilon, \quad \forall \iterIdx \geq \niter_\varepsilon\,,
\end{equation}
for some $\niter_\varepsilon \in \N$. For almost sure convergence, we may also say that a sequence of random variables $\anyrv_\iterIdx$, $\iterIdx\in\N$, is almost surely $\bigo(\anyfunction(\iterIdx))$ if $\Probm{\anyrv_\iterIdx \in \bigo(\anyfunction(\iterIdx))} = 1$. A deeper overview on these notations and their properties can be found in \citet{garcia-portugues2024}. In addition, as it is common in the multi-armed bandits literature, we use the variants $\widetilde{\bigo}$ and $\widetilde{\bigo}_\pMeasure$ to denote asymptotic rates which are valid up to logarithmic factors.

\subsection{Auxiliary results}
We start with a few technical results which will form the basis for our derivations. The following recursive relations allow us to derive convergence rates for the variance of a GP posterior by analyzing how much it reduces per iteration.
\begin{lemma}[{\citet[Appendix F]{chowdhury2017}}]
	\label{thr:gp-recursion}
	The posterior mean and covariance functions of a Gaussian process given $\iterIdx \geq 1$ observations obey the following recursive identities:
	\begin{align}
		\gpmean_\iterIdx(\obs) &= \gpmean_{\iterIdx-1}(\obs) + \frac{\gpkernel(\obs, \obs_\iterIdx)}{\sigma_\obsnoise^2 + \sigma_{\iterIdx-1}^2(\obs_\iterIdx)}(\tar_\iterIdx - \gpmean_{\iterIdx-1}(\obs))\\
		\gpkernel_\iterIdx(\obs, \obs') &= \gpkernel_{\iterIdx-1}(\obs, \obs') - \frac{\gpkernel_{\iterIdx-1}(\obs, \obs_\iterIdx)\gpkernel_{\iterIdx-1}(\obs_\iterIdx, \obs')}{\sigma_\obsnoise^2 + \sigma_{\iterIdx-1}^2(\obs_\iterIdx)}\\
		\sigma_\iterIdx^2(\obs) &= \sigma_{\iterIdx-1}^2(\obs) - \frac{\gpkernel_{\iterIdx-1}^2(\obs, \obs_\iterIdx)}{\sigma_\obsnoise^2 + \sigma_{\iterIdx-1}^2(\obs_\iterIdx)}\,,
	\end{align}
	for $\obs, \obs'\in\obsspace$.
\end{lemma}

We will also make use of the following version of the second Borel-Cantelli lemma adapted from \citet[Thr. 4.5.5]{durrett2019} and its original statement in \citet{dubins1965}.
\begin{lemma}[Second Borel-Cantelli lemma]
	\label{thr:borel-cantelli}
	Let $\{\anyevent_\iterIdx\}_{\iterIdx=1}^\infty$ be a sequence of events where $\anyevent_\iterIdx \in \filtration_\iterIdx$, for all $\iterIdx\in\N$, and let $\indrv_\iterIdx: \sample \mapsto \indic{\sample\in\anyevent_\iterIdx}$, for $\sample\in\SampleSpace$. Then the following holds with probability 1:
	\begin{equation}
		\lim_{\niter\to\infty} \frac{\sum_{\iterIdx=1}^\niter \indrv_\iterIdx}{\sum_{\iterIdx=1}^\niter \Probm{\anyevent_\iterIdx | \filtration_{\iterIdx-1}}} = \plimit < \infty\,,
	\end{equation}
	assuming $ \Probm{\anyevent_1 | \filtration_{0}} > 0$. In addition, if $\lim_{\niter\to\infty} \sum_{\iterIdx=1}^\niter \Probm{\anyevent_\iterIdx | \filtration_{\iterIdx-1}} = \infty$, then $\plimit = 1$.
\end{lemma}

The next result provides us with an upper bound on the posterior variance of a Gaussian process which is valid for any covariance function.
\begin{lemma}
	\label{thr:gp-variance-upper-bound}
	Let $\gpkernel: \obsspace\times\obsspace \to \real$ be any positive-semidefinite kernel on $\obsspace$, and let $\tilde{\gpkernel}:\obsspace\times\obsspace\to\real$ be a kernel defined as:
	\begin{equation}
		\tilde{\gpkernel}(\obs, \obs') =
		\begin{cases}
			\gpkernel(\obs, \obs), \quad &\obs = \obs'\\
			0, \quad &\obs \neq \obs',
		\end{cases}
	\end{equation}
	for $\obs, \obs' \in \obsspace$. Given any set of observations $\{\obs_i, \tar_i\}_{i=1}^\iterIdx$, for $\iterIdx\geq 1$, denote by $\sigma_\iterIdx^2$ the predictive variance of a GP model with prior covariance given by $\gpkernel$, and let $\tilde{\sigma}_\iterIdx^2$ denote the predictive variance of a GP model configured with $\tilde{\gpkernel}$ as prior covariance function, where both models are given the same set of observations. Then the following holds for all $\iterIdx\geq 0$:
	\begin{equation}
		\sigma_\iterIdx^2(\obs) \leq \tilde{\sigma}_\iterIdx^2(\obs) = \frac{\sigma_\obsnoise^2\tilde{\sigma}_0^2(\obs)}{\sigma_\obsnoise^2 + \sumrv_\iterIdx(\obs)\tilde{\sigma}_0^2(\obs)}, \quad \forall \obs\in\obsspace\,,
	\end{equation}
	where $\sumrv_\iterIdx(\obs)$ denotes the number of observations at $\obs$, and $\tilde{\sigma}_0^2(\obs) = \sigma_0^2(\obs) := \gpkernel(\obs,\obs)$, for $\obs\in\obsspace$.
\end{lemma}
\begin{proof}
	It is not hard to show that $\tilde{\gpkernel}$ defines a valid positive-semidefinite covariance function whenever $\gpkernel$ is positive semidefinite. We will then focus on proving the main statement by an induction argument. The proof that the statement holds for the base case at $\iterIdx=0$ is trivial given the definition:
	\begin{equation}
		\sigma_0^2(\obs) = \gpkernel(\obs, \obs) = \tilde{\gpkernel}(\obs, \obs) = \tilde{\sigma}_0^2(\obs), \quad \forall \obs \in \obsspace\,.
	\end{equation}
	Now assume that, for a given $\iterIdx > 0$, it holds that $\sigma_\iterIdx^2(\obs) \leq \tilde{\sigma}_\iterIdx^2(\obs)$, for all $\obs\in\obsspace$. We will then check if the inequality remains valid at $\iterIdx+1$. By \autoref{thr:gp-recursion}, we have that:
	\begin{equation}
		\begin{split}
			\sigma_{\iterIdx+1}^2(\obs) &= \sigma_\iterIdx^2(\obs) - \frac{\gpkernel_\iterIdx^2(\obs, \obs_{\iterIdx+1})}{\sigma_\iterIdx^2(\obs_{\iterIdx+1}) + \sigma_\obsnoise^2}\\
		\end{split}
	\end{equation}
	For any $\obs\in\obsspace$ such that $\obs \neq \obs_{\iterIdx+1}$, we know that $\tilde{\gpkernel}_\iterIdx(\obs,\obs_{\iterIdx+1}) \geq 0$, so that (again by \autoref{thr:gp-recursion}):
	\begin{equation}
		\tilde{\gpkernel}_{\iterIdx}^2(\obs,\obs_{\iterIdx+1}) \leq \tilde{\gpkernel}^2(\obs,\obs_{\iterIdx+1}) = 0\,,
	\end{equation}
	which shows that:
	\begin{equation}
		\forall \obs \neq \obs_{\iterIdx+1}, \quad \sigma_{\iterIdx+1}^2(\obs) \leq \sigma_\iterIdx^2(\obs) \leq \tilde{\sigma}_\iterIdx^2(\obs) = \tilde{\sigma}_{\iterIdx+1}^2(\obs)\,.
	\end{equation}
	At $\obs = \obs_{\iterIdx+1}$, we can rewrite $\sigma_{\iterIdx+1}^2(\obs) = \sigma_{\iterIdx+1}^2(\obs_{\iterIdx+1})$ as:
	\begin{equation}
		\begin{split}
			\sigma_{\iterIdx+1}^2(\obs_{\iterIdx+1}) = \frac{
				\sigma_\obsnoise^2\sigma_\iterIdx^2(\obs_{\iterIdx+1})
			}{
				\sigma_\iterIdx^2(\obs_{\iterIdx+1}) + \sigma_\obsnoise^2
			}\,.
		\end{split}
	\end{equation}
	We then check the difference:
	\begin{equation}
		\begin{split}
			\sigma_{\iterIdx+1}^2(\obs_{\iterIdx+1}) - \tilde\sigma_{\iterIdx+1}^2(\obs_{\iterIdx+1})
			&=\frac{
				\sigma_\obsnoise^2\sigma_\iterIdx^2(\obs_{\iterIdx+1})
			}{
				\sigma_\iterIdx^2(\obs_{\iterIdx+1}) + \sigma_\obsnoise^2
			}
			-
			\frac{
				\sigma_\obsnoise^2\tilde\sigma_\iterIdx^2(\obs_{\iterIdx+1})
			}{
				\tilde\sigma_\iterIdx^2(\obs_{\iterIdx+1}) + \sigma_\obsnoise^2
			}\\
			&= \frac{
				\sigma_\obsnoise^2\sigma_\iterIdx^2(\obs_{\iterIdx+1}) (\tilde\sigma_\iterIdx^2(\obs_{\iterIdx+1}) + \sigma_\obsnoise^2) - \sigma_\obsnoise^2\tilde\sigma_\iterIdx^2(\obs_{\iterIdx+1}) (\sigma_\iterIdx^2(\obs_{\iterIdx+1}) + \sigma_\obsnoise^2)
			}{(\sigma_\iterIdx^2(\obs_{\iterIdx+1}) + \sigma_\obsnoise^2) (\tilde\sigma_\iterIdx^2(\obs_{\iterIdx+1}) + \sigma_\obsnoise^2)}\\
			&= \frac{
				\sigma_\obsnoise^4(\sigma_\iterIdx^2(\obs_{\iterIdx+1}) - \tilde\sigma_\iterIdx^2(\obs_{\iterIdx+1}))
			}{(\sigma_\iterIdx^2(\obs_{\iterIdx+1}) + \sigma_\obsnoise^2) (\tilde\sigma_\iterIdx^2(\obs_{\iterIdx+1}) + \sigma_\obsnoise^2)}\\
			&\leq 0\,,
		\end{split}
	\end{equation}
	since $\sigma_\iterIdx^2(\obs_{\iterIdx+1}) \leq \tilde\sigma_\iterIdx^2(\obs_{\iterIdx+1})$ by our assumption for time $\iterIdx$.	Therefore, we have shown that:
	\begin{equation}
		\sigma_{\iterIdx}^2(\obs) \leq \tilde\sigma_{\iterIdx}^2(\obs) \implies \sigma_{\iterIdx+1}^2(\obs) \leq \tilde\sigma_{\iterIdx+1}^2(\obs)\,, \quad \forall \obs \in \obsspace\,.
	\end{equation}
	From the conclusion above and the base case, the inequality in the main result follows by induction. 
	
	Now we derive an explicit form for $\tilde{\sigma}_\iterIdx^2$. Note that this case corresponds to an independent Gaussian model, i.e., $\bbf(\obs) \indep \bbf(\obs')$ whenever $\obs\neq\obs'$, for $\bbf\sim\gp(0, \tilde{\gpkernel})$. For any $\iterIdx \geq 1$, this model's predictive variance at any $\obs\in\obsspace$ is given by:
	\begin{equation}
		\tilde\sigma_\iterIdx^2(\obs) =
		\begin{dcases}
			\tilde{\sigma}_{\iterIdx-1}^2(\obs), &\obs \neq \obs_\iterIdx\\
			\frac{
				\sigma_\obsnoise^2\tilde\sigma_{\iterIdx-1}^2(\obs_{\iterIdx})
			}{
				\sigma_\obsnoise^2 + \tilde\sigma_{\iterIdx-1}^2(\obs_{\iterIdx})
			} = \left( \frac{1}{\tilde{\sigma}_{\iterIdx-1}^2(\obs_\iterIdx)} + \frac{1}{\sigma_\obsnoise^2} \right)^{-1}, &\obs = \obs_\iterIdx
		\end{dcases}
	\end{equation}
	Looking at the reciprocal, we have that:
	\begin{equation}
		\forall\iterIdx\geq 1, \quad \frac{1}{\tilde{\sigma}_\iterIdx^2(\obs)} =
			\frac{1}{\tilde\sigma_{\iterIdx - 1}^2(\obs_\iterIdx)} + \frac{\indic{\obs_\iterIdx = \obs}}{\sigma_\obsnoise^2}, \quad \forall \obs\in\obsspace.
	\end{equation}
	Therefore, every observation at $\obs$ is simply adding a factor of $\sigma_\obsnoise^{-2}$ to $\tilde\sigma_\iterIdx^{-2}(\obs)$. Unwrapping this recursion leads us to:
	\begin{equation}
		\forall\iterIdx \geq 1, \quad \frac{1}{\tilde{\sigma}_\iterIdx^2(\obs)}  = \frac{1}{\tilde{\sigma}_0^2(\obs)} + \frac{1}{\sigma_\obsnoise^2} \sum_{i=1}^\iterIdx \indic{\obs_i = \obs}\,, \quad \forall\obs\in\obsspace\,.
	\end{equation}
	The result in \autoref{thr:gp-variance-upper-bound} then follows as the reciprocal of the above, which concludes the proof.
\end{proof}

\begin{lemma}
	\label{thr:gp-sup-bound}
	Let $\bbf \sim \gp(0, \gpkernel)$ for a given $\gpkernel:\obsspace\times\obsspace\to\real$, where $\sigma_\obsspace^2 := \sup_{\obs \in \obsspace} \gpkernel(\obs,\obs) < \infty$, and $\card{\obsspace} < \infty$. Then $\bbf$ is almost surely bounded, and:
	\begin{equation}
		\expectation\!\left[\sup_{\obs\in\obsspace} \lvert \bbf(\obs)\rvert \right] \leq \sigma_\obsspace \sqrt{2\log\card{\obsspace}}\,.
	\end{equation}
\end{lemma}
\begin{proof}
	The result follows by an application of a concentration inequality for the maximum of a finite collection of sub-Gaussian random variables \citep[Sec. 2.5]{boucheron2013}. Note that $\{\bbf(\obs)\}_{\obs\in\obsspace}$ is a collection of $\card{\obsspace}$ Gaussian, and therefore sub-Gaussian, random variables with sub-Gaussian parameter given by $\sigma_\obsspace^2 \geq \sigma_\iterIdx^2(\obs)$, for all $\obsspace$. Applying the maximal inequality for a finite collection sub-Gaussian random variables \citep[Thr. 2.5]{boucheron2013}, we have that:
	\begin{equation}
		\expectation\!\left[\max_{\obs\in\obsspace} \bbf(\obs) \right] \leq \sigma_\obsspace \sqrt{2\log\card{\obsspace}} < \infty\,.
	\end{equation}
	By symmetry, we know that $-\bbf(\obs)$ is also sub-Gaussian with the same parameter, so that the bound remains valid for $\max_{\obs\in\obsspace} -\bbf(\obs)$. As a consequence, the expected value of the maximum of $\rvert\bbf(\obs)\rvert$ is upper bounded by the same constant. On a finite set, the maximum and the supremum coincide. As the expected value of the supremum is finite, the supremum must be almost surely finite by Markov's inequality, and therefore $\bbf$ is almost surely bounded.
\end{proof}

\subsection{Asymptotic convergence}
The main assumption we will be working with in this section is the following.
\begin{assumption}
	\label{a:gp}
	The objective function is a sample from a Gaussian process $\bbf \sim \gp(0, \gpkernel)$, where $\gpkernel:\obsspace\times\obsspace\to\real$ is a bounded positive-semidefinite kernel on $\obsspace$.
\end{assumption}

The next result allows us to derive a convergence rate for the posterior variance of a GP as a function of the sampling probabilities. This result might also be useful by itself for other sampling problems involving GP-based approximations.
\begin{lemma}
    \label{thr:gp-variance-convergence}
    Let $\{\obs_\iterIdx\}_{\iterIdx\geq 1}$ be a sequence of $\obsspace$-valued random variables adapted to the filtration $\{\filtration_\iterIdx\}_{\iterIdx\geq 1}$. For a given $\obs\in\obsspace$, assume that the following holds:
    \begin{equation}
    	\exists \niter_* \in \N: \quad \forall\niter \geq \niter_*, \quad \sum_{\iterIdx=1}^\niter \Probm{\obs_\iterIdx = \obs \mid \filtration_{\iterIdx-1}} \geq \psum_\niter > 0\,,
    \end{equation} 
    for a some sequence of lower bounds $\{\psum_\iterIdx\}_{\iterIdx\in\N}$. Then, under \autoref{a:gp}, given observations  at $\{\obs_i\}_{i=1}^\iterIdx$, the following holds with probability 1:
	\begin{equation}
	    \sigma_\iterIdx^2(\obs) \in \set{O}(\psum_\iterIdx^{-1}).
	\end{equation}
	In addition, if $\psum_\iterIdx \to \infty$, then $\lim_{\iterIdx\to\infty} \psum_\iterIdx\sigma_\iterIdx^2(\obs) \leq \sigma_\obsnoise^2$.
\end{lemma}
\begin{proof}
    At any iteration $\iterIdx$, the posterior variance $\sigma_\iterIdx^2$ of a GP model is upper bounded by a worst case assumption of no correlation between observations (see \autoref{thr:gp-variance-upper-bound}). In this case, we have that:
    \begin{equation}
        \sigma_\iterIdx^2(\obs) \leq \tilde{\sigma}_\iterIdx^2(\obs) = \frac{\sigma_\obsnoise^2\tilde{\sigma}_0^2(\obs)}{\sigma_\obsnoise^2 + \sumrv_\iterIdx\tilde{\sigma}_0^2(\obs)}\,,
        \label{eq:gp-variance-bound}
    \end{equation}
    where $\tilde{\sigma}_0^2(\obs) := \tilde{\gpkernel}(\obs,\obs) = \gpkernel(\obs,\obs)$, and $\sumrv_\iterIdx := \sumrv_\iterIdx(\obs) \leq \iterIdx$ denotes the total number of observations taken at $\obs$ as of iteration $\iterIdx$. Without loss of generality, assume that $\tilde\sigma_0^2(\obs) = 1$.
    
    The only random variable to be bounded in \autoref{eq:gp-variance-bound} is $\sumrv_\iterIdx$. Let $\indrv_\iterIdx := \indic{\obs_\iterIdx = \obs}$, so that:
    \begin{equation}
 	    	\sumrv_\iterIdx = \sum_{i=1}^\iterIdx \indrv_i =  \sum_{i=1}^\iterIdx \indic{\obs_\iterIdx = \obs}\,, \quad \iterIdx \geq 1.
   \end{equation}
    We now apply the second Borel-Cantelli lemma (\autoref{thr:borel-cantelli}) to $\sumrv_\iterIdx$. Namely, let $\widehat\sumrv_\iterIdx$ denote the sum of conditional expectations of $\{\indrv_i\}_{i=1}^\iterIdx$ given available data, i.e.: 
    \begin{equation}
    	\widehat\sumrv_\iterIdx := \sum_{i=1}^\iterIdx \expectation[\indrv_i \mid \filtration_{i-1}] =  \sum_{i=1}^\iterIdx \expectation[\indic{\obs_\iterIdx = \obs} \mid \filtration_{i-1}] = \sum_{i=1}^\iterIdx \Probm{\obs_i = \obs \mid \filtration_{i-1}} \,.
    \end{equation}
    By \autoref{thr:borel-cantelli}, we know that the following holds for some $\plimit \in \real$:
    \begin{equation}
    	\lim_{\iterIdx\to\infty} \frac{\sumrv_\iterIdx}{\widehat{\sumrv}_\iterIdx} = \plimit < \infty\,.
    \end{equation}
    
    Hence, $\sumrv_\iterIdx$ is asymptotically equivalent to $\widehat{\sumrv}_\iterIdx$. Applying this fact to $\tilde{\sigma}_\iterIdx^2$, we have that:
    \begin{equation}
    	\begin{split}
    		\lim_{\iterIdx \to \infty} \psum_\iterIdx \tilde{\sigma}_\iterIdx^2(\obs) &= \lim_{\iterIdx\to\infty} \frac{\psum_\iterIdx\sigma_\obsnoise^2}{\sigma_\obsnoise^2 + \sumrv_\iterIdx}\\
    		&=\lim_{\iterIdx\to\infty} \frac{\psum_\iterIdx\sigma_\obsnoise^2}{\sigma_\obsnoise^2 + \plimit \widehat\sumrv_\iterIdx}\\
    		&\leq \lim_{\iterIdx\to\infty} \frac{\psum_\iterIdx\sigma_\obsnoise^2}{\sigma_\obsnoise^2 + \plimit\psum_\iterIdx}\\
    		&\leq \frac{1}{\plimit}\lim_{\iterIdx\to\infty} \min\{\plimit\psum_\iterIdx, \sigma_\obsnoise^2\}\\
    		&< \infty\,,
    	\end{split}
    \end{equation}
    which holds with probability 1. Lastly, note that, if $\psum_\iterIdx\to\infty$, then $\plimit = 1$ by \autoref{thr:borel-cantelli}, and the last limit above becomes $\sigma_\obsnoise^2$. The main result then follows by an application of \autoref{thr:gp-variance-upper-bound} and the definition of the big-$\bigo$ notation (see \autoref{eq:big-o}).\footnote{Recall that for convergent sequences $\lim$ and $\limsup$ coincide.}
\end{proof}

We assume a finite search space, which is the case for spaces of discrete sequences of bounded length. However, we conjecture that our results can be extended to continuous or mixed discrete-continuous search spaces via a discretization argument under further assumptions on the kernel $\gpkernel$ (e.g., ensuring that $\bbf$ is Lipschitz continuous, as in \citet{srinivas2010gaussian}).
\begin{assumption}
	\label{a:finite-domain}
	The search space $\obsspace$ is finite, $\card{\obsspace} < \infty$.
\end{assumption}

We assume that our family of variational distributions is rich enough to be able to represent the PI-based distribution $\probc{\obs}{\tar > \thresh_\iterIdx, \data_\iterIdx}$, which is the optimum of our variational objective when the optimal classifier is given by GP-PI. Although this assumption could be seen as strong, note that, due to Gaussian noise, the classification probability $\probc{\tar > \thresh_\iterIdx}{\obs, \data_\iterIdx}$ should be a reasonably smooth function of $\obs$, which facilitates the approximation of the resulting posterior by a generative model.
\begin{assumption}
	\label{a:identifiability}
	For every $\iterIdx\geq 0$, $\probc{\obs}{\tar > \thresh_\iterIdx, \data_\iterIdx}$ is a member of the variational family, i.e.:
	\begin{equation}
		\exists\qparam^*_\iterIdx: \quad \diver{\qrobc{\obs}{\qparam_\iterIdx^*}}{\probc{\obs}{\tar > \thresh_\iterIdx, \data_\iterIdx}} = 0.
	\end{equation}
\end{assumption}

The next assumption is a technical one to ensure that the thresholds will not diverge to infinity.
\begin{assumption}
	\label{a:thresholds}
	The sequence of thresholds is almost surely bounded:\footnote{We do not require $\thresh_*$ to be known, only finite.}
	\begin{equation}
		\sup_{\iterIdx \in\N} |\thresh_\iterIdx| \leq \thresh_* < \infty\,.
	\end{equation}
\end{assumption}

We can now state our main result regarding the GP-based approximations learned by VSD.
\begin{theorem}
	\label{thr:vsd-variance-rate}
	Let assumptions \ref{a:gp} to \ref{a:thresholds} hold. Then the following holds with probability 1 for VSD equipped with GP-PI:
	\begin{equation}
		\sigma_\iterIdx^2(\obs) \in \bigo(\iterIdx^{-1})\,,
	\end{equation}
	at every $\obs \in \obsspace$ such that $\pdist(\obs) > 0$.
\end{theorem}
\begin{proof}
Let $\lik_\iterIdx(\obs) := \probc{\tar > \thresh_\iterIdx}{\obs, \data_\iterIdx}$. For any given $\obs\in\obsspace$ where $\prob{\obs} > 0$, by \autoref{a:finite-domain}, we have that the next candidate will be sampled according to:
\begin{equation}
	\begin{split}
		\forall\iterIdx \geq 0, \quad \Probm{\obs_{\iterIdx+1} = \obs \mid \filtration_{\iterIdx} }
		&= \probc{\obs}{\tar > \thresh_\iterIdx, \data_\iterIdx}\\
		&= \frac{
			\lik_\iterIdx(\obs)\prob{\obs}
		}{
			\expec{\prob{\obs}}{\lik_\iterIdx(\obs)}
		}\\
		&\geq \lik_\iterIdx(\obs)\prob{\obs},
	\end{split}
\end{equation}
where we used the fact that $\expec{\prob{\obs}}{\lik_\iterIdx(\obs)} \leq 1$, since $\lik_\iterIdx(\obs) \leq 1$, for all $\obs\in\obsspace$. As $\prob{\obs} > 0$, we only have to derive a lower bound on $\lik_{\iterIdx}(\obs)$ to apply \autoref{thr:gp-variance-convergence} and derive a convergence rate.

A lower bound on $\lik_\iterIdx(\obs)$ is given by:
\begin{equation}
	\begin{split}
		\forall \iterIdx\geq 0, \quad \lik_\iterIdx(\obs) = \cdf\!\left( \frac{\gpmean_\iterIdx(\obs) - \thresh_\iterIdx}{\sqrt{\sigma_\iterIdx^2(\obs) + \sigma_\obsnoise^2}} \right) \geq \cdf\!\left( - \frac{\norm{\gpmean_\iterIdx}_\infty + \thresh_* }{\sigma_\obsnoise} \right),
	\end{split}
\end{equation}
where $\cdf(\cdot)$ denotes the cumulative distribution function of a standard normal random variable, and $\norm{\cdot}_\infty$ denotes the essential supremum of a function under $\pMeasure$ (the probability measure of the underlying abstract probability space). Therefore, if $\lim_{\iterIdx \to \infty} \norm{\gpmean_\iterIdx}_\infty < \infty$, we will have that $\lim_{\iterIdx\to\infty} \lik_\iterIdx(\obs) > 0$, and the sum in \autoref{thr:gp-variance-convergence} will diverge.

By Jensen's inequality for conditional expectations, we have that:
\begin{equation}
	\forall\iterIdx\geq 0, \quad \norm{\gpmean_\iterIdx}_\infty = \norm{\expectation[\bbf \mid \filtration_\iterIdx]}_\infty \leq \expectation[\norm{\bbf}_\infty \mid \filtration_\iterIdx].
\end{equation}
As $\expectation[\expectation[\norm{\bbf}_\infty \mid \filtration_\iterIdx]] = \expectation[\norm{\bbf}_\infty] < \infty$ (cf. \autoref{thr:gp-sup-bound}), an application of Markov's inequality implies that:
\begin{equation}
	\lim_{\anyscalar\to\infty} \Probm{\expectation[\norm{\bbf}_\infty \mid \filtration_\iterIdx] \geq \anyscalar} \leq \lim_{\anyscalar\to\infty} \frac{1}{\anyscalar}\expectation[\norm{\bbf}_\infty] = 0.
\end{equation}
Furthermore, $m_\iterIdx:= \expectation[\norm{\bbf}_\infty \mid \filtration_\iterIdx]$ also defines a non-negative martingale, and by the martingale convergence theorem \citep[Thr. 4.2.11]{durrett2019}, $\lim_{\iterIdx\to\infty} m_\iterIdx = m_\infty := \expectation[\norm{\bbf}_\infty \mid \filtration_\infty]$ is well defined and $\expectation[\expectation[\norm{\bbf}_\infty \mid \filtration_\infty]] = \expectation[\norm{\bbf}_\infty] < \infty$. Again, by Markov's inequality, for any $\anyscalar > 0$, we have that:
\begin{equation}
	\Probm{\lim_{\iterIdx\to\infty}  \norm{\gpmean_\iterIdx}_\infty \geq \anyscalar\expectation[\norm{\bbf}_\infty]} \leq \frac{\expectation\left[\lim_{\iterIdx\to\infty} \norm{\gpmean_\iterIdx}_\infty\right]}{\anyscalar\expectation[\norm{\bbf}_\infty]} \leq \frac{\expectation\left[\lim_{\iterIdx\to\infty} \expectation[\norm{\bbf}_\infty \mid \filtration_\iterIdx] \right]}{\anyscalar\expectation[\norm{\bbf}_\infty]} = \frac{1}{\anyscalar}\,.
\end{equation}
Therefore, for any $\anyscalar > 0$ and any given $\obs\in\obsspace$, with probability at least $1 - \frac{1}{\anyscalar}$, the following holds:
\begin{equation}
	\begin{split}
		\lim_{\iterIdx \to \infty} \Probm{\obs_{\iterIdx} = \obs \mid \filtration_{\iterIdx-1} } &\geq \prob{\obs} \lim_{\iterIdx \to \infty} \lik_{\iterIdx-1}(\obs)\\
		&\geq  \prob{\obs} \lim_{\iterIdx \to \infty} \cdf\!\left( - \frac{\norm{\gpmean_{\iterIdx-1}}_\infty + \thresh_* }{\sigma_\obsnoise} \right)\\
		&\geq \prob{\obs}\cdf\!\left( - \frac{\anyscalar\expectation[\norm{\bbf}_\infty] + \thresh_* }{\sigma_\obsnoise} \right)\\
		&=: \pmin_\infty(\anyscalar) > 0\,.
	\end{split}
\end{equation}
Hence, for any $\varepsilon_\anyscalar \in (0, \pmin_\infty(\anyscalar))$, there is $\nobs_\anyscalar \in \N$, such that $\Probm{\obs_{\iterIdx} = \obs \mid \filtration_{\iterIdx-1}} \geq \pmin_\infty(\anyscalar) - \varepsilon_\anyscalar > 0$, for all $\iterIdx\geq \nobs_\anyscalar$. As a result, $\sum_{\iterIdx'=1}^\iterIdx \Probm{\obs_{\iterIdx'} = \obs \mid \filtration_{\iterIdx'-1}} \geq  (\pmin_\infty(\anyscalar) - \varepsilon_\anyscalar) (\iterIdx - \nobs_\anyscalar)$, for all $\iterIdx \geq \nobs_\anyscalar$, which asymptotically diverges at a rate proportional to $\iterIdx$. By \autoref{thr:gp-variance-convergence} and the definition of the big-$\bigo$ notation, for any $\obs\in\obsspace$, we then have that:
\begin{equation}
	\forall \anyscalar > 0, \quad \Probm{\limsup_{\iterIdx\to\infty} \left\lvert \iterIdx \sigma_\iterIdx^2(\obs)\right\rvert \leq \sigma_\obsnoise^2 <  \infty } \geq 1 - \frac{1}{\anyscalar}\,.
\end{equation}
Taking the limit as $\anyscalar\to\infty$, we can finally conclude that:
\begin{equation}
	\Probm{\limsup_{\iterIdx\to\infty} \left\lvert \iterIdx \sigma_\iterIdx^2(\obs)\right\rvert < \infty} = 1,
\end{equation}
i.e., $\sigma_\iterIdx^2$ is almost surely $\bigo(\iterIdx^{-1})$, which concludes the proof.
\end{proof}

\begin{rremark}
	The convergence rate in \autoref{thr:vsd-variance-rate} is optimal and cannot be further improved. As shown by previous works in the online learning literature \citep{mutny2018, takeno2023}, a lower bound on the GP variance at each iteration $\iterIdx \geq 1$ is given by $\sigma_\iterIdx^2(\obs) \geq \sigma_\obsnoise^2(\sigma_\obsnoise^2 + \iterIdx)^{-1}$ (assuming $\gpkernel(\obs, \obs) = 1$), which is the case when every observation in the dataset was collected at the same point $\obs\in\obsspace$ \citep[see][Lem. 4.2]{takeno2023}. Therefore, the lower and upper bounds on the asymptotic convergence rates for the GP variance differ by only up to a multiplicative constant.
\end{rremark}

The result in \autoref{thr:vsd-variance-rate} now allows us to derive a convergence rate for VSD's approximations to the level-set distributions. To do so, however, we will require the following mild assumption, which is satisfied by any prior distribution which has support  on the entire domain $\obsspace$.

\begin{assumption}
	\label{a:prior}
	The prior distribution is such that $\prob{\obs} > 0$, for all $\obs\in\obsspace$.
\end{assumption}

\mainthm*
\begin{proof}
	We first prove an upper bound for the KL divergence in terms of the PI approximation error. We then derive a bound for this term and apply \autoref{thr:vsd-variance-rate} to obtain a convergence rate.
	
	\emph{KL bound formulation.} Let $\lik_\iterIdx(\obs) := \probc{\tar > \thresh_\iterIdx}{\obs, \data_\iterIdx}$ and $\lik_\iterIdx^*(\obs) := \probc{\tar > \thresh_\iterIdx}{\obs, \bbf}$, for $\obs\in\obsspace$. From the definition of the KL divergence, we have that:
	\begin{equation}
		\begin{split}
			\diver{\probc{\obs}{\tar > \thresh_\iterIdx, \data_\iterIdx}}{\probc{\obs}{\tar > \thresh_\iterIdx, \bbf}}
			&= \expec{\probc{\obs}{\tar > \thresh_\iterIdx, \data_\iterIdx}}{\log \probc{\obs}{\tar > \thresh_\iterIdx, \data_\iterIdx} - \log \probc{\obs}{\tar > \thresh_\iterIdx, \bbf}}\\
			&=\expec{\probc{\obs}{\tar > \thresh_\iterIdx, \data_\iterIdx}}{\log \lik_\iterIdx(\obs) - \log \lik_\iterIdx^*(\obs)}\\
			&\quad + \log \expec{\prob{\obs}}{\lik_\iterIdx^*(\obs)} - \log \expec{\prob{\obs}}{\lik_\iterIdx(\obs)}\\
			&= \expec{\probc{\obs}{\tar > \thresh_\iterIdx, \data_\iterIdx}}{\log \left( \frac{\lik_\iterIdx(\obs)}{\lik_\iterIdx^*(\obs)} \right) } + \log \left( \frac{\expec{\prob{\obs}}{\lik_\iterIdx^*(\obs)}}{\expec{\prob{\obs}}{\lik_\iterIdx(\obs)}} \right).
		\end{split}
		\label{eq:kl-components}
	\end{equation}
	For logarithms, we know that $\log(1+\anyscalar) \leq \anyscalar$, for all $\anyscalar > -1$, which shows that:
	\begin{align}
		\log \left( \frac{\lik_\iterIdx(\obs)}{\lik_\iterIdx^*(\obs)} \right) &= \log \left( 1 +  \frac{\lik_\iterIdx(\obs) - \lik_\iterIdx^*(\obs)}{\lik_\iterIdx^*(\obs)} \right) \leq \frac{\lik_\iterIdx(\obs) - \lik_\iterIdx^*(\obs)}{\lik_\iterIdx^*(\obs)}\\
		\log \left( \frac{\expec{\prob{\obs}}{\lik_\iterIdx^*(\obs)}}{\expec{\prob{\obs}}{\lik_\iterIdx(\obs)}} \right) &= \log\left(1 + \frac{\expec{\prob{\obs}}{\lik_\iterIdx^*(\obs) - \lik_\iterIdx(\obs)}}{\expec{\prob{\obs}}{\lik_\iterIdx(\obs)}}\right) \leq \frac{\expec{\prob{\obs}}{\lik_\iterIdx^*(\obs) - \lik_\iterIdx(\obs)}}{\expec{\prob{\obs}}{\lik_\iterIdx(\obs)}}\,.
	\end{align}
	Combining the above into \autoref{eq:kl-components} yields:
	\begin{equation}
		\begin{split}
			\diver{\probc{\obs}{\tar > \thresh_\iterIdx, \data_\iterIdx}}{\probc{\obs}{\tar > \thresh_\iterIdx, \bbf}}
			&\leq  \expec{\probc{\obs}{\tar > \thresh_\iterIdx, \data_\iterIdx}}{ \frac{\lik_\iterIdx(\obs) - \lik_\iterIdx^*(\obs)}{\lik_\iterIdx^*(\obs)} } + \frac{ \expec{\prob{\obs}}{\lik_\iterIdx^*(\obs) - \lik_\iterIdx(\obs)} }{ \expec{\prob{\obs}}{\lik_\iterIdx(\obs)} }
		\end{split}\,.
		\label{eq:kl-bound}
	\end{equation}
	The denominator in the expression above is such that:
	\begin{equation}
		\forall\iterIdx\geq 0, \quad \lik_\iterIdx^*(\obs) =  \probc{\tar > \thresh_\iterIdx}{\obs, \bbf} = \cdf\!\left( \frac{\bbf(\obs) - \thresh_\iterIdx}{\sigma_\obsnoise} \right) \geq \cdf\!\left( - \frac{\norm{\bbf}_\infty + \thresh_*}{\sigma_\obsnoise} \right), \quad \forall \obs\in\obsspace\,.
        \label{eq:true-lik-lb}
	\end{equation}
	By \autoref{thr:gp-sup-bound}, we know that $\expectation[\norm{\bbf}_\infty] < \infty$, which implies that $\Probm{\norm{\bbf}_\infty < \infty} = 1$ by Markov's inequality. Next, we derive a bound for the approximation error term.
	
	\emph{Error bound.} We now derive an upper bound for the difference $\ldiff_\iterIdx(\obs) :=  \lik_\iterIdx(\obs) - \lik_\iterIdx^*(\obs)$ and then show that it asymptotically vanishes. 
	Applying Taylor's theorem to $\cdf$, we can bound $\ldiff_\iterIdx$ as a function of the approximation error between the mean $\gpmean_\iterIdx$ and the true function $\bbf$ as:
	\begin{equation}
		\begin{split}
			\forall\iterIdx \geq 0, \quad |\ldiff_\iterIdx(\obs)| &= \left\lvert  \cdf\!\left( \frac{\gpmean_\iterIdx(\obs) - \thresh_\iterIdx}{\sqrt{\sigma_\iterIdx^2(\obs) + \sigma_\obsnoise^2}} \right) - \cdf\!\left( \frac{\bbf(\obs) - \thresh_\iterIdx}{\sigma_\obsnoise} \right) \right\rvert\\
			&\leq\frac{1}{\sqrt{2\pi}} \left\lvert  \frac{\gpmean_\iterIdx(\obs) - \thresh_\iterIdx}{\sqrt{\sigma_\iterIdx^2(\obs) + \sigma_\obsnoise^2}} - \frac{\bbf(\obs) - \thresh_\iterIdx}{\sigma_\obsnoise} \right\rvert\\
			&=\frac{1}{\sqrt{2\pi}} \left\lvert 
				\frac{
					\sigma_\obsnoise \gpmean_\iterIdx(\obs) - \bbf(\obs)\sqrt{\sigma_\iterIdx^2(\obs) + \sigma_\obsnoise^2} + \thresh_\iterIdx(\sqrt{\sigma_\iterIdx^2(\obs) + \sigma_\obsnoise^2} - \sigma_\obsnoise)
				}{
					\sigma_\obsnoise \sqrt{\sigma_\iterIdx^2(\obs) + \sigma_\obsnoise^2}
				}
			\right\rvert\\
			&\leq
			\frac{
				\lvert \sigma_\obsnoise \gpmean_\iterIdx(\obs) - \bbf(\obs)\sqrt{\sigma_\iterIdx^2(\obs) + \sigma_\obsnoise^2} \rvert + |\thresh_\iterIdx| \sigma_\iterIdx(\obs)
			}{
				\sigma_\obsnoise^2\sqrt{2\pi}
			}\\
			&\leq
			\frac{
				 \sigma_\obsnoise \lvert \gpmean_\iterIdx(\obs) - \bbf(\obs) \rvert + \sigma_\iterIdx(\obs)(|\bbf(\obs)| + |\thresh_\iterIdx|)
			}{
				\sigma_\obsnoise^2\sqrt{2\pi}
			}
			,\quad \forall \obs\in\obsspace,\\
		\end{split}
	\end{equation}
	since $\sup_{\obsnoise\in\real} \left\lvert\frac{\diff\cdf(\obsnoise)}{\diff\obsnoise}\right\rvert = \frac{1}{\sqrt{2\pi}} < 1$, and we used the fact that $\sigma_\obsnoise \leq \sqrt{\sigma_\iterIdx^2(\obs) + \sigma_\obsnoise^2} \leq \sigma_\iterIdx(\obs) + \sigma_\obsnoise$ to obtain the last two inequalities. 
	
	\emph{Convergence rate.} To derive a convergence rate, given any $\obs \in \obsspace$ and $\iterIdx\geq 0$, we have that:
	\begin{equation}
		\expectation[|\ldiff_\iterIdx(\obs)| \mid \filtration_\iterIdx] \leq \frac{
			\sigma_\obsnoise \expectation[\lvert \gpmean_\iterIdx(\obs) - \bbf(\obs) \rvert \mid \filtration_\iterIdx] + \sigma_\iterIdx(\obs)(\expectation[|\bbf(\obs)| \mid \filtration_\iterIdx] + |\thresh_\iterIdx|)
		}{
			\sigma_\obsnoise^2\sqrt{2\pi}
		}.
	\end{equation}
	We know that $\expectation[|\bbf(\obs)| \mid \filtration_\iterIdx]$ is almost surely bounded, and by Jensen's inequality, it also holds that:
	\begin{equation}
		 \expectation[\lvert \gpmean_\iterIdx(\obs) - \bbf(\obs) \rvert \mid \filtration_\iterIdx] \leq \sigma_\iterIdx(\obs).
	\end{equation}
	Applying \autoref{thr:vsd-variance-rate}, we then have that:
	\begin{equation}
		|\ldiff_\iterIdx(\obs)| \in \bigo_\pMeasure(\iterIdx^{-1/2}).
	\end{equation}
	Since $\norm{\gpmean_\iterIdx}_\infty \leq \expectation[\norm{\bbf}_\infty |\filtration_\iterIdx] \in \bigo_\pMeasure(1)$, we also have that:
	\begin{equation}
		\frac{1}{\expec{\prob{\obs}}{\lik_\iterIdx(\obs)}} \in \bigo_\pMeasure(1)\,.
	\end{equation}
	Lastly, we know that $\frac{1}{\lik_\iterIdx^*(\obs)} \in \bigo_\pMeasure(1)$ by \autoref{eq:true-lik-lb} and the observation that $\norm{\bbf}_\infty \in \bigo_\pMeasure(1)$. The main result then follows by combining the rates above into \autoref{eq:kl-bound}.
\end{proof}

\subsection{Performance analysis} 
At every iteration $\iterIdx\geq 1$, VSD samples $\obs_\iterIdx$ from (an approximation to) the target $\probc{\obs}{\tar > \thresh_{\iterIdx-1}, \data_{\iterIdx-1}}$ and obtains an observation $\tar_\iterIdx \sim \probc{\tar}{\obs_\iterIdx}$. A positive hit consists of an event $\tar_\iterIdx > \thresh_{\iterIdx-1}$, where $\thresh_{\iterIdx-1}$ is computed based on the data available in $\data_{\iterIdx-1}$ or a constant. Therefore, we can compute the probability of a positive hit for a given realization of $\bbf$ as:
\begin{equation}
	\Probm{\tar_\iterIdx > \thresh_{\iterIdx-1} \mid \data_{\iterIdx-1}, \bbf} = \expec{ \probc{\obs}{\tar > \thresh_{\iterIdx-1}, \data_{\iterIdx-1}} }{\probc{\tar > \thresh_{\iterIdx-1}}{\obs, \bbf}}.
\end{equation}
Then the expected number of hits $\nhits_\niter$ after $\niter \geq 1$ iterations is given by:
\begin{equation}
	\expectation[\nhits_\niter \mid \bbf] = \sum_{\iterIdx=1}^\niter \expec{ \probc{\obs}{\tar > \thresh_{\iterIdx-1}, \data_{\iterIdx-1}} }{\probc{\tar > \thresh_{\iterIdx-1}}{\obs, \bbf}}.
\end{equation}
We will compare this quantity with the expected number of hits $\nhits_\niter^*$ obtained by a sampling distribution with full knowledge of the objective function $\bbf$:
\begin{equation}
	\expectation[\nhits_\niter^* \mid \bbf] = \sum_{\iterIdx=1}^\niter \expec{\probc{\obs}{\tar > \thresh_{\iterIdx-1}, \bbf}}{ \probc{\tar_\iterIdx > \thresh_{\iterIdx-1}}{\obs, \bbf} }.
\end{equation}
The next result allows us to bound the difference between these two quantities.

\hitscor*
\begin{proof}
	For all $\niter\geq 1$, we have that:
	\begin{equation}
		\begin{split}
			\expectation[ \nhits_\niter - \nhits_\niter^* ]
			&= \expectation\!\left[ 
				\sum_{\iterIdx=1}^\niter 
					\expec{ \probc{\obs}{ \tar > \thresh_{\iterIdx-1}, \data_{\iterIdx-1}} }{ \probc{\tar > \thresh_{\iterIdx-1}}{\obs, \bbf} } 
					- \expec{ \probc{\obs}{ \tar > \thresh_{\iterIdx-1}, \bbf} }{ \probc{\tar_\iterIdx > \thresh_{\iterIdx-1}}{\obs, \bbf} } 
				\right]\\
			&= \expectation\!\left[ 
					\sum_{\iterIdx=1}^\niter  
					\sum_{\obs \in \obsspace} 
					\probc{\tar > \thresh_{\iterIdx-1}}{\obs, \bbf}
					\left(
						 \probc{\obs}{ \tar > \thresh_{\iterIdx-1}, \data_{\iterIdx-1}} - \probc{\obs}{ \tar > \thresh_{\iterIdx-1}, \bbf} 
					\right)
				\right]\\
			&= \expectation\!\left[ 
				\sum_{\iterIdx=0}^{\niter  - 1}
				\sum_{\obs \in \obsspace} 
					\probc{\tar > \thresh_{\iterIdx}}{\obs, \bbf}\prob{\obs}
					\left(
						\frac{\lik_\iterIdx(\obs)}{\expec{\prob{\obs'}}{\lik_\iterIdx(\obs')}} - \frac{\lik_\iterIdx^*(\obs)}{\expec{\prob{\obs'}}{\lik_\iterIdx^*(\obs')}} 
					\right)
				\right]\\
			&\leq \expectation\!\left[ 
				\sum_{\iterIdx=0}^{\niter - 1}
				\sum_{\obs \in \obsspace} 
					\prob{\obs}
					\left(
						\frac{|\ldiff_\iterIdx(\obs)|}{\min\{ \expec{\prob{\obs'}}{\lik_\iterIdx(\obs')}, \expec{\prob{\obs'}}{\lik_\iterIdx^*(\obs')} \} }
					\right)
				\right],
		\end{split}
	\end{equation}
	since $\probc{\tar > \thresh_{\iterIdx-1}}{\obs, \bbf}\leq 1$, for all $\iterIdx\geq 1$. As both $\norm{\gpmean_\iterIdx}_\infty$ and $\norm{\bbf}_\infty$ are in $\bigo_\pMeasure(1)$, $\min\{ \expec{\prob{\obs'}}{\lik_\iterIdx(\obs')}, \expec{\prob{\obs'}}{\lik_\iterIdx^*(\obs')} \}$ is lower bounded by some constant. As $\ldiff_\iterIdx(\obs) \in \bigo_\pMeasure(\iterIdx^{-1/2})$, for $\niter$ large enough and some $\constant > 0$, we then have that:
	\begin{equation}
		\begin{split}
			\expectation[ |\nhits_\niter - \nhits_\niter^*| ]
			\leq \constant\sum_{\iterIdx=1}^\niter \frac{1}{\sqrt{\iterIdx}}
			\leq 2\constant \sqrt{\niter}
			\in \bigo(\sqrt{\niter}),
		\end{split}
	\end{equation}
	which follows by an application of the Euler-Maclaurin formula, since $\int_{1}^{\niter} \frac{1}{\sqrt{\iterIdx}} \diff\iterIdx = 2\sqrt{\niter} - 2$ and the remainder term asymptotically vanishes.
\end{proof}

\begin{rremark}
    If the oracle achieves $\expectation[\nhits_\niter^*] = \niter$, the error bound in \autoref{thm:hits} suggests an increasing rate of positive hits by \gls{vsd} as $\frac{1}{\niter}\expectation[\nhits_\niter] \geq 1 - \constant\niter^{-1/2}$, for some constant $\constant >0$ and large enough $\niter$. Therefore, \gls{vsd} should asymptotically achieve a full rate of 1 positive hit per iteration in the single-point batch setting we consider. Note, however, that the results above do not discount for repeated samples, though should still indicate that \gls{vsd} achieves a high discovery rate over the course of its execution.
\end{rremark}


\section{VSD with neural network CPEs}
\label{app:ntk-theory}
In this section, we consider \gls{vsd} with class probability estimators that are not based on \gls{gp} regression, which was the case for the previous section, while specifically focusing on neural network models. We will, however, show that with a kernel-based formulation we are able to capture the classification models based on neural networks which we use. This is possible by analyzing the behavior of infinite-width neural networks \citep{jacot2018, lee2019}, whose approximation error with respect to the finite-width model can be bounded \citep{liu2020c, eldan2021}.

Although our classifiers are learned by minimizing the cross-entropy (CE) loss, we can connect their approximations with theoretical results from the infinite-width neural network (NN) literature, which are mostly based on the mean squared error (MSE) loss. Recall that, given a dataset $\data_\nobs^\labl := \{(\obs_\obsIdx, \labl_\obsIdx)\}_{\obsIdx=1}^\nobs$ with binary labels $\labl_\obsIdx\in\{0,1\}$, the cross-entropy loss for a probabilistic classifier $\classifier_\mparam: \obsspace\to [0,1]$ parameterized by $\mparam$ is given by\footnote{We implicitly assume that $0<\cpe{\mparam}{\obs_\obsIdx}<1$, for $\obsIdx \in \{1, \dots, \nobs\}$, so that the CE loss is well defined. This assumption can, however, be relaxed when dealing with the MSE loss, which remains well defined otherwise.}:
\begin{align}
	\lcpe{\mparam, \data_\nobs^\labl} := - \frac{1}{\nobs} \sum_{\obsIdx=1}^\nobs
	\labl_\obsIdx \log \cpe{\mparam}{\obs_\obsIdx}
	+ (1-\labl_\obsIdx) \log (1 - \cpe{\mparam}{\obs_\obsIdx})\,.
\end{align}
The MSE loss for the same model corresponds to:
\begin{equation}
	\lmse{\mparam, \data_\nobs^\labl} := \frac{1}{\nobs}  \sum_{\obsIdx=1}^\nobs (\labl_\obsIdx - \cpe{\mparam}{\obs_\obsIdx})^2\,.
\end{equation}
The following result establishes a connection between the two loss functions.

\begin{proposition}
	\label{thr:ce-mse-bound}
	Given a binary classification dataset $\data_\nobs^\labl$ of size $\nobs\geq 1$, the following holds for the cross-entropy and the mean-square error losses:
	\begin{equation}
		\lcpe{\mparam, \data_\nobs^\labl} \geq \lmse{\mparam, \data_\nobs^\labl}, \quad \forall \nobs \in \N\,.
	\end{equation}
\end{proposition}
\begin{proof}
	Applying the basic logarithmic inequality $\log(1 + \anyscalar) \leq \anyscalar$, for all $\anyscalar > -1$, to the cross-entropy loss definition yields:
	\begin{equation}
		\begin{split}
			\lcpe{\mparam, \data_\nobs^\labl}
			&:= - \frac{1}{\nobs} \sum_{\obsIdx=1}^\nobs
			\labl_\obsIdx \log \cpe{\mparam}{\obs_\obsIdx}
			+ (1-\labl_\obsIdx) \log (1 - \cpe{\mparam}{\obs_\obsIdx})\\
			&\geq - \frac{1}{\nobs} \sum_{\obsIdx=1}^\nobs
			\labl_\obsIdx (\cpe{\mparam}{\obs_\obsIdx} - 1)
			- (1-\labl_\obsIdx) \cpe{\mparam}{\obs_\obsIdx}\\
			&= - \frac{1}{\nobs} \sum_{\obsIdx=1}^\nobs 2\labl_\obsIdx\cpe{\mparam}{\obs_\obsIdx} - \labl_\obsIdx -  \cpe{\mparam}{\obs_\obsIdx}\\
			&= \frac{1}{\nobs} \sum_{\obsIdx=1}^\nobs \labl_\obsIdx - 2\labl_\obsIdx\cpe{\mparam}{\obs_\obsIdx} + \cpe{\mparam}{\obs_\obsIdx}\,.
		\end{split}
		\label{eq:ce-loss-lb}
	\end{equation}
	Now note that $\labl_\obsIdx = \labl_\obsIdx^2$, for $\labl_\obsIdx\in\{0, 1\}$, and $\cpe{\mparam}{\obs_\obsIdx} \geq \cpe{\mparam}{\obs_\obsIdx}^2$, as $\cpe{\mparam}{\obs_\obsIdx} \in [0, 1]$, for all $\obsIdx \in \{1, \dots, \nobs\}$. Making these substitutions in \autoref{eq:ce-loss-lb}, we obtain:
	\begin{equation}
		\lcpe{\mparam, \data_\nobs^\labl} \geq \frac{1}{\nobs} \sum_{\obsIdx=1}^\nobs \labl_\obsIdx^2 - 2\labl_\obsIdx\cpe{\mparam}{\obs_\obsIdx} + \cpe{\mparam}{\obs_\obsIdx}^2 = \lmse{\mparam, \data_\nobs^\labl}\,,
	\end{equation}
	which concludes the proof.
\end{proof}

The result in \autoref{thr:ce-mse-bound} suggests that minimizing the cross-entropy loss will lead us to minimize the MSE loss as well, since the latter is upper bounded by the former. This result provides us with theoretical justification to derive convergence results based on the MSE loss, which has been better analyzed in the NN literature \citep{jacot2018, lee2019}, as a proxy to establish convergence guarantees for the CE-based VSD setting.

\subsection{Linear approximations via the neural tangent kernel}
For this analysis, we will follow a frequentist setting. Namely, let $\classifier^*$ denote the unknown true classifier, i.e., $\classifier(\obs) := \probc{\tar > \thresh}{\obs, \bbf}$, for $\obs\in\obsspace$. We assume that $\classifier^*$ is an unknown, fixed element of a reproducing kernel Hilbert space (RKHS) associated with a given kernel \citep{scholkopf2001}. In the case of infinite-width neural networks, we know that under certain assumptions the NN trained via gradient descent under the MSE loss will asymptotically converge to a kernel ridge regression solution whose kernel is given by the neural tangent kernel (NTK, \citeauthor{jacot2018}, \citeyear{jacot2018}). This asymptotic solution is equivalent to the posterior mean of a Gaussian process that assumes no observation noise. However, for a finite number of training steps $\tstop<\infty$, the literature has shown that gradient-based training provides a form of implicit regularization, which we use to ensure robustness to label noise. 
Moreover, although our analysis will be based on the NTK, the approximation error between the infinite-width and the finite-width NN vanishes with the square root of the network width for most popular NN architectures \citep{liu2020c}. Therefore, we can assume that these approximation guarantees will remain useful for wide-enough, finite-width NN models.

\paragraph{Reproducing kernel Hilbert spaces.} The RKHS $\fspace_\gpkernel$ associated with a positive-semidefinite kernel $\gpkernel:\obsspace\times\obsspace\to\real$ is a Hilbert space of functions over $\obsspace$ with an inner product $\inner{\cdot,\cdot}_\gpkernel$ and corresponding norm $\norm{\cdot}_\gpkernel := \sqrt{\inner{\cdot,\cdot}}_\gpkernel$ such that, for every $\classifier\in\fspace_\gpkernel$, the reproducing property $\classifier(\obs) = \inner{\classifier,\gpkernel(\cdot,\obs)}_\gpkernel$ holds for all $\obs\in\obsspace$ \citep{scholkopf2001}. 

\paragraph{Implicit regularization.} Several results in the literature have shown that training overparameterized neural networks via gradient descent provides a form of implicit regularization on the learned model \citep{fleming1990, yao2007, soudry2018, barrett2021}, with some of the same behavior extending to the stochastic gradient setting \citep{smith2021}. In earlier works, \citet{fleming1990} showed a direct equivalence between an early stopped gradient-descent linear model and the solution of a regularized least-squares problem with a penalty on the parameters vector Euclidean norm. 
In the NTK regime, the network output predictions at iteration $\tstop \in \N$ of gradient descent are given by \citep{lee2019}:
\begin{equation}
    \hat\classifier_{\nobs}(\obs) = \classifier_0(\obs) + \vec\gpkernel_\nobs(\obs)^\transpose\mat\gpkernel_\nobs^{-1}(\eye - e^{-\lrate\tstop\mat\gpkernel_\nobs})(\vec\labl_\nobs - \classifier_0(\obsspace_\nobs))\,, \quad \obs\in\obsspace\,,
\end{equation}
where $\classifier_0$ represents the network's initialization, $\lrate > 0$ denotes the learning rate, the kernel $\gpkernel$ corresponds to the NTK associated with the given architecture, and the data is represented by $\obsspace_\nobs := \{\obs_i\}_{i=1}^\nobs \subset \obsspace$ and $\vec\labl_\nobs := [\labl_i]_{i=1}^\nobs\in\{0,1\}^\nobs$. Rearranging terms and performing basic algebraic manipulations, the equation above can be shown to be equivalent to:
\begin{equation}
    \hat\classifier_{\nobs}(\obs) = \classifier_0(\obs) + \vec\gpkernel_\nobs(\obs)^\transpose(\mat\gpkernel_\nobs + \regmat_{\nobs})^{-1}(\vec\labl_\nobs - \classifier_0(\obsspace_\nobs))\,, \quad \obs\in\obsspace\,,
    \label{eq:es-solution}
\end{equation}
where $\regmat_{\nobs} := \mat\gpkernel_\nobs(e^{\lrate\tstop\mat\gpkernel_\nobs} - \eye)^{-1}$ corresponds to a data-dependent regularization matrix. The above is equivalent to the solution of the a regularized least-squares problem, as we show below.

\begin{lemma}
\label{thr:es-implicit-reg}
Assume $\gpkernel:\obsspace\times\obsspace\to\real$ is positive definite and $\classifier_0 = 0$. Then \autoref{eq:es-solution} solves the following regularized least-squares problem:
\begin{equation}
    \hat\classifier_{\nobs} \in \argmin_{\classifier\in\fspace_\gpkernel} \sum_{i=1}^\nobs (\classifier(\obs_i) - \labl_i)^2 + \norm{\regop_{\nobs}^{1/2}\classifier}_\gpkernel^2\,,
\end{equation}
where $\regop_{\nobs} := \features_\nobs(e^{-\lrate\tstop\mat\gpkernel_\nobs}-\eye)^{-1}\features_\nobs^\transpose$, $\features_\nobs := [\feature(\obs_1), \dots, \feature(\obs_\nobs)]$, and $\feature(\obs) := \gpkernel(\cdot,\obs) \in \fspace_\gpkernel$ represents the kernel's canonical feature map, for $\obs\in\obsspace$.
\end{lemma}
\begin{proof}
    The least-squares loss can be rewritten as:
    \begin{equation}
        \begin{split}
            \ell_{\nobs}(\classifier) &:= \sum_{i=1}^\nobs (\classifier(\obs_i) - \labl_i)^2 + \norm{\classifier}_{\regop_{\nobs}}^2
            = \norm{\features_\nobs^\transpose\classifier - \vec\labl_\nobs}_2^2 + \norm{\regop_{\nobs}^{1/2}\classifier}_\gpkernel^2\,,
        \end{split}
    \end{equation}
    where $\norm{\cdot}_2$ denotes the Euclidean norm of a vector.
    Taking the functional gradient with respect to $\classifier\in\fspace_\gpkernel$ and equating it to zero, we have that an optimal solution $\hat\classifier$ satisfies:
    \begin{equation}
        \begin{split}
            \nabla \ell_{\nobs}(\hat\classifier) &= 2 \features_\nobs (\features_\nobs^\transpose\hat\classifier - \vec\labl_\nobs) + 2 \regop_{\nobs}\hat\classifier = 0\,.
        \end{split}
    \end{equation}
    The minimum-norm solution is then given by:
    \begin{equation}
        \hat\classifier = (\features_\nobs \features_\nobs^\transpose + \regop_{\nobs})^+ \features_\nobs \vec\labl_\nobs\,,
        \label{eq:rkhs-solution-pinv}
    \end{equation}
    where $\mat{A}^+$ denotes the Moore-Penrose pseudo-inverse of an operator $\mat{A}:\fspace_\gpkernel\to\fspace_\gpkernel$.
    
    Let $\features_\nobs = \lsv_\nobs \svs_\nobs\rsv_\nobs^\transpose$ represent the singular value decomposition of $\features_\nobs$ in the RKHS \citep{mollenhauer2020}, where $\lsv_\nobs := [\lsvec_1, \dots, \lsvec_\nobs]$, $\rsv_\nobs := [\rsvec_1,\dots, \rsvec_\nobs]$, with $\{\lsvec_i\}_{i=1}^\nobs\subset\fspace_\gpkernel$ and $\{\rsvec_i\}_{i=1}^\nobs\subset\real^\nobs$ denoting the left and right singular vectors, respectively, and $\svs_\nobs \in \real^{\nobs\times\nobs}$ corresponds to the diagonal matrix of singular values of $\features_\nobs$. There are $\nobs$ non-zero singular values, since $\gpkernel$ is assumed to be positive definite, and we have the correspondence $\mat\gpkernel_\nobs = \features_\nobs^\transpose\features_\nobs = \rsv_\nobs \evs_\nobs\rsv_\nobs^\transpose$ with $\evs_\nobs = \svs_\nobs^2$ representing the diagonal matrix of eigenvalues of $\mat\gpkernel_\nobs$, which is full-rank for a positive-definite kernel with distinct entries $\obsspace_\nobs := \{\obs_i\}_{i=1}^\nobs \subset \obsspace$. Applying the SVD to derive the pseudo-inverse in \autoref{eq:rkhs-solution-pinv} then yields:
    \begin{equation}
        \begin{split}
            \hat\classifier &= (\features_\nobs \features_\nobs^\transpose + \features_\nobs(e^{-\lrate\tstop\mat\gpkernel_\nobs}-\eye)^{-1}\features_\nobs^\transpose)^+ \features_\nobs \vec\labl_\nobs\\
            &= (\lsv_\nobs \evs_\nobs \lsv_\nobs^\transpose + \lsv_\nobs \svs_\nobs\rsv_\nobs^\transpose(e^{-\lrate\tstop \rsv_\nobs \evs_\nobs\rsv_\nobs^\transpose}-\eye)^{-1}\rsv_\nobs\svs_\nobs\lsv_\nobs^\transpose)^+ \lsv_\nobs \svs_\nobs\rsv_\nobs^\transpose \vec\labl_\nobs\\
            &= \lsv_\nobs( \evs_\nobs  + \evs_\nobs(e^{-\lrate\tstop\evs_\nobs}-\eye)^{-1})^{-1}\svs_\nobs \rsv_\nobs^\transpose \vec\labl_\nobs\\
            &= \lsv_\nobs\svs_\nobs\rsv_\nobs^\transpose\rsv_\nobs (\evs_\nobs  + \evs_\nobs(e^{-\lrate\tstop\evs_\nobs}-\eye)^{-1})^{-1} \rsv_\nobs^\transpose \vec\labl_\nobs\\
            &= \features_\nobs (\mat\gpkernel_\nobs + \mat\gpkernel_\nobs(e^{-\lrate\tstop\mat\gpkernel_\nobs} - \eye)^{-1})^{-1}\vec\labl_\nobs\\
            &= \features_\nobs (\mat\gpkernel_\nobs + \regmat_{\nobs})^{-1}\vec\labl_\nobs\,,
        \end{split}
    \end{equation}
    which concludes the proof.
\end{proof}

For our analysis, we will assume that the classifier network is zero initialized with $\classifier_0 = 0$, noting that the least-squares problem can always be solved for the residuals $\labl - \classifier_0(\obs)$ and then have $\classifier_0$ added back to the solution. We refer the reader to \citet{lee2019} for further discussion on the effect of the network initialization.

\paragraph{Approximation for finite-width networks.} For fully connected, convolutional or residual networks equipped with smooth activation functions (e.g., sigmoid or $\tanh$), \citet{liu2020c} showed that the approximation error between the linear model and the finite-width NN is $\widetilde{\bigo}(\mwidth^{-1/2})$, where $\mwidth$ denotes the minimum layer width, and the $\widetilde{\bigo}$ notation corresponds to the $\bigo$-notation with logarithmic factors suppressed. NTK results for other activation functions and different neural network architectures, such as multi-head attention \citep{hron2020}, are also available in the literature.

\subsection{Assumptions}
In the following, we present a series of mild technical assumptions needed for our theoretical analysis of NN-based CPEs.
For this analysis, we mainly assume that the true classifier $\classifier^*(\obs) = \probc{\tar>\thresh}{\obs, \bbf}$ is a fixed, though unknown, element of the RKHS $\fspace_\gpkernel$ given by a bounded NTK $\gpkernel$, which is formalized by the following two assumptions. As in the GP case, we assume batches of size $B=1$ to simplify the analysis.
\begin{assumption}
    \label{a:rkhs}
    There is $\classifier^*\in\fspace_\gpkernel$ such that:
    \begin{equation}
        \classifier^*(\obs) = \probc{\tar>\thresh}{\obs, \bbf}, \quad \forall\obs\in\obsspace.
    \end{equation}
\end{assumption}
For a rich enough RKHS, such assumption is mild, especially given that most popular NN architectures offer universal approximation guarantees \citep{hornik1989}.

\begin{assumption}
	\label{a:ntk-bound}
    The NTK $\gpkernel$ corresponding to the network architecture in $\classifier_\mparam$ is positive definite and bounded in $\obsspace$.
\end{assumption}

We will also assume that the threshold is fixed to simplify the analysis. However, our results should asymptotically hold for time-varying thresholds as long as the limit $\lim_{\iterIdx\to\infty} \thresh_\iterIdx=\thresh$ exists.
\begin{assumption}
	\label{a:thresh-fixed}
	The threshold is fixed, i.e., $\thresh_\iterIdx = \thresh \in \real$, for all $\iterIdx\geq 1$.
\end{assumption}

The following assumption on label noise should always hold for Bernoulli random variables \citep{boucheron2013}. Any upper bound on the sub-Gaussian parameter should suffice for the analysis (e.g., $\sigma_\lablnoise \leq 1$ for Bernoulli variables).
\begin{assumption}
    \label{a:var-bound}
    For all $\iterIdx\in\N$ and all $\obs\in\obsspace$, label noise $\lablnoise = \indic{\tar > \thresh} - \classifier^*(\obs)$, with $\tar \sim \probc{\tar}{\obs, \bbf}$, is $\sigma_\lablnoise$-sub-Gaussian:
    \begin{equation}
        \forall\anyscalar \in \real, \quad \expectation\left[\exp\left(\anyscalar\lablnoise\right) \right] \leq \exp\left(\frac{\anyscalar^2\sigma_\lablnoise^2}{2}\right),
    \end{equation}
    for some $\sigma_\lablnoise \geq 0$.
\end{assumption}

The next assumption ensures a sufficient amount of sampling is asymptotically achieved over the domain $\obsspace$, which we still assume is finite.
\begin{assumption}
    \label{a:qmin}
    For any $\iterIdx \geq 1$, the variational family is such that sampling probabilities are bounded away from 0, i.e.:
    \begin{equation}
        \exists \pmin > 0: \quad \forall \iterIdx\in\N, \quad \qrobc{\obs}{\qparam_\iterIdx} \geq \pmin, \quad \forall \obs\in\obsspace\,.
    \end{equation}
\end{assumption}
The assumption above only imposes mild constraints on the generative models $\qrobc{\obs}{\qparam}$, so that probabilities for all candidates $\obs\in\obsspace$ are never exactly 0, though still allowed to be arbitrarily small.

\begin{assumption}
    \label{a:learning-rate}
    The learning rate $\lrate_\iterIdx$ at each round $\iterIdx$ is such that:
    \begin{equation}
        0 < \lrate_* \leq \lrate_\iterIdx \leq \frac{1}{\eigval_{\max}(\mat\gpkernel_\iterIdx)}\,, \quad \forall \iterIdx\in\N\,,
    \end{equation}
    for some $\lrate_* > 0$, where $\eigval_{\max}(\cdot)$ denotes the maximum eigenvalue of a matrix, and $\mat\gpkernel_\iterIdx$ denotes the NTK matrix evaluated at the training points available at iteration $\iterIdx\geq 1$.
\end{assumption}

This last assumption ensures that a gradient descent algorithm is convergent \citep{fleming1990}, though we use it to bound the spectrum of the implicit regularization matrix $\regmat_\iterIdx$ after a finite number of training steps $\tstop < \infty$ (a.k.a. early stopping), which is needed for our results. We highlight that, under mild assumptions on the data distribution, $\eigval_{\max}(\mat\gpkernel_\iterIdx) \in \bigo(1)$ w.r.t. the number of data points \citep{murray2023}, so that the bound in \autoref{a:learning-rate} will not vanish. Therefore, such assumption is easily satisfied by maintaining a sufficiently small learning rate.

\subsection{Approximation error for NN-based CPEs}
Similar to the GP-PI setting, we will assume a batch size of 1, so that we can simply use the iteration index $\iterIdx \geq 0$ for our estimators. We recall that convergence rates for the batch setting should only be affected by a batch-size-dependent multiplicative factor, preserving big-$\bigo$ convergence rates. We start by defining the following \emph{proxy} variance:
\begin{equation}
    \iterIdx\geq 1, \quad \hat\sigma_\iterIdx^2(\obs) = \gpkernel(\obs,\obs) - \vec{\gpkernel}_\iterIdx(\obs)^\transpose(\mat\gpkernel_\iterIdx + \regmat_\iterIdx)^{-1}\vec{\gpkernel}_\iterIdx(\obs)\,, \quad \obs\in\obsspace,
    \label{eq:ntk-pred-var}
\end{equation}
where $\regmat_\iterIdx$ is the implicit regularization matrix in \autoref{eq:es-solution} due to early stopping. The proxy variance is then equivalent to a GP posterior variance under the assumption of heteroscedastic (i.e., input dependent) Gaussian noise with covariance matrix given by $\regmat_\iterIdx$. Given its similarities, we have that if enough sampling is asymptotically guaranteed, we can apply the same convergence results available for the GP-PI-based CPE, i.e., $\hat\sigma_\iterIdx^2 \in \bigo(\iterIdx^{-1})$ almost surely.

\begin{lemma}
    \label{thr:ntk-variance}
    Let assumptions \ref{a:ntk-bound}, \ref{a:qmin} and \ref{a:learning-rate} hold. Then the following almost surely holds for the proxy variance:
    \begin{equation}
        \hat\sigma_\iterIdx^2 \in \bigo(\iterIdx^{-1})\,.
    \end{equation}
\end{lemma}
\begin{proof}
    We first observe that $\hat\sigma_\iterIdx^2$ \eqref{eq:ntk-pred-var} is upper bounded by the posterior predictive variance of a GP model assuming i.i.d. Gaussian noise (cf. \autoref{sec:gp-posterior}) with variance $\regfactor^*$ satisfying:
    \begin{equation}
        \regfactor^* \geq \eigval_{\max}(\regmat_\iterIdx)\,, \quad \forall\iterIdx\in\N\,,
    \end{equation}
    which is such that:
    \begin{equation}
        \begin{split}
            \regmat_\iterIdx &= \mat\gpkernel_\iterIdx(e^{\lrate_\iterIdx\tstop\mat\gpkernel_\iterIdx} - \eye)^{-1}\\
            &\preceq \mat\gpkernel_\iterIdx(\lrate_\iterIdx\tstop\mat\gpkernel_\iterIdx)^{-1}\\
            &\preceq \frac{1}{\tstop\lrate_\iterIdx}\eye
        \end{split}
    \end{equation}
    since $e^{\mat A} \succeq \eye + \mat A$, for any Hermitian matrix $\mat A$, where $\succeq$ denotes the Loewner partial ordering in the space of positive-semidefinite matrices, i.e., $\mat{A} \succeq \mat{B}$ if and only if $\mat{A} - \mat{B}$ is positive semidefinite.
    Noting that the sum of sampling probabilities at any point $\obs\in\obsspace$ diverges as $\iterIdx\to\infty$ by \autoref{a:qmin}, the result then follows by applying \autoref{thr:gp-variance-convergence} to the GP predictive variance upper bound with noise variance set to $\regfactor^* := (\tstop\lrate_*)^{-1}$ (\autoref{a:learning-rate}).
\end{proof}

\begin{lemma}
    \label{thr:ntk-ucb}
    Let assumptions \ref{a:rkhs} to \ref{a:learning-rate} hold. Then, given any $\delta\in (0, 1]$, the following holds with probability at least $1-\delta$ for the approximation error between $\hat\classifier_\iterIdx$ and $\classifier^*$:
    \begin{equation}
        \forall \iterIdx \geq 1, \quad | \hat\classifier_\iterIdx(\obs) - \classifier^*(\obs) | \leq \beta_\iterIdx(\delta) \hat\sigma_\iterIdx(\obs), \quad \obs\in\obsspace,
    \end{equation}
    where $\beta_\iterIdx(\delta) := \norm{\classifier^*}_\gpkernel + \sigma_\lablnoise\sqrt{2\regfactor^{-1}\log(\det(\eye + \regfactor^{-1}\mat\gpkernel_\iterIdx)^{1/2}/\delta)}$, and $\regfactor :=\frac{e^{-\tstop}}{\tstop\lrate_*}$.
\end{lemma}
\begin{proof}
    The result above is a direct application of Theorem 3.5 in \citet{maillard2016} which provides an upper confidence bound on the kernelized least-squares regressor approximation error (another version of the same result is also available in \citet[Thr. 1]{durand2018}).

    Let $\lablnoise_i := \labl_i - \classifier^*(\obs_i)$ denote the label noise in observation $i$, for $i \in \{1, \dots, \iterIdx\}$. Expanding the definition of $\hat{\classifier}_\iterIdx$ \eqref{eq:es-solution} with $\classifier_0 = 0$, given any $\iterIdx\in\N$ and $\obs\in\obsspace$, we can decompose the approximation error as:
	\begin{equation}
		\begin{split}
			| \classifier^*(\obs) - \hat\classifier_\iterIdx(\obs) |
			&= | \classifier^*(\obs) - \vec{\gpkernel}_\iterIdx(\obs)^\transpose(\mat{\gpkernel}_\iterIdx + \regmat_\iterIdx)^{-1}\vec{\labl}_\iterIdx | \\
			&= | {\classifier}^*(\obs) - \vec{\gpkernel}_\iterIdx(\obs)^\transpose(\mat{\gpkernel}_\iterIdx + \regmat_\iterIdx)^{-1}(\vec{{\classifier}}_\iterIdx^* + \vec{\lablnoise}_\iterIdx) |\\
			&\leq | {\classifier}^*(\obs) - \vec{\gpkernel}_\iterIdx(\obs)^\transpose(\mat{\gpkernel}_\iterIdx + \regmat_\iterIdx)^{-1}\vec{{\classifier}}_\iterIdx^*| + |\vec{\gpkernel}_\iterIdx(\obs)^\transpose(\mat{\gpkernel}_\iterIdx + \regmat_\iterIdx)^{-1}\vec{\lablnoise}_\iterIdx |
		\end{split}
		\label{eq:rkhs-error-decomp}
	\end{equation}
	where we applied the triangle inequality to obtain the last line. Analyzing the two terms on the right-hand side, by the reproducing property, we now have for the first term:
	\begin{equation}
		\begin{split}
			| {\classifier}^*(\obs) - \vec{\gpkernel}_\iterIdx(\obs)^\transpose(\mat{\gpkernel}_\iterIdx + \regmat_\iterIdx)^{-1}\vec{{\classifier}}_\iterIdx^*|
			&= |\inner{
                    \classifier^*, 
                    (\eye - \features_\iterIdx(\mat{\gpkernel}_\iterIdx + \regmat_\iterIdx)^{-1}\features_\iterIdx^\transpose)\feature(\obs)}_\gpkernel| \\
                &= |\inner{
                    \classifier^*, 
                    (\eye + \features_\iterIdx\regmat_\iterIdx^{-1}\features_\iterIdx^\transpose)^{-1}\feature(\obs)}_\gpkernel| \\
			&\leq \norm{{\classifier}^*}_\gpkernel
                    \norm{(\eye + \features_\iterIdx\regmat_\iterIdx^{-1}\features_\iterIdx^\transpose)^{-1}\feature(\obs)}_\gpkernel  \\
                &= \norm{{\classifier}^*}_\gpkernel
                    \sqrt{
                        \feature(\obs)^\transpose(\eye + \features_\iterIdx\regmat_\iterIdx^{-1}\features_\iterIdx^\transpose)^{-2}\feature(\obs)
                    }\\
                &\leq \norm{{\classifier}^*}_\gpkernel
                    \sqrt{
                        \feature(\obs)^\transpose(\eye + \features_\iterIdx\regmat_\iterIdx^{-1}\features_\iterIdx^\transpose)^{-1}\feature(\obs)
                    }\\
			&=\norm{{\classifier}^*}_\gpkernel\hat\sigma_\iterIdx(\obs),
		\end{split}
	\end{equation}
	where the second equality follows by an application of Woodbury's identity, the first inequality is due to Cauchy-Schwarz, the second inequality is due to the fact that $\mat{A}^{-2} \preceq \mat{A}^{-1}$ whenever $\mat{A} \succeq \eye$, and the last line follows from the definition of $\hat\sigma_\iterIdx^2$. For the remaining in term \eqref{eq:rkhs-error-decomp}, we have that:
    \begin{equation}
        \begin{split}
            |\vec{\gpkernel}_\iterIdx(\obs)^\transpose(\mat{\gpkernel}_\iterIdx + \regmat_\iterIdx)^{-1}\vec{\lablnoise}_\iterIdx | 
            &= |\inner{\feature(\obs), \features_\iterIdx(\mat{\gpkernel}_\iterIdx + \regmat_\iterIdx)^{-1}\vec{\lablnoise}_\iterIdx}_\gpkernel|\\
            &= |\inner{\feature(\obs), (\eye + \features_\iterIdx\regmat_\iterIdx^{-1}\features_\iterIdx^\transpose)^{-1}\features_\iterIdx\regmat_\iterIdx^{-1}\vec{\lablnoise}_\iterIdx}_\gpkernel|\\
            &= |\inner{(
                \eye + \features_\iterIdx\regmat_\iterIdx^{-1}\features_\iterIdx^\transpose)^{-1/2}\feature(\obs), 
                (\eye + \features_\iterIdx\regmat_\iterIdx^{-1}\features_\iterIdx^\transpose)^{-1/2}\features_\iterIdx\regmat_\iterIdx^{-1}\vec{\lablnoise}_\iterIdx
                }_\gpkernel|\\
            &\leq 
            \sqrt{
                \feature(\obs)^\transpose(\eye + \features_\iterIdx\regmat_\iterIdx^{-1}\features_\iterIdx^\transpose)^{-1}\feature(\obs)
                }
                \norm{(\eye + \features_\iterIdx\regmat_\iterIdx^{-1}\features_\iterIdx^\transpose)^{-1/2}\features_\iterIdx\regmat_\iterIdx^{-1}\vec{\lablnoise}_\iterIdx}_\gpkernel\\
            &= \hat\sigma_\iterIdx(\obs) \sqrt{
                \vec\lablnoise_\iterIdx^\transpose
                \regmat_\iterIdx^{-1}\features_\iterIdx^\transpose
                (\eye + \features_\iterIdx\regmat_\iterIdx^{-1}\features_\iterIdx^\transpose)^{-1}
                \features_\iterIdx\regmat_\iterIdx^{-1}
                \vec{\lablnoise}_\iterIdx
            },
        \end{split} 
    \end{equation}
    where we applied the identity $\mat B^\transpose(\mat B \mat B^\transpose + \mat A)^{-1} = (\eye + \mat B^\transpose\mat A^{-1}\mat B)^{-1} \mat B^\transpose \mat A^{-1}$, which holds for an invertible matrix $\mat A$ \citep{searle1982matrix}, to obtain the second equality, and the upper bound follows by the Cauchy-Schwarz inequality.
    For the norm of the noise-dependent term, we will apply a concentration bound by \citet{abbasi-yadkori2012}, which first requires a few transformations towards a non-time-varying regularization factor. Applying the SVD $\features_\iterIdx = \lsv_\iterIdx\svs_\iterIdx\rsv_\iterIdx^\transpose$ \citep{mollenhauer2020}, as in the proof of \autoref{thr:es-implicit-reg}, we have that $\regmat_\iterIdx = \rsv_\iterIdx\regdiag_\iterIdx\rsv_\iterIdx^\transpose$, where $\regdiag_\iterIdx := \evs_\iterIdx(e^{\lrate_\iterIdx\tstop\evs_\iterIdx} - \eye)^{-1}$ and $\evs_\iterIdx = \svs_\iterIdx^2$, which leads us to:
    %
    \begin{equation}
        \begin{split}
            \norm{(\eye + \features_\iterIdx\regmat_\iterIdx^{-1}\features_\iterIdx^\transpose)^{-1/2}\features_\iterIdx\regmat_\iterIdx^{-1}\vec{\lablnoise}_\iterIdx}_\gpkernel^2
            &= 
            \vec\lablnoise_\iterIdx^\transpose\regmat_\iterIdx^{-1}\features_\iterIdx^\transpose(\eye + \features_\iterIdx\regmat_\iterIdx^{-1}\features_\iterIdx^\transpose)^{-1}\features_\iterIdx\regmat_\iterIdx^{-1}\vec{\lablnoise}_\iterIdx
            \\
            &= 
            \vec\lablnoise_\iterIdx^\transpose \rsv_\iterIdx\regdiag_\iterIdx^{-1}\svs_\iterIdx\lsv_\iterIdx^\transpose(\eye + \lsv_\iterIdx\svs_\iterIdx\regdiag_\iterIdx^{-1}\svs_\iterIdx\lsv_\iterIdx^\transpose)^{-1}\lsv_\iterIdx\svs_\iterIdx\regdiag_\iterIdx^{-1}\rsv_\iterIdx^\transpose\vec\lablnoise_\iterIdx
            \\
            &=\vec\lablnoise_\iterIdx^\transpose 
            \rsv_\iterIdx
            \svs_\iterIdx^2\regdiag_\iterIdx^{-2}(
                \eye
                + \regdiag_\iterIdx^{-1}\svs_\iterIdx^2
                )^{-1}
                \rsv_\iterIdx^\transpose
                \vec\lablnoise_\iterIdx
            \\
            &=\vec\lablnoise_\iterIdx^\transpose 
            \rsv_\iterIdx
            (
                \svs_\iterIdx^{-2}\regdiag_\iterIdx^2
                + \regdiag_\iterIdx
            )^{-1}
            \rsv_\iterIdx^\transpose
            \vec\lablnoise_\iterIdx\,,
        \end{split}
    \end{equation}
    where we applied the identity $(\eye + \mat A \mat B)^{-1}\mat A = \mat A (\eye + \mat B \mat A)^{-1}$ \citep{searle1982matrix}. For the eigenvalues of $\regmat_\iterIdx$, we have the following lower bound:
    \begin{equation}
        \begin{split}
            \regdiag_\iterIdx 
            &= \evs_\iterIdx(e^{\lrate_\iterIdx\tstop\evs_\iterIdx} - \eye)^{-1}\\
            &= \evs_\iterIdx e^{-\lrate_\iterIdx\tstop\evs_\iterIdx}(\eye - e^{-\lrate_\iterIdx\tstop\evs_\iterIdx})^{-1}\\
            &\succeq \evs_\iterIdx e^{-\lrate_\iterIdx\tstop\evs_\iterIdx} (\lrate_\iterIdx\tstop\evs_\iterIdx)^{-1}\\
            &= \frac{1}{\lrate_\iterIdx\tstop} e^{-\lrate_\iterIdx\tstop\evs_\iterIdx}\\
            &\succeq \frac{1}{\lrate_*\tstop} e^{-\tstop} \eye\,,
        \end{split}
    \end{equation}
    where the first inequality is due to $e^{-\mat A} \succeq \eye - \mat A$, and the last inequality holds by \autoref{a:learning-rate}. Hence, setting $\regfactor := \frac{e^{-\tstop}}{\tstop\lrate_*}$, we have that:
    \begin{equation}
        \begin{split}
            \norm{(\eye + \features_\iterIdx\regmat_\iterIdx^{-1}\features_\iterIdx^\transpose)^{-1/2}\features_\iterIdx\regmat_\iterIdx^{-1}\vec{\lablnoise}_\iterIdx}_\gpkernel^2 
            &\leq \vec\lablnoise_\iterIdx^\transpose 
            \rsv_\iterIdx
            (
                \regfactor^2\svs_\iterIdx^{-2}
                + \regfactor\eye
            )^{-1}
            \rsv_\iterIdx^\transpose
            \vec\lablnoise_\iterIdx\\
            &=\regfactor\vec\lablnoise_\iterIdx^\transpose\rsv_\iterIdx\svs_\iterIdx(\regfactor\eye + \svs_\iterIdx^2)^{-1}\svs_\iterIdx\rsv_\iterIdx^\transpose\vec\lablnoise_\iterIdx
            \\
            &=\regfactor^{-1}\vec\lablnoise_\iterIdx^\transpose\rsv_\iterIdx\svs_\iterIdx(\regfactor\eye + \svs_\iterIdx^2)^{-1}\svs_\iterIdx\rsv_\iterIdx^\transpose\vec\lablnoise_\iterIdx
            \\
            &=\regfactor^{-1}\vec\lablnoise_\iterIdx^\transpose\features_\iterIdx^\transpose(\regfactor\eye + \features_\iterIdx\features_\iterIdx^\transpose)^{-1}\features_\iterIdx\vec\lablnoise_\iterIdx
            \\
            &= \norm{(\regfactor\eye + \features_\iterIdx\features_\iterIdx^\transpose)^{-1/2}\features_\iterIdx\vec\lablnoise_\iterIdx}_\gpkernel^2
        \end{split}
    \end{equation}
    By \citet[Cor. 3.6]{abbasi-yadkori2012}, given any $\delta\in(0,1]$, we then have that the following holds with probability at least $1-\delta$:
    \begin{equation}
        \forall\iterIdx\geq 1,\quad \norm{(\eye + \features_\iterIdx\regmat_\iterIdx^{-1}\features_\iterIdx^\transpose)^{-1/2}\features_\iterIdx\regmat_\iterIdx^{-1}\vec{\lablnoise}_\iterIdx}_\gpkernel^2 \leq 2\sigma_\lablnoise^2\log\left(\frac{\det(\eye + \regfactor^{-1}\mat\gpkernel_\iterIdx)^{1/2}}{\delta}\right)\,.
    \end{equation}
    Finally, combining the bounds above into \autoref{eq:rkhs-error-decomp} leads to the result in \autoref{thr:ntk-ucb}.
\end{proof}

For the next result, we need to define the following quantity:
\begin{equation}
    \mig_\niter := \max_{\obsspace_\niter \subset \obsspace: \card{\obsspace_\niter} \leq \niter} \frac{1}{2} \log \det(\eye + \regfactor^{-1}\mat\gpkernel(\obsspace_\niter))\,,
    \label{eq:mig}
\end{equation}
where $\mat\gpkernel(\obsspace_\niter) := [\gpkernel(\obs, \obs')]_{\obs,\obs'\in\obsspace_\niter} \in \real^{\card{\obsspace_\niter}\times \card{\obsspace_\niter}}$. Note that $\mig_\niter$ corresponds to the maximum information gain of a GP model \citep{srinivas2010gaussian} with covariance function given by the NTK, assuming Gaussian observation noise with variance given by $\regfactor$. Then $\mig_\niter$ is mainly dependent on the eigenvalue decay of the kernel under its spectral decomposition \citep{vakili2021}. For the spectrum of the NTK, a few results are available in the literature \citep{murray2023}.

\begin{proposition}
    \label{thr:ntk-kl}
    Let assumptions \ref{a:rkhs} to \ref{a:learning-rate} hold. Then, given $\delta\in(0,1]$, the following holds with probability at least $1-\delta$ for VSD equipped with a wide enough NN-based CPE model $\hat\classifier_\iterIdx$:
    \begin{equation}
        \diver{\probc{\obs}{\tar > \thresh_\iterIdx, \data_\iterIdx}}{\probc{\obs}{\tar > \thresh_\iterIdx, \bbf}} \in \bigo_\pMeasure\left(\sqrt{\frac{\mig_\iterIdx}{\iterIdx}}\right)\,.
    \end{equation}
\end{proposition}
\begin{proof}
    The result follows by applying the same steps as in the proof of \autoref{thr:kl-bound}. We note that $\lik_\iterIdx^*(\obs) = \classifier^*(\obs) > 0$, due to observation noise, so that $\lik_\iterIdx^*(\obs)^{-1} \in \bigo_\pMeasure(1)$. Similarly, \autoref{thr:ntk-ucb} implies that $|\hat\classifier_\iterIdx(\obs) - \classifier^*(\obs)| \leq \beta_\iterIdx(\delta)\sigma_\iterIdx(\obs)$ with probability at least $1-\delta$ simultaneously over all $\obs\in\obsspace$, so that ratio-dependent terms in \autoref{thr:kl-bound} should remain bounded in probability. The upper bound in the result then follows by noticing that in our case $|\ldiff_\iterIdx(\obs)| \leq \beta_\iterIdx(\delta)\hat\sigma_\iterIdx(\obs)$ with high probability, where $\hat\sigma_\iterIdx \in \bigo(\iterIdx^{-1/2})$ by \autoref{thr:ntk-variance}, and $\beta_\iterIdx(\delta) \in \bigo(\sqrt{\mig_\iterIdx})$ by \autoref{thr:ntk-ucb} and the definition of $\mig_\iterIdx$ in \autoref{eq:mig}.
\end{proof}

The result above tells us that VSD equipped with an NN-based CPE can recover a similar asymptotic convergence guarantee to the one we derived for the GP-PI case, depending on the choice of NN architecture and more specifically on the spectrum of its associated NTK. In the case of a fully connected multi-layer ReLU network, for example, \citet{chen2021deepntklap} showed an equivalence between the RKHS of the ReLU NTK and that of the Laplace kernel $\gpkernel(\obs,\obs') = \exp(-\constant\norm{\obs-\obs'})$. As the latter is equivalent to a Mat\'ern kernel with smoothness parameter set to $0.5$ \citep{rasmussen2006}, the corresponding information gain bound is $\mig_\iterIdx \in \widetilde\bigo(\iterIdx^{\frac{d}{1+d}})$, where $d$ here denotes the dimensionality of the domain $\obsspace$ \citep{vakili2023a}. In the case of discrete sequences of length $M$, the dimensionality of $\obsspace$ is determined by $M$. Hence, we have proven \autoref{thr:relu-kl}.\footnote{Here $\widetilde{\bigo}_\pMeasure$ suppresses logarithmic factors, as in $\widetilde{\bigo}$, and holds in probability.}

\relukl*

Similar steps can be applied to derive convergence guarantees for VSD with other neural network architectures based on the eigenspectrum of their NTK \citep{murray2023} and following the recipe in, e.g., \citet{vakili2021} or \citet{srinivas2010gaussian}.


\section{VSD as a Black-Box Optimization Lower Bound}
\label{app:vsdbound}

A natural question to ask is how \gls{vsd} relates to the \gls{bo} objective for probability of improvement~\citep[Ch.7]{garnett2023bayesian},
\begin{align}
    \obs^*_t = \argmax_\obs \log \acqfn{\obs, \data_N, \thresh}{PI}.
    \label{eq:bopi}
\end{align}
Firstly, we can see that the expected log-likelihood of term of \autoref{eq:vsd_elbo} lower-bounds this quantity.
\begin{proposition}\label{prop:maxpr}
For a parametric model, $\qrobc{\obs}{\qparam}$, given $\qparam \in \qparamspace \subseteq \real^m$ and $q \in \mathcal{P} : \obsspace \times \qparamspace \to [0, 1]$,
\begin{align}
    \max_\obs \log \acqfn{\obs, \data_N, \thresh}{PI}
        \geq \max_\qparam \expec{\qrobc{\obs}{\qparam}}{\log \acqfn{\obs, \data_N, \thresh}{PI}},
    \label{eq:maxpr}
\end{align}%
and the bound becomes tight as $\qrobc{\obs}{\qparam^*_t} \to \delta(\obs^*_t)$, a Dirac delta function at the maximizer $\obs^*_t$.
\end{proposition}
Taking the argmax of the RHS will result in the variational distribution collapsing to a delta distribution at $\obs^*_t$ for an appropriate choice of $\qrobc{\obs}{\qparam}$. The intuition for \autoref{eq:maxpr} is that the expected value of a random variable is always less than or equal to its maximum. The proof of this is in \citet{daulton2022bayesian, staines2013optimization}. Extending this lower bound, we can show the following.
\begin{proposition}\label{prop:maxprdiv}
For a divergence $\mathbb{D} : \mathcal{P}(\obsspace) \times \mathcal{P}(\obsspace) \to [0, \infty)$, and a prior $\pdist_0 \in \mathcal{P}(\obsspace)$,
\begin{align}
    \max_\obs \log \acqfn{\obs, \data_N, \thresh}{PI}
    \geq \max_\qparam \expec{\qrobc{\obs}{\qparam}}{\log \acqfn{\obs, \data_N, \thresh}{PI}} - \diver{\qrobc{\obs}{\qparam}}{\pdist_0(\obs)}.
    \label{eq:maxprdiv}
\end{align}
\end{proposition}
We can see that this bound is trivially true given the range of divergences, and this covers \gls{vsd} as a special case. However, this bound is tight if and only if $\pdist_0$ concentrates as a Dirac delta at $\obs_\iterIdx^*$ with an appropriate choice of $\qrobc{\obs}{\qparam}$. In any case, the lower bound remains valid for any choice of informative prior $\pdist_0$ or even a uninformed prior, which allows us to maintain the framework flexible to incorporate existing prior information whenever that is available.


\end{document}